\documentclass{article} 

\usepackage{hyperref}
\usepackage{url}
\usepackage[numbers,sort]{natbib}

\usepackage{amssymb}
\usepackage{amsthm}
\usepackage{amsmath}
\usepackage{mathtools}
\usepackage{microtype}
\usepackage{graphicx}
\usepackage{booktabs} 
\usepackage{sidecap}
\usepackage{algorithm}
\usepackage{algorithmic}
\usepackage{subcaption}
\usepackage{bbm}
\usepackage{amsfonts,caption}
\usepackage{url,epsfig,epsf,xcolor,mathbbol,fmtcount,multirow}
\usepackage{MnSymbol}
\usepackage{latexsym}
\usepackage{tabularx}
\usepackage[bottom,hang,flushmargin]{footmisc}
\usepackage{comment}
\usepackage[bottom]{footmisc}
\usepackage{caption}
\usepackage{enumitem}
\usepackage{listings}
\usepackage{comment}
\usepackage{pifont} 
\newcommand{\cmark}{\ding{51}}%
\newcommand{\xmark}{\ding{55}}%
\usepackage{placeins}
\usepackage{hhline}
\usepackage[accepted]{icml2020}
\usepackage{xr}
\makeatletter
\newcommand*{\addFileDependency}[1]{
  \typeout{(#1)}
  \@addtofilelist{#1}
  \IfFileExists{#1}{}{\typeout{No file #1.}}
}
\makeatother
\newcommand*{\myexternaldocument}[1]{
    \externaldocument{#1}
    \addFileDependency{#1.tex}
    \addFileDependency{#1.aux}
}
\myexternaldocument{supplement}

\newcolumntype{C}[1]{>{\centering\arraybackslash}m{#1}}

\newtheorem{definition}{\textbf{Definition}}

\newtheorem{theorem}{\textbf{Theorem}}

\newtheorem{lemma}{\textbf{Lemma}}

\newcommand{\bWW}{\mathbf{W}}

\newcommand{\bb}{\mathbf{b}}

\newcommand{\bx}{\mathbf{x}}

\newcommand{\ba}{\mathbf{a}}
\newcommand{\bv}{\mathbf{v}}

\newcommand{\bg}{\mathbf{g}}

\newcommand{\bu}{\mathbf{u}}

\newcommand{\bA}{\mathbf{A}}
\newcommand{\bB}{\mathbf{B}}

\newcommand{\bF}{\mathbf{F}}

\newcommand{\bS}{\mathbf{S}}

\newcommand{\bP}{\mathbf{P}}
\newcommand{\bE}{\mathbf{E}}

\newcommand{\bN}{\mathbf{N}}
\newcommand{\bI}{\mathbf{I}}

\newcommand{\bz}{\mathbf{z}}

\newcommand{\bH}{\mathbf{H}}

\newcommand{\bq}{\mathbf{q}}

\newcommand{\grad}{\nabla}

\def\<{\langle}
\def\>{\rangle}

\date{}
\icmltitlerunning{Second-Order Provable Defenses against Adversarial Attacks}
\begin{document}

\twocolumn[
\icmltitle{Second-Order Provable Defenses against Adversarial Attacks}

\icmlsetsymbol{equal}{*}

\begin{icmlauthorlist}
\icmlauthor{Sahil Singla}{umd}
\icmlauthor{Soheil Feizi}{umd}
\end{icmlauthorlist}

\icmlaffiliation{umd}{Department of Computer Science, University of Maryland, College Park}

\icmlcorrespondingauthor{Sahil Singla}{ssingla@cs.umd.edu}
\icmlcorrespondingauthor{Soheil Feizi}{sfeizi@cs.umd.edu}

\icmlkeywords{Machine Learning, ICML}
\vskip 0.3in
]
\printAffiliationsAndNotice{}

	\begin{abstract}
		 A robustness certificate is the minimum distance of a given input to the decision boundary of the classifier (or its lower bound). For {\it any} input perturbations with a magnitude smaller than the certificate value, the classification output will provably remain unchanged. Exactly computing the robustness certificates for neural networks is difficult since it requires solving a non-convex optimization. In this paper, we provide computationally-efficient robustness certificates for neural networks with differentiable activation functions in two steps. First, we show that if the eigenvalues of the Hessian of the network are bounded, we can compute a robustness certificate in the $l_2$ norm efficiently using convex optimization. Second, we derive a computationally-efficient differentiable upper bound on the curvature of a deep network. We also use the curvature bound as a regularization term during the training of the network to boost its certified robustness. Putting these results together leads to our proposed {\bf C}urvature-based {\bf R}obustness {\bf C}ertificate (CRC) and {\bf C}urvature-based {\bf R}obust {\bf T}raining (CRT). Our numerical results show that CRT leads to significantly higher certified robust accuracy compared to interval-bound propagation (IBP) based training. We achieve certified robust accuracy 69.79\%, 57.78\% and 53.19\% while IBP-based methods achieve 44.96\%, 44.74\% and 44.66\% on 2,3 and 4 layer networks respectively on the MNIST-dataset.
	\end{abstract}
	
	\section{Introduction} \label{bounded_curvature_condition}
	
	Modern neural networks achieve high accuracy on tasks such as image classification and speech recognition, but are known to be brittle to small, adversarially chosen perturbations of their inputs \cite{42503}. A classifier which correctly classifies an image $\bx$, can be fooled by an adversary to misclassify an \textit{adversarial example} $\bx + \delta$, such that $\bx + \delta$ is indistinguishable from $\bx$ to a human. Adversarial examples can also fool systems when they are printed out on a paper and photographed with a smart phone \cite{Kurakin2016AdversarialEI}. Even in a black box threat model, where the adversary has no access to the model parameters, attackers could target autonomous vehicles by using stickers or paint to create an adversarial stop sign that the vehicle would interpret as a ‘yield’ or another sign \cite{DBLP:journals/corr/PapernotMGJCS16}. This trend is worrisome and suggests that adversarial vulnerabilities need to be appropriately addressed before neural networks can be deployed in security critical applications. 
	
	In this work, we propose a new approach for developing provable defenses against $\ell_2$-bounded adversarial attacks as well as computing robustness certifications of pre-trained deep networks with differentiable activations. In contrast to the existing certificates \cite{Zhang2018EfficientNN, Weng2018TowardsFC} that use the first-order information (upper and lower bounds on the slope), our approach is based on the second-order information (upper and lower bounds on curvature values i.e. eigenvalues of the Hessian). Our approach is based on two key theoretically-justified steps: First, in Theorems \ref{thm:certificate} and \ref{thm:attack}, we show that if the eigenvalues of the Hessian of the network (curvatures of the network) are bounded (globally or locally), we can efficiently compute a robustness certificate and develop a defense method against $\ell_2$-bounded adversarial attacks using convex optimization. Second, in Theorem \ref{thm:L_layer_p_n_theorem}, we derive a computationally-efficient differentiable bound on the curvature (eigenvalues of the Hessian) of a deep network. We derive this bound by explicitly characterizing the Hessian of a deep network in Lemma \ref{thm:deep_hessian_closed}.
	
	Although the problem of finding the closest adversarial example to a given point for deep nets leads to a non-convex optimization problem, our proposed Curvature-based Robustness Certificate (CRC), under some verifiable conditions, is able to compute points on the decision boundary that are provably closest to the input. That is, it provides the tightest certificate in those cases. For example, for a 2,3,4 layer networks trained on MNIST, we can find provably closest adversarial points for 44.17\%, 22.59\%, 19.53\% cases, respectively (Table \ref{table:primal_dual_eq_fraction_summary}). To the best of our knowledge, our method is the first approach that can efficiently compute provably closest adversarial examples for a significant fraction of examples in non-trivial neural networks. 
	
    We note that un-regularized networks, specially deep ones, can obtain large curvature bounds which can lead to small robustness certificates. However, by using the derived curvature bound as a regularizer during training, we significantly decrease curvature values of the network, with little or no decrease in its performance (Table \ref{table:cert_comparison_short}, Figure \ref{fig:gamma_effect_2_3}). Using this technique, our method significantly outperforms interval-bound propagation (IBP) \cite{ZhangCROWNIBP, Wong2018ScalingPA} and achieves state of the art certified accuracy (Tables \ref{table:empirical_provable_adversarial_1.58_summary} and \ref{table:empirical_provable_adversarial_1.58_fashion_mnist}). In particular, our method achieves certified robust accuracy 69.79\%, 57.78\% and 53.19\% while IBP-based methods achieve 44.96\%, 44.74\% and 44.66\% on 2,3 and 4 layer networks, respectively, on the MNIST-dataset (similar results for Fashion-MNIST).
	
	Other recent works (e.g. \citet{MoosaviDezfooliCVPR2019, qin2019adversarial}) empirically show that using an \textit{estimate} of curvature at inputs as a regularizer leads to \textit{empirical} robustness on par with the adversarial training. In this work, however, we use a bound on the absolute value of curvature (and not an estimate) as a regularizer and show that it results in high \textit{certified} robustness. Moreover, previous works have tried to certify robustness by bounding the Lipschitz constant of the neural network \cite{42503, Peck2017LowerBO, Zhang2018RecurJacAE, Anil2018SortingOL, NIPS2017_6821}. Our approach, however, is based on bounding the Lipschitz constant of the gradient which in turn leads to bound on the eigenvalues of the Hessian of deep neural networks. 

	\begin{table*}
%        \vskip -0.1in
		\centering
		\renewcommand{\arraystretch}{1.5}
		\caption{A summary of various primal and dual concepts used in the paper. $f$ denotes the function of the decision boundary, i.e. $\bz^{(L)}_{y} - \bz^{(L)}_{t}$ where $y$ is the true label and $t$ is the attack target. $m$ and $M$ are lower and upper bounds on the smallest and largest eigenvalues of the Hessian of $f$, respectively. }
		\begin{tabular}{ | c | c | c | } 
			\hline
			& Certificate problem $_{(-)\ =\ cert}$ & Attack problem $_{(-)\ =\ attack}$\\
			\hline
			primal problem, $p^{*}_{(-)}$ & $\min_{f(\bx)=0} 1/2\|\bx - \bx^{(0)}\|^{2}$ & $\min_{\|\bx-\bx^{(0)}\|\leq \rho} f(\bx)$\\
			\hline
			dual function, $d_{(-)}(\eta)$ & $\min_{\bx} 1/2\|\bx - \bx^{(0)}\|^{2} + \eta f(\bx)$ & $\min_{\bx} f(\bx) + \eta/2(\|\bx - \bx^{(0)}\|^{2}-\rho^{2})$\\
			\hline
			When is dual solvable? & $-1/M \leq \eta \leq -1/m$ & $-m \leq \eta$\\
			\hline
			dual problem, $d^{*}_{(-)}$ & $\max_{-1/M \leq \eta \leq -1/m} d_{cert}(\eta)$ & $\max_{-m \leq \eta } d_{attack}(\eta)$\\
			\hline
			When primal$\ =\ $dual? & $f(\bx^{(cert)}) = 0$ & $\|\bx^{(attack)} - \bx^{(0)}\| = \rho$\\
			\hline
		\end{tabular}
		\label{table:cert_attack_gist}
	\end{table*}

	In summary, we make the following contributions:
	\begin{itemize}
		\item We derive a closed-form expression for the Hessian of a deep network with differentiable activation functions (Lemma \ref{thm:deep_hessian_closed}) and derive bounds on the curvature using this closed-form formula (Theorems  \ref{thm:single_layer_p_n_theorem} and  \ref{thm:L_layer_p_n_theorem}). 		
		\item We develop computationally efficient methods for both the robustness certification as well as the adversarial attack problems (Theorems \ref{thm:certificate} and \ref{thm:attack}).
		\item We provide verifiable conditions under which our method is able to compute points on the decision boundary that are provably closest to the input. Empirically, we show that this condition holds for a significant fraction of examples (Table \ref{table:primal_dual_eq_fraction_summary}). 
		\item We show that using our proposed curvature bounds as a regularizer during training leads to improved certified accuracy on 2,3 and 4 layer networks (on the MNIST and Fashion-MNIST datasets) compared to IBP-based adversarial training \cite{Wong2017ProvableDA, ZhangCROWNIBP} (Tables \ref{table:empirical_provable_adversarial_1.58_summary} and \ref{table:empirical_provable_adversarial_1.58_fashion_mnist}). Our robustness certificate (CRC) outperforms CROWN's certificate \cite{Zhang2018EfficientNN} significantly when trained with our regularizer (Table \ref{table:cert_comparison_short}).
	\end{itemize}
	To the best of our knowledge, this is the first work that (a) demonstrates the utility of second-order information for provable robustness, (b) derives a framework to find the exact robustness certificates in the $l_{2}$ norm and the exact worst case adversarial perturbation in an $l_{2}$ ball of given a radius under some conditions, and (c) derives an exact closed form expression for the Hessian and bounds on the curvature values using the same.
	
	\section{Related work}\label{sec:main_related_work}
	In the last couple of years, several \textit{empirical defenses} have been proposed for training classifiers to be robust against adversarial perturbations \cite{madry2018towards, samangouei2018defensegan, Zhang2019TheoreticallyPT, 7546524, Kurakin2016AdversarialML, Miyato2017VirtualAT, Zheng2016ImprovingTR} %(adversarial training, defense GAN,  {\bf XXXX add citations}). 
	Although these defenses robustify classifiers to particular types of attacks, they can be still vulnerable against stronger attacks \cite{Athalye2018ObfuscatedGG, Carlini:2017:AEE:3128572.3140444, Uesato2018AdversarialRA, Athalye2018OnTR}. For example, \cite{Athalye2018ObfuscatedGG} showed most of the empirical defenses proposed in ICLR 2018 can be broken by developing tailored attacks for each of them.  
	
	To end the cycle between defenses and attacks, a line of work on \textit{certified defenses} has gained attention where the goal is to train classifiers whose predictions are {\it provably} robust within some given region \cite{Huang2016SafetyVO, Katz2017ReluplexAE, Ehlers2017FormalVO, Carlini2017ProvablyMA, Cheng2017MaximumRO, Lomuscio2017AnAT, Dutta2018OutputRA, Fischetti2018DeepNN, Bunel2017AUV,Wang2018MixTrainST, Wong2017ProvableDA, Wang2018EfficientFS, Wong2018ScalingPA, Raghunathan2018SemidefiniteRF, Raghunathan2018CertifiedDA, Dvijotham2018TrainingVL, Dvijotham2018ADA, Croce2018ProvableRO, Singh2018FastAE, Gowal2018OnTE, Gehr2018AI2SA, Mirman2018DifferentiableAI, Zhang2018EfficientNN, Weng2018TowardsFC, ZhangCROWNIBP}. These methods, however, do not scale to large and practical networks used in solving modern machine learning problems. Another line of defense work focuses on \textit{randomized smoothing} where the prediction is robust within some region around the input with a user-chosen probability \cite{Liu2017TowardsRN, Cao2017MitigatingEA,Lcuyer2018CertifiedRT,Li2018CertifiedAR,Cohen2019CertifiedAR,Salman2019ProvablyRD}. Although these methods can scale to large networks, certifying robustness with probability close to 1 often requires generating a large number of noisy samples around the input which leads to high inference-time computational complexity. We discuss existing works in more details in Appendix \ref{sec:appendix_related_work}.

	\section{Notation}\label{sec:notation}
	Consider a fully connected neural network with $L$ layers and $N_{I}$ neurons in the $I^{th}$ layer ($L\ge2$ and $I \in$ $[L]$) for a multi-label classification problem with $C$ classes ($N_L=C$). The corresponding function of the neural network is $\bz^{(L)}:\mathbf{R}^D \to \mathbf{R}^C$ where $D$ is the dimension of the input. For an input $\bx$, we use $\bz^{(I)}(\bx) \in \mathbf{R}^{N_{I}}$ and $\ba^{(I)}(\bx) \in \mathbf{R}^{N_{I}}$ to denote the input ({\it before} applying the activation function) and output ({\it after} applying the activation function) of neurons in the $I^{th}$ hidden layer of the network, respectively. To simplify notation and when no confusion arises, we make the dependency of $\bz^{(I)}$ and $\ba^{(I)}$ to $\bx$ implicit. We define $\ba^{(0)}(\bx)=\bx$ and $N_{0}=D$. 
	
	With a fully connected architecture, each $\bz^{(I)}$ and $\ba^{(I)}$ is computed using a transformation matrix $\bWW^{(I)} \in R^{N_{I} \times N_{I-1}}$, the bias vector $\bb^{(I)} \in R^{N_{I}}$ and an activation function $\sigma(.)$ as follows:
	\begin{align*}
	&\bz^{(I)} = \bWW^{(I)}\ba^{(I-1)} + \bb^{(I)},\qquad \ba^{(I)} = \sigma\left(\bz^{(I)}\right)
	\end{align*}
	We use $(\bz^{(L)}_{i} -\bz^{(L)}_{j})(\bx)$ as a shorthand for $\bz^{(L)}_{i}(\bx) - \bz^{(L)}_{j}(\bx)$. 
	
	We use $[p]$ to denote the set $\{1,\dotsc,p\}$ and $[p,q],\ p \leq q$ to denote the set $\{p,p+1,\dotsc,q\}$. We use small letters $i,j,k$ etc to denote the index over a vector or rows of a matrix and capital letters $I,J$ to denote the index over layers of network. The element in the $i^{th}$ position of a vector $\bv$ is given by $\bv_{i}$, the vector in the $i^{th}$ row of a matrix $\bA$ is $\bA_{i}$ while the element in the $i^{th}$ row and $j^{th}$ column of $\bA$ is $\bA_{i,j}$. We use $\|\bv\|$ and $\|\bA\|$ to denote the 2-norm and the operator 2-norm of the vector $\bv$ and the matrix $\bA$, respectively. We use $\left|\bv\right|$ and $\left|\bA\right|$ to denote the vector and matrix constructed by taking the elementwise absolute values. We use $\lambda_{max}(\bA)$ and $\lambda_{min}(\bA)$ to denote the largest and smallest eigenvalues of a symmetric matrix $\bA$. We use $diag(\bv)$ to denote the diagonal matrix constructed by placing each element of $\bv$ along the diagonal. We use $\odot$ to denote the Hadamard Product, $\bI$ to denote the identity matrix. We use $\preccurlyeq$ and $\succcurlyeq$ to denote Linear Matrix Inequalities (LMIs) such that given two symmetric matrices $\bA$ and $\bB$ where $\bA \succcurlyeq \bB$ means  $\bA-\bB$ Positive Semi-Definite (PSD).
	
	\section{Using duality to solve the attack and certificate problems}\label{sec:nonconvex_results}
	Consider an input $\bx^{(0)}$ with true label $y$ and attack target $t$. In the certificate problem, our goal is to find a lower bound of minimum $l_{2}$ distance between $\bx^{(0)}$ and decision boundary $f(\bx)=0$ where $f(\bx)=(\bz^{(L)}_{y}-\bz^{(L)}_{t})(\bx)$. The problem for solving the exact distance (\textit{primal}) can be written as:
	\begin{align}
	&p^{*}_{cert} = \min_{f(\bx)=0}  \left[\frac{1}{2}\left\|\bx - \bx^{(0)}\right\|^{2} \right] \nonumber\\
	&p^{*}_{cert} = \min_{\bx} \max_{\eta} \left[\frac{1}{2}\left\|\bx - \bx^{(0)}\right\|^{2} + \eta f(\bx)\right] \label{eq:def_p_star_cert}
	\end{align}
	However, solving the above problem can be hard in general. Using the minimax theorem (primal $\geq$ dual), we can write the \textit{dual} of the above problem as follows:
	\begin{align}
	&p^{*}_{cert} \geq \max_{\eta} d_{cert}(\eta) \nonumber\\ 
	&d_{cert}(\eta) = \min_{\bx} \left[\frac{1}{2}\left\|\bx - \bx^{(0)}\right\|^{2} + \eta f(\bx)\right] \label{eq:def_d_eta}
	\end{align}
	From the theory of duality, we know that $d_{cert}(\eta)$ for each value of $\eta$ gives a lower bound on the exact certification value (the primal solution) $p^{*}_{cert}$. However, since $f$ is non-convex, solving $d_{cert}(\eta)$ for every $\eta$ can be difficult. In the next section, we will prove that the curvature of the function $f$ is bounded globally: 
	\begin{align}
	m\bI \preccurlyeq \nabla^{2}_{\bx} f \preccurlyeq M\bI\qquad \forall \bx \in \mathbb{R}^{D}\label{eq:eig_bound}
	\end{align}
	In this case, we have the following theorem ($d^{*}_{cert}$ is defined in Table \ref{table:cert_attack_gist}):
	\begin{theorem}\label{thm:certificate}
		$d_{cert}(\eta)$ is a convex optimization problem for $ -1/M \leq\eta \leq -1/m$. Moreover, If $\bx^{(cert)}$ is the solution to $d^{*}_{cert}$ such that $f(\bx^{(cert)})=0$, then  $p^{*}_{cert}=d^{*}_{cert}$. 
	\end{theorem}
	Below, we briefly outline the proof while the full proof is presented in Appendix \ref{proof:certificate}. The Hessian of the \textit{objective function} of the dual $d_{cert}(\eta)$, i.e the function inside the $\min_{\bx}$ is given by:
	$$ \nabla_{\bx}^{2} \left[\frac{1}{2}\left\|\bx - \bx^{(0)}\right\|^{2} + \eta f(\bx) \right] = \bI + \eta\nabla^{2}_{\bx} f $$
	From equation \eqref{eq:eig_bound}, we know that the eigenvalues of $\bI + \eta\nabla^{2}_{\bx} f$ are bounded between $(1 + \eta m, 1 + \eta M)$ if $\eta \geq 0$, and in $(1 + \eta M, 1 + \eta m)$ if $\eta \leq 0$. In both cases, we can see that for $-1/M \leq \eta \leq -1/m$, all eigenvalues will be non-negative, making the objective function convex. When $\bx^{(cert)}$ satisfies  $f(\bx^{(cert)})=0$, we have $d^{*}_{cert} = 1/2\|\bx^{(cert)}-\bx^{(0)}\|^{2}$. Using the duality theorem we have $d^{*}_{cert} \leq p^{*}_{cert}$ and from the definition of $p^{*}_{cert}$, we have $p^{*}_{cert} \leq d^{*}_{cert}$. Combining the two inequalities, we get $p^*_{cert} = d^*_{cert}$.
	
	Next, we consider the attack problem. The goal here is to find an adversarial example inside an $l_{2}$ ball of radius $\rho$ such that $f(\bx)$ is minimized. Using similar arguments, we can get the following theorem for the attack problem
	($p^{*}_{attack}$, $d^{*}_{attack}$ and $d_{attack}$ are defined in Table \ref{table:cert_attack_gist}):
	\begin{theorem}\label{thm:attack} $d_{attack}(\eta)$ is a convex optimization problem for $ -m \leq\eta$. Moreover, if $\bx^{(attack)}$ is the solution to $d^{*}_{attack}$ such that $\left\|\bx^{(attack)} - \bx^{(0)}\right\|=\rho$,\  $p^{*}_{attack}=d^{*}_{attack}.$
	\end{theorem} 
	The proof is presented in Appendix \ref{proof:attack}. We note that both Theorems \ref{thm:certificate} and \ref{thm:attack} hold for any non-convex function with continuous gradients. Thus they can also be of interest in problems such as optimization of neural nets.
	
	Using Theorems \ref{thm:certificate} and \ref{thm:attack}, we have the following definitions for certification and attack optimizations:
	\begin{definition}{(\textbf{Curvature-based Certificate Optimization})}\label{def:cert_opt}
		Given an input $\bx^{(0)}$ with true label $y$, false target $t$, we define $(\eta^{(cert)}, \bx^{(cert)})$ as the solution of the following max-min optimization:
		$$\max_{-1/M \leq \eta \leq -1/m} \min_{\bx} \left[\frac{1}{2}\left\|\bx-\bx^{(0)}\right\|^{2} + \eta f(\bx)\right] $$
		We refer to $\left\|\bx^{(cert)} - \bx^{(0)}\right\|$ as the \textbf{C}urvature-based \textbf{R}obustness \textbf{C}ertificate (CRC).
	\end{definition}
	
	\begin{definition}{(\textbf{Curvature-based Attack Optimization})}\label{def:attack_opt}
		Given input $\bx^{(0)}$ with label $y$, false target $t$, and the $l_{2}$ ball radius $\rho$, we define $(\eta^{(attack)}, \bx^{(attack)})$ as the solution of the following optimization:  
		$$ \max_{\eta \geq -m} \min_{\bx} \left[\frac{\eta}{2}\left(\left\|\bx-\bx^{(0)}\right\|^{2}-\rho^{2}\right) + f(\bx)\right] $$
		When $\bx^{(attack)}$ is used for training in an adversarial training framework, we call the method the \textbf{C}urvature-based \textbf{R}obust \textbf{T}raining (CRT).
	\end{definition}
	%Since both curvature-based certificate and attack optimizations are convex optimization problems, any convex optimization solver can be used to solve them. In our implementation, we use majorization-minimization to solve the dual function for a given $\eta$ and bisection method to maximize over $\eta$. More details are given in Appendix \ref{subsec:algorithm_certificate} and \ref{subsec:algorithm_attack}.
	
	A direct implication of Theorems \ref{thm:certificate} and \ref{thm:attack} is that the tightness of our robustness certificate crucially depends on the tightness of our curvature bounds, $m$ and $M$. If $m$ and $M$ are very large compared to the true eigenvalue bounds of the Hessian of the network, the resulting robustness certificate will be vacuous. In Table \ref{table:cert_comparison_short} (and Figure \ref{fig:gamma_effect_2_3}), we show that by adding the derived bound as a regularization term during the training, we can significantly decrease curvature bounds of the network, with little or no decrease in its performance. This leads to high robustness certifications against adversarial attacks.
	
	\section{Curvature Bounds for deep networks}\label{curvature_bounds}
	In this section, we provide a computationally efficient approach to compute the curvature bounds for neural networks with differentiable activation functions. To the best of our knowledge, there is no prior work on finding provable bounds on the curvature values of deep neural networks. %Our results rely on a closed form expression for the Hessian of the $i^{th}$ logit as a sum of matrix products (Section \ref{deeper_networks_hessian_bounds}). After establishing this result, we first derive curvature bounds for a two-layer network in Section \ref{single_hidden_layer_hessian_bounds} and then extend the bounds to deeper networks in Section \ref{deep_hessian_bounds}.
	\subsection{Closed form expression for the Hessian}\label{deeper_networks_hessian_bounds}
	Using the chain rule of second derivatives, we can derive $\nabla^{2}_{\bx} \bz^{(L)}_{i}$ as a sum of matrix products:
	\begin{lemma}\label{thm:deep_hessian_closed}
		Given an $L$ layer neural network, the Hessian of the $i^{th}$ hidden unit with respect to the input $\bx$, i.e $\nabla_{\bx}^{2} \bz^{(L)}_{i}$ is given by the following formula:
		\begin{align}
		\nabla^{2}_{\bx} \bz^{(L)}_{i} = \sum_{I=1}^{L-1}\left(\bB^{(I)}\right)^{T}diag\bigg(\bF^{(L,I)}_{i}\odot\sigma^{''}\left(\bz^{(I)}\right)\bigg)\bB^{(I)}\nonumber%\label{final_formula}
		\end{align}
		where $\bB^{(I)}$ is the Jacobian of $\bz^{(I)}$ with respect to $\bx$ (dimensions $N_{I} \times D$), and $\bF^{(L,I)}$ is the Jacobian of $\bz^{(L)}$ with respect to $\ba^{(I)}$ (dimensions $N_{L} \times N_{I}$). 
		\end{lemma}
	The proof is presented in Appendix \ref{proof:deep_hessian_closed}. Using the chain rule, we can compute $\bB^{(I)}$,\ $\bF^{(L,I)}$ matrices in Lemma \ref{thm:deep_hessian_closed} recursively as follows:\\
    \begin{align*}
    &\bB^{(I)} = \begin{cases}
        \bWW^{(1)}, & I=1\\
        \bWW^{(I)}diag\left(\sigma^{'}\left(\bz^{(I-1)}\right)\right)\bB^{(I-1)}, & I \geq 2\\
    \end{cases}\\
    &\bF^{(L,I)} = \begin{cases}
        \bWW^{(L)},& I=L-1\\
        \bWW^{(L)} diag \left(\sigma^{'}\left(\bz^{(L-1)}\right)\right)\bF^{(L-1,I)},& I\leq L-2
    \end{cases}
    \end{align*}
	This leads to a fast back-propagation like method that can be used to compute the Hessian. Note that Lemma \ref{thm:deep_hessian_closed} only assumes a matrix multiplication operation from $\ba^{(I-1)}$ to $\bz^{(I)}$. Since a convolution operation can also be expressed as a matrix multiplication, we can directly extend this lemma to deep convolutional networks. Furthermore, Lemma \ref{thm:deep_hessian_closed} can also be of independent interest in other related problems such as second-order interpretation methods for deep learning (e.g. \cite{Singla2019UnderstandingIO}).
	\subsection{Curvature bounds for Two Layer networks}\label{single_hidden_layer_hessian_bounds}
	For a two-layer network and using Lemma \ref{thm:deep_hessian_closed}, $\nabla^{2}_{\bx} \left(\bz^{(2)}_{y} - \bz^{(2)}_{t}\right)$ is given by:
	\begin{align}
	\left(\bWW^{(1)}\right)^{T} diag \bigg(\left(\bWW^{(2)}_{y} - \bWW^{(2)}_{t}\right)\odot\sigma^{''}\left(\bz^{(1)}\right)\bigg)\bWW^{(1)}\nonumber
	\end{align}
	In the above equation, note that only the term $\sigma^{''}(\bz^{(1)})$ depends on $\bx$. We can maximize and minimize each element in the diag term, $(\bWW^{(2)}_{y,i} - \bWW^{(2)}_{t,i})\sigma^{''}(\bz^{(1)}_{i})$ independently subject to the constraint that $\sigma^{''}(.)$ is bounded. Using this procedure, we construct matrices $\bP$ and $\bN$ that satisfy properties given in the following theorem:
	\begin{theorem}\label{thm:single_layer_p_n_theorem}
		Given a two layer network whose activation function has bounded second derivative:
		\begin{align}
		h_{L} \leq \sigma^{''}(x) \leq h_{U} \quad  \forall x \in \mathbb{R} \nonumber %\label{scalar_hessian_bound}
		\end{align}
		\begin{enumerate}[label=(\alph*)]
			\item We have the following linear matrix inequalities (LMIs):
			\begin{align*}
			\bN \preccurlyeq \nabla^{2}_{\bx} \left(\bz^{(2)}_{y} - \bz^{(2)}_{t}\right) &\preccurlyeq \bP\qquad \forall \bx \in \mathbb{R}^{D}
			\end{align*}
			\item If $h_{U} \geq 0$ and $h_{L} \leq 0$, $\bP$ is PSD, $\bN$ is a NSD matrix.
			\item This gives the following global bounds on the eigenvalues of the Hessian:
			\begin{align}
			&m\bI \preccurlyeq \nabla^{2}_{\bx}\left(\bz^{(2)}_{y} - \bz^{(2)}_{t}\right) \preccurlyeq M\bI
			\end{align}
			where $M = \lambda_{max}(\bP) ,\ m = \lambda_{min}(\bN)$ \end{enumerate}
		$\bP$ and $\bN$ are independent of $\bx$ and defined in equations \eqref{bP_eq} and \eqref{bN_eq} in Appendix \ref{proof:single_layer_p_n_theorem}.
	\end{theorem} 
	The proof is presented in Appendix \ref{proof:single_layer_p_n_theorem}. Because power iteration finds the eigenvalue with largest magnitude, we can use it to find $m$ and $M$ only when $\bP$ is PSD and $\bN$ is NSD. We solve for $h_{U}$ and $h_{L}$ for sigmoid, tanh, softplus activation functions in Appendix \ref{grad_hess_bounds_appendix} and show that this is in fact the case for them. 
	
	We note that this result does not hold for ReLU networks since the ReLU function is not differentiable everywhere. However, in Appendix $\ref{relu_quad_bound}$, we devise a method to compute the certificate for a two layer ReLU network by finding a quadratic function that is a provable lower bound for $\bz^{(2)}_{y} - \bz^{(2)}_{t}$. We show that the resulting method significantly outperforms CROWN-Ada (see Appendix Table \ref{table:cert_compare_diff_targets}). 
	\subsection{Curvature bounds for Deep networks}\label{deep_hessian_bounds}
	Using Lemma \ref{thm:deep_hessian_closed}, we know that $\nabla_{\bx}^{2} \bz^{(L)}_{i}$ is a sum product of matrices $\bB^{(I)}$ and $\bF^{(L,I)}_{i}$. Thus, if we can find upper bounds for $\|\bB^{(I)}\|$ and $\|\bF^{(L,I)}_{i}\|_{\infty}$, we can get upper bounds for $\|\nabla_{\bx}^{2} \bz^{(L)}_{i}\|$. Using this intuition (proof is presented in Appendix \ref{proof:L_layer_p_n_theorem}), we have the following result:  
	\begin{theorem}\label{thm:L_layer_p_n_theorem}
		Given an $L$ layer neural network whose activation function satifies:
		\begin{align}
		&|\sigma^{'}(x)| \leq g,\ |\sigma^{''}(x)| \leq h \qquad \forall x \in \mathbb{R},\nonumber%\label{grad_hess_bound}
		\end{align}	
		the absolute value of eigenvalues of  $\ \nabla^{2}_{\bx} \bz^{(L)}_{i}$ is globally bounded by the following quantity:
		\begin{align}
		&\left\|\nabla^{2}_{\bx} \bz^{(L)}_{i}\right\| \leq h\sum_{I=1}^{L-1} \left(r^{(I)}\right)^{2}\max_{j}\left(\bS^{(L,I)}_{i,j}\right),\quad \forall \bx \in \mathbb{R}^{D} \nonumber
		\end{align}
		where $r^{(I)}$ and $\bS^{(L,I)}$ are independent of $\bx$ and defined recursively as:
        \begin{align}
        &r^{(I)} = \begin{cases}
            \left\|\bWW^{(1)}\right\|, & I=1\\
            g\left\|\bWW^{(I)}\right\|r^{(I-1)}, & I \geq 2\\
        \end{cases}\label{eq:r_eff}\\
        &\bS^{(L,I)} = \begin{cases}
            \left|\bWW^{(L)}\right|,& I=L-1\\
            g\left|\bWW^{(L)}\right|\bS^{(L-1,I)},& I\leq L-2
        \end{cases}\label{eq:S_eff}
        \end{align}
	\end{theorem}
	The above expressions allows for an extremely efficient computation of the curvature bounds for deep networks. We consider simplification of this result for sigmoid, tanh, softplus activations in Appendix \ref{grad_hess_bounds_appendix}. 
	The curvature bounds for $\bz^{(L)}_{y}-\bz^{(L)}_{t}$ can be computed by replacing $\bWW^{(L)}_{i}$ with $\bWW^{(L)}_{y} - \bWW^{(L)}_{t}$ in Theorem \ref{thm:L_layer_p_n_theorem}. The resulting bound is independent of $\bx$, and only depends on network weights $\bWW^{(I)}$, the true label $y$, and the target $t$. We denote it with $K(\bWW, y, t)$.\ To simplify notation, when no confusion arises we denote it with $K$. In our experiments, for two layer networks, we use $M$, $m$ from Theorem \ref{thm:single_layer_p_n_theorem} since it provides tighter curvature bounds. For deeper networks ($L \geq 3$), we use $M=K,\ m=-K$. 

	\begin{figure*}
		\centering
		\begin{subfigure}{.5\textwidth}
			\centering
			\includegraphics[width=0.9\textwidth]{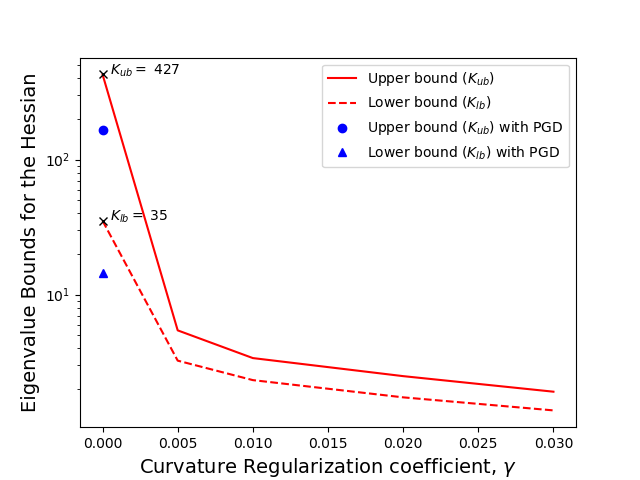}
			\caption{2 layer network}
			\label{fig:sub1}
		\end{subfigure}%
		\begin{subfigure}{.5\textwidth}
			\centering
			\includegraphics[width=.9\linewidth]{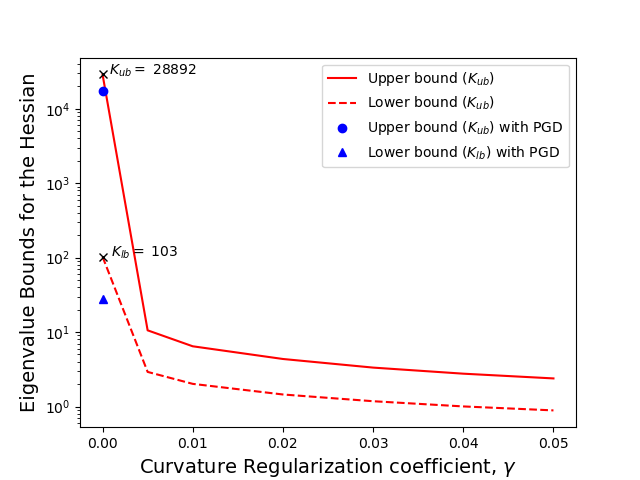}
			\caption{3 layer network}
			\label{fig:sub2}
		\end{subfigure}
		\caption{Illustration of lower ($K_{lb}$) and upper ($K_{ub}$) bounds on the curvature of 2 and 3 layer networks with sigmoid activations trained on MNIST. Without any curvature regularization ($\gamma=0$), curvature bounds increase significantly for deeper networks. Similarly with $\gamma=0$, networks adversarially trained with PGD have high curvature as well (note the log scale of the $y$-axis). However, using our curvature bound as a regularizer, the bound becomes tight and CRC gives high certificate values (Table \ref{table:cert_comparison_short}). We report the curvature bounds ($K_{lb}$ and $K_{ub}$) for networks with different depths in Appendix Table \ref{table:hyperparam_search}.}
				\label{fig:gamma_effect_2_3}
	\end{figure*}

	\section{Adversarial training with curvature regularization} \label{adversarial_training}
	Since the term $\bB^{(I)}$ in Lemma \ref{thm:deep_hessian_closed} is the Jacobian of $\bz^{(I)}$ with respect to $\bx$,\ $\|\bB^{(I)}\|$, it is equal to the lipschitz constant of the neural network constructed from the first $I$ layers of the original network. Finding tight bounds on the lipschitz constant is an active area of research \cite{DBLP:journals/corr/abs-1906-04893, Weng2018TowardsFC, scaman2018lipschitz} and the product of the operator norm of weight matrices is known to be a loose bound on the lipschitz constant for deep networks. Since we use the same product to compute the bound for $\|\bB^{(I)}\|$ in Theorem \ref{thm:L_layer_p_n_theorem}, the resulting curvature bound is likely to be loose for very deep networks. 

	In Figure \ref{fig:gamma_effect_2_3}, we observe the same trend: as the depth of the network increases, the upper bound $K_{ub}$ computed using Theorem \ref{thm:L_layer_p_n_theorem} becomes significantly larger than the lower bound $K_{lb}$ (computed by taking the maximum of the largest eigenvalue of the Hessian across all test images with label $y$ and the second largest logit $t$, then averaging across different $(y, t)$). However, by regularizing the network to have small curvature during training, the bound becomes significantly tighter. Interestingly, using curvature regularization, even with this loose curvature bound for deep nets, we achieve significantly higher robust accuracy than the current state of the art methods while enjoying significantly higher standard accuracy as well (see Tables \ref{table:empirical_provable_adversarial_1.58_summary} and \ref{table:empirical_provable_adversarial_1.58_fashion_mnist}). 
	%Using Theorem \ref{thm:attack} (b), we know that if we solve the curvature-based attack optimization (Definition \ref{def:attack_opt}) and obtain $\rho=\|\bx^{(attack)}-\bx^{(0)}\|$, then $\bx^{(attack)}$ is provably the closest adversarial example (in $l_{2}$ distance) to $\bx^{(0)}$. However, in our experiments we found that almost none of training inputs lead to zero primal-dual gap with $\rho=\|\bx^{(attack)}-\bx^{(0)}\|$ (as shown in Table \ref{table:primal_dual_eq_fraction_summary}).
	
	To regularize the network to have small curvature values, we penalize the curvature bound $K$ during training. To compute the gradient of $K$ with respect to the network weights, note that using Theorem \ref{thm:L_layer_p_n_theorem}, we can compute $K$ using absolute value, matrix multiplications, and operator norm. Since the gradient of operator norm does not exist in standard libraries, we created a new layer where the gradient of $\|\bWW^{(I)}\|$  i.e $\nabla_{\bWW^{(I)}} \|\bWW^{(I)}\|$ is given by:
	\begin{align*}
	&\nabla_{\bWW^{(I)}} \|\bWW^{(I)}\| =  \bu^{(I)}\left(\bv^{(I)}\right)^{T}\\ 
	&\text{ where } \ \ \bWW^{(I)}\bv^{(I)} = \|\bWW^{(I)}\|\bu^{(I)}
	\end{align*}
	Note that $\|\bWW^{(I)}\|$, $\bu^{(I)}$ and $\bv^{(I)}$ can be computed using power iteration. Since the network weights do not change significantly during a single training step, we can use the singular vectors $\bu^{(I)}$ and $\bv^{(I)}$ computed in the previous training step to update $\bWW^{(I)}$ using one iteration of power method. This approach to compute the gradient of the largest singular value of a matrix has also been used in previous published work \cite{miyato2018spectral}. 
	Thus, the per-sample loss for training with curvature regularization is given by:
	\begin{align}
	\ell\left(\bz^{(L)}(\bx^{(0)}),\ y\right) + \gamma K(\bWW,\ y,\ t) \label{eq:loss_curv_reg}
	\end{align}
	where $\ell$ denotes the cross entropy loss, $y$ is the true label of the input $\bx^{(0)}$, $t$ is the attack target and $\gamma$ is the regularization coefficient for penalizing large curvature values. Similar to the adversarial training, in CRT, we use $\bx^{(attack)}$ instead of $\bx^{(0)}$ in equation \eqref{eq:loss_curv_reg}.
	
	\section{Experiments} \label{experiments}
	The {\it certified robust accuracy} means the fraction of correctly classified test samples whose robustness certificates (computed using CRC) are greater than a pre-specified radius $\rho$. Unless otherwise specified, we use the class with the second largest logit as the attack target (i.e. the class $t$). The notation ($L \times [1024]$, activation) denotes a neural network with $L$ layers with the specified activation, ($\gamma=c$) denotes standard training with $\gamma$ set to $c$, while (CRT, $c$) denotes CRT training with $\gamma=c$. Certificates are computed over 150 randomly chosen correctly classified images. We use a single NVIDIA GeForce RTX 2080 Ti GPU.

	\subsection{Fraction of inputs with tightest robustness certificate}
    Using the verifiable condition of Theorems \ref{thm:certificate} and \ref{thm:attack}, our approach is able to (1) find points that are provably the worst case adversarial perturbations (in the $l_{2}$ norm) in the attack problem and (2) find points on the decision boundary that are provably closest to the input in the $l_{2}$ norm in the certification problem. In particular, in Table \ref{table:primal_dual_eq_fraction_summary}, we observe that for curvature regularized networks, our approach finds provably worst-case adversarial perturbations for {\it all} of the inputs with a small drop in the accuracy. Moreover, for 2,3,and 4 layer networks, our method finds provably closest adversarial examples for 44.17\%, 22.59\% and 19.53\% of inputs in the MNIST test set, respectively.
    
    \begin{table}[h!]
%        \vskip -0.1in
		\centering
		\renewcommand{\arraystretch}{1.15} 
		\caption{Certificate success rate denotes the fraction of points satisfying $\bz_{y}-\bz_{t}=0$, Attack success rate denotes the fraction of points ($\bx^{(0)})$ satisfying $\|\bx^{(attack)} - \bx^{(0)}\|_{2} = \rho=0.5$ implying \textit{primal}=\textit{dual} in Theorems \ref{thm:certificate} and \ref{thm:attack} respectively. We use the MNIST dataset.}
		\begin{tabular}{ | l | l |  l | l | l | } 
			\hline
			\multicolumn{1}{|c|}{\multirow{2}{*}{Network}} & \multicolumn{1}{c|}{\multirow{2}{*}{$\gamma$}} & \multicolumn{1}{c|}{\multirow{2}{*}{Accuracy}} & \multirow{2}{3em}{Attack success} & \multirow{2}{4em}{Certificate success} \\ 
			& & & & \\
			\hline
			\multirow{2}{4em}{$2\times[1024]$, sigmoid} & 0. & 98.77\%  & 5.05\% & 2.24\% \\ 
			\cline{2-5}
			& 0.03 & 98.30\% & 100\% & 44.17\%\\ 
			\hline
			\multirow{2}{4em}{$3\times[1024]$, sigmoid} & 0. & 98.52\% & 0.\% & 0.12\% \\ 
			\cline{2-5}
			& 0.05 & 97.60\% & 100\% & 22.59\%\\ 
			\hline
			\multirow{2}{4em}{$4\times[1024]$, sigmoid} & 0. & 98.22\% & 0.\% & 0.01\%\\ 
			\cline{2-5}
			& 0.07 & 95.24\% & 100\% & 19.53\% \\ 
			\hline
		\end{tabular}
		\label{table:primal_dual_eq_fraction_summary}
	\end{table}	

	Note that the technique presented in this work is not applicable to ReLU networks due to the absence of curvature information. However, since verifying the robustness property in an $l_{2}$ ball around the input is known to be an NP-complete problem for ReLU networks \cite{Katz2017ReluplexAE}, it is computationally infeasible to even verify that a given adversarial perturbation is the worst case perturbation in polynomial time unless P=NP. We however show that it is possible to find (and not just verify) the exact worst case perturbation (and robustness certificate) for neural networks with smooth activation functions. We believe that these theoretical and empirical results provide evidence that bounding the curvature of the network and using smooth activation functions can be critical to achieve high robustness guarantees.
	
    \subsection{Comparison with existing provable defenses} 

	We compare against certified defense techniques proposed in \citet{Wong2018ScalingPA} and \citet{ZhangCROWNIBP} in Table \ref{table:empirical_provable_adversarial_1.58_summary} for the MNIST dataset \cite{lecunMNIST} and Table \ref{table:empirical_provable_adversarial_1.58_fashion_mnist} for the Fashion-MNIST dataset \cite{fashionMNIST} with $l_{2}$ radius $\rho=1.58$. Even though our proposed CRT requires fully differentiable activation functions such as softplus, sigmoid, tanh etc, we include comparison with ReLU networks because the methods proposed in \citet{Wong2018ScalingPA, ZhangCROWNIBP} use ReLU. Since CROWN-IBP can be trained using the softplus activation function, we include it in our comparison. Similar comparison with $l_{2}$ radius $\rho=0.5$ is given in Appendix Table \ref{table:empirical_provable_adversarial_0.5} (MNIST dataset) and Table \ref{table:empirical_provable_adversarial_0.5_fashion_mnist} (Fashion-MNIST dataset). We observe that CRT (certified with CRC) gives significantly higher certified robust accuracy as well as standard accuracy compared to either of the methods on both MNIST and Fashion-MNIST datasets for both different values of $\rho$. Since shallow fully connected networks are known to perform poorly on the CIFAR-10 dataset, we do not include those results in our comparison. 
	\begin{table}[h!]
%        \vskip -0.2in
		\centering
		\renewcommand{\arraystretch}{1.15}
		\caption{Comparison with interval-bound propagation based adversarial training methods: COAP i.e Convex Outer Adversarial Polytope \cite{Wong2018ScalingPA}, CROWN-IBP \cite{ZhangCROWNIBP}) and Curvature-based Robust Training (Ours) with attack radius $\rho=1.58$ on MNIST. For CROWN-IBP, we vary the final\_beta hyperparameter between 0.8 and 3, and use the model with best certified accuracy. Results with $\rho=0.5$ are in Appendix Table \ref{table:empirical_provable_adversarial_0.5}.}
		\begin{tabular}{ | l | l | l | l | } 
			\hline
			\multicolumn{1}{|c|}{\multirow{3}{4.5em}{Network}} & \multicolumn{1}{c|}{\multirow{3}{*}{Training}} & \multicolumn{1}{c|}{\multirow{3}{3.5em}{Standard Accuracy}} & 
			\multicolumn{1}{c|}{\multirow{3}{3.5em}{Certified Robust Accuracy}} \\ 
			& & & \\ 
			& & & \\ 
			\hline
			\multirow{2}{4em}{$2\times[1024]$, softplus} & \multirow{1}{4.5em}{\textbf{CRT, 0.01}} & \multirow{1}{4em}{\textbf{98.68\%}} & \multirow{1}{4em}{\textbf{69.79\%}} \\ 
			\cline{2-4}
			& CROWN-IBP & 88.48\% & 42.36\% \\ 
			\hline 
			\multirow{2}{4em}{$2\times[1024]$, relu} &  COAP & 89.33\% & 44.29\% \\ 
			\cline{2-4}
			& CROWN-IBP & 89.49\% & 44.96\% \\ 
			\hline 
			\hline
			\multirow{2}{4em}{$3\times[1024]$, softplus} &  \multirow{1}{*}{\textbf{CRT, 0.05}} &  \multirow{1}{*}{\textbf{97.43\%}} &  \multirow{1}{*}{\textbf{57.78\%}} \\ 
			\cline{2-4}
			& CROWN-IBP & 86.58\% & 42.14\% \\ 
			\hline 
			\multirow{2}{4em}{$3\times[1024]$, relu} &  COAP  & 89.12\% & 44.21\% \\ 
			\cline{2-4}
			& CROWN-IBP & 87.77\% & 44.74\% \\ 
			\hline 
			\hline
			\multirow{2}{4em}{$4\times[1024]$, softplus} & \multirow{1}{*}{\textbf{CRT, 0.07}} & \multirow{1}{*}{\textbf{95.60\%}} & \multirow{1}{*}{\textbf{53.19\%}} \\ 
            \cline{2-4}
			& CROWN-IBP & 82.74\% & 41.34\% \\ 
			\hline 
			\multirow{2}{4em}{$4\times[1024]$, relu} & COAP & 90.17\% & 44.66\% \\ 
			\cline{2-4}
			& CROWN-IBP & 84.4\% & 43.83\% \\ 
			\hline 
		\end{tabular}
		\label{table:empirical_provable_adversarial_1.58_summary}
	\end{table}

    \begin{table}[h!]
%        \vskip -0.2in
		\centering
		\renewcommand{\arraystretch}{1.15}
		\caption{Comparison between COAP\ \cite{Wong2018ScalingPA}, CROWN-IBP \cite{ZhangCROWNIBP}) and Curvature-based Robust Training (Ours) with attack radius $\rho=1.58$ on Fashion-MNIST. Results with $\rho=0.5$ for are in Appendix Table \ref{table:empirical_provable_adversarial_0.5_fashion_mnist}. }
		\begin{tabular}{ | l | l |l | l | } 
			\hline
			\multicolumn{1}{|c|}{\multirow{3}{4.5em}{Network}} & \multicolumn{1}{c|}{\multirow{3}{*}{Training}} &		\multicolumn{1}{c|}{\multirow{3}{3.5em}{Standard Accuracy}} & 
			\multicolumn{1}{c|}{\multirow{3}{3.5em}{Certified Robust Accuracy}} \\ 
			& & & \\ 
			& & & \\ 
			\hline
			\multirow{2}{4em}{$2\times[1024]$, softplus} & 
			\multirow{1}{*}{\textbf{CRT, 0.01}} &
			\multirow{1}{4em}{\textbf{80.31\%}} & \multirow{1}{4em}{\textbf{54.39\%}} \\ 
			\cline{2-4}
			& CROWN-IBP & 69.23\% & 47.19\% \\ 
			\hline
			\multirow{2}{4em}{$2\times[1024]$, relu} & COAP & 74.1\% & 46.3\% \\ 
			\cline{2-4}
			& CROWN-IBP & 70.73\% & 48.61\% \\ 
			%		& & & \\
			\hline 
			\hline 
			\multirow{2}{4em}{$3\times[1024]$, softplus} &
			\multirow{1}{*}{\textbf{CRT, 0.05}} &
			\multirow{1}{4em}{\textbf{78.39\%}} & \multirow{1}{4em}{\textbf{53.4\%}} \\ 
			\cline{2-4}
			& CROWN-IBP & 68.72\% & 46.52\% \\
			\hline
			\multirow{2}{4em}{$3\times[1024]$, relu} &  COAP & 73.9\% & 46.3\% \\ 
			\cline{2-4}
			& CROWN-IBP & 70.79\% & 48.69\% \\ 
			\hline 
			\hline 
			\multirow{2}{4em}{$4\times[1024]$, softplus} & 
			\multirow{1}{*}{\textbf{CRT, 0.07}} &
			\multirow{1}{4em}{\textbf{75.61\%}} & \multirow{1}{4.5em}{\textbf{49.6\%}} \\ 
			\cline{2-4}
			& CROWN-IBP & 68.31\% & 46.21\%\\
			\hline 
			\multirow{2}{4em}{$4\times[1024]$, relu} &  COAP & 73.6\% & 45.1\% \\ 
			\cline{2-4}
			& CROWN-IBP & 70.21\% & 48.08\% \\ 
			\hline 
		\end{tabular}
		\vskip -0.4cm
		\label{table:empirical_provable_adversarial_1.58_fashion_mnist}
	\end{table}

	In Appendix Table \ref{table:smoothing_adversarial}, we compare CRT with Randomized Smoothing \cite{Cohen2019CertifiedAR}. For 2 \& 3 layer networks, we achieve higher robust accuracy. However, we note that since our certificate is deterministic while the smoothing-based certificate is probabilistic (although with high probability), the results are not directly comparable. 	As a separate result, we also prove that randomized smoothing bounds the curvature of the network (Theorem \ref{thm:rand_smooth} in Appendix \ref{proof:rand_smooth}). We also include comparison with empirical defense methods namely PGD and TRADES in Appendix Table \ref{table:empirical_adversarial_appendix}. 
    \begin{table}[h]
%    \vskip -0.1in
    \caption{Effect of curvature regularization and CRT on certified robust accuracy and robustness certificate}
    \begin{subtable}[h]{0.45\textwidth}
        \centering
		\renewcommand{\arraystretch}{1.15}
		\begin{tabular}{ | l | l | l | l | } 
		    %\caption{Effect of \gamma on certified robust accuracy}
			\hline
			\multicolumn{1}{|c|}{\multirow{3}{4em}{Network}} & \multicolumn{1}{c|}{\multirow{3}{4em}{Training}} & \multicolumn{1}{c|}{\multirow{3}{4em}{Standard Accuracy}} & \multicolumn{1}{c|}{\multirow{3}{4em}{Certified Robust Accuracy}} \\ 
			& & & \multicolumn{1}{c|}{} \\ 
			& & & \\ 
			\hline			
			\multirow{3}{4em}{$2\times[1024]$, sigmoid} &  \multirow{1}{*}{standard} & 98.37\% & 54.17\%\\ 
			\cline{2-4}
			&  $\gamma=0.01$ & 98.08\% & 83.53\%\\ 
			\cline{2-4}
			&  \textbf{CRT, 0.01} & 98.57\% & \textbf{95.59\%}\\ 
			\hline
			\multirow{3}{4em}{$3\times[1024]$, sigmoid} & standard & 98.37\% & 0.00\%\\ 
			\cline{2-4} 
			&  $\gamma=0.01$ & 97.71\% & 88.33\%\\ 
			\cline{2-4}
			&  \textbf{CRT, 0.01} & 97.23\% & \textbf{94.99\%}\\ 
			\hline 
			\multirow{3}{4em}{$4\times[1024]$, sigmoid} & standard & 98.39\%  & 0.00\%\\ 
			\cline{2-4} 
			&  $\gamma=0.01$ & 97.41\% & 89.61\%\\ 
			\cline{2-4}
			&  \textbf{CRT, 0.01} & 97.83\% & \textbf{93.41\%} \\ 
			\hline 
		\end{tabular}
        \caption{Effect of $\gamma$ on certified robust accuracy}
        \label{tab:week1}
    \end{subtable}
    \hfill
    \begin{subtable}[h]{0.45\textwidth}
        \centering
		\renewcommand{\arraystretch}{1.15}
		\begin{tabular}{ | l | l | l | l | } 
			\hline
			\multicolumn{1}{|c|}{\multirow{2}{4em}{Network}} & \multicolumn{1}{c|}{\multirow{2}{4em}{Training}} & 
			\multicolumn{2}{c|}{\multirow{1}{8em}{Certificate (mean)}}\\
			\cline{3-4}
			%\multicolumn{1}{c|}{\multirow{1}{4em}{CROWN}} & \multicolumn{1}{c|}{\multirow{1}{4em}{CRC}} \\ 
			& & \multirow{1}{4em}{CROWN} & \multirow{1}{4em}{CRC} \\
%			& & & \\ 
			\hline			
			\multirow{3}{4em}{$2\times[1024]$, sigmoid} &  \multirow{1}{*}{standard} & 0.28395 & \textbf{0.48500}\\
			\cline{2-4}
			&  $\gamma=0.01$ & 0.32548 & \textbf{0.84719}\\ 
			\cline{2-4}
			&  \textbf{CRT, 0.01} & 0.43061 & \textbf{1.54673}\\ 
			\hline
			\multirow{3}{4em}{$3\times[1024]$, sigmoid} & standard & \textbf{0.24644} & 0.06874\\ 
			\cline{2-4} 
			&  $\gamma=0.01$ & 0.39799 & \textbf{1.07842}\\ 
			\cline{2-4}
			&  \textbf{CRT, 0.01} & 0.39799 & \textbf{1.07842}\\ 
			\hline 
			\multirow{3}{4em}{$4\times[1024]$, sigmoid} & standard & \textbf{0.19501} & 0.00454\\ 
			\cline{2-4} 
			&  $\gamma=0.01$ & 0.40620 & \textbf{1.05323}\\ 
			\cline{2-4}
			&  \textbf{CRT, 0.01} & 0.40327 & \textbf{1.06208} \\ 
			\hline 
		\end{tabular}
        \caption{Comparison between CROWN-general\ \cite{NIPS2018_7742} and CRC.}
    \end{subtable}
    \vskip -0.4cm
    \label{table:cert_comparison_short}
\end{table}
	\subsection{Comparison with existing certificates}
	In Table \ref{table:cert_comparison_short}, we compare CRC with CROWN-general \cite{NIPS2018_7742}. For 2-layer networks, CRC outperforms CROWN significantly. For deeper networks, CRC works better than CROWN when the network is trained with curvature regularization. However, with small $\gamma=0.01$, we see a significant increase in CRC but a very small drop in the test accuracy (without any adversarial training). We can see that with $\gamma=0.01$, non-trivial certified accuracies of $83.53\%,\ 88.33\%,\ 89.61\%$ can be achieved on $2,3,4$ layer sigmoid networks, respectively, without any adversarial training. Adversarial training using CRT further boosts certified accuracy to $95.59\%,\ 94.99\%$ and $93.41\%$, respectively. We show some results on CIFAR-10 dataset in Appendix Table \ref{table:cifar_results}. We again observe improvements in the robustness certificate and certified robust accuracy using CRC and CRT. 

	\subsection{Results using local curvature bounds}\label{subsec:main_local_bounds}
	From Theorems \ref{thm:certificate} and \ref{thm:attack}, we can observe that if the curvature is {\it locally} bounded within a convex region around the input (we call it the "safe" region), then the corresponding dual problems ($d^{*}_{cert}$, $d^{*}_{attack}$) are again convex optimization problems provided the optimization trajectory does not escape the safe region. 
	
	Theorem \ref{thm:single_layer_p_n_theorem} can be directly extended to compute the local curvature bound using bounds on the second derivatives, i.e. $\sigma^{''}(\bz^{(1)})$ in the local region. In Table \ref{table:local_global}, we show significant improvements for the CRC certificate for two-layer sigmoid networks on the MNIST dataset for $\gamma=0$. However, with the curvature regularization, the difference is insignificant. We also observe that the certified accuracy for (CRT, 0.0) improves from 95.04\% to 95.31\% and for standard improves from 54.17\% to 58.06\%. The certified accuracy remains the same for other cases. Implementation details are in the Appendix Section \ref{appendix_local_global}.

	Computing local curvature bounds for deeper networks, however, is more challenging due to the presence of terms involving multiplication of first and second derivatives. A straightforward extension of Theorem \ref{thm:L_layer_p_n_theorem} \ref{thm:L_layer_p_n_theorem}, wherein we compute the upper bound on $\sigma^{'}$ and $\sigma{''}$ in a local region around the input across all neurons in all layers does not yield significant improvements over the global method, therefore we do not include those results in our comparison.

	\begin{table}[h!]
	    \vskip -0.4cm
		\centering
		\renewcommand{\arraystretch}{1.20}
		\caption{Comparison between Certified Robust accuracy and CRC for 2 layer sigmoid and tanh networks using global and local curvature bounds on MNIST dataset with $\rho=0.5$}
		\begin{tabular}{ | l | l | l | l | } 
			\hline
			\multicolumn{1}{|c|}{\multirow{2}{*}{Network}} & \multicolumn{1}{c|}{\multirow{2}{*}{Training}} & \multicolumn{1}{c|}{\multirow{2}{5em}{CRC (Global)}} & 
			\multicolumn{1}{c|}{\multirow{2}{5em}{CRC (Local)}}  \\ 
			& & & \\ 
			\hline
			\multirow{4}{4em}{$2\times[1024]$, sigmoid} & standard & 0.5013 & \textbf{0.5847} \\%95.04\% & \textbf{95.31\%}\\
            \cline{2-4}
             & CRT, $0.0$ & 1.0011 & \textbf{1.1741} \\%95.04\% & \textbf{95.31\%}  \\ 
			\cline{2-4} 
			& CRT, $0.01$ & 1.5705 & \textbf{1.6047} \\%& 95.59\% & 95.59\% \\
			\cline{2-4} 
			& CRT, $0.02$ & 1.6720 & \textbf{1.6831} \\%& 95.21\% & 95.21\% \\ \\ 
			\hline 
		\end{tabular}
		\label{table:local_global}
    \vskip -0.4cm
	\end{table}

	\section{Conclusion}\label{conclusion}
    In this paper, we develop computationally-efficient convex relaxations for robustness certification and adversarial attack problems given the classifier has a bounded curvature. We also show that this convex relaxation is tight under some general verifiable conditions. To be able to use proposed certification and attack convex optimizations, we derive global curvature bounds for deep networks with differentiable activation functions. This result is a consequence of a closed-form expression that we derive for the Hessian of a deep network. Adversarial training using our attack method coupled with curvature regularization results in a significantly higher certified robust accuracy than the existing provable defense methods. Our proposed curvature-based robustness certificate significantly outperforms the CROWN certificate when trained with our regularizer. Scaling up our proposed curvature-based robustness certification and training methods as well as further tightening the derived curvature bounds are among interesting directions for the future work. 
    
	{\bf \LARGE Appendix}
	\appendix
	\section{Related work}\label{sec:appendix_related_work}
	Many defenses have been proposed to make neural networks robust against adversarial examples. These methods can be classified into empirical defenses which empirically seem to be robust against known adversarial attacks, and certified defenses, which are provably robust against such attacks. 
	
	\textbf{Empirical defenses} The best known empirical defense is adversarial training \cite{Kurakin2016AdversarialML, madry2018towards, Zhang2019TheoreticallyPT}. In this method, a neural network is trained to minimize the worst-case loss over a region around the input. Although such defenses seem to work on existing attacks, there is no guarantee that a more powerful attack would not break them. In fact, most such defenses proposed in the literature were later broken by stronger attacks \cite{Athalye2018ObfuscatedGG, Carlini:2017:AEE:3128572.3140444, Uesato2018AdversarialRA, Athalye2018OnTR}. To end this arms race between defenses and attacks, a number of works have tried to focus on certified defenses that have formal robustness guarantees.
	
	\textbf{Certified defenses} A classifier is said to be certifiably robust if one can easily obtain a guarantee that a classifier's prediction remains constant within some region around the input. Such defenses typically rely on certification methods which are either exact or conservative. Exact methods report whether or not there exists a adversarial perturbation inside some $l_{p}$ norm ball. In contrast, conservative methods either certify that no adversarial perturbation exists or decline to make a certification; they may decline even when no such perturbation exists.
	Exact methods are usually based on Satisfiability Modulo Theories \cite{Huang2016SafetyVO, Katz2017ReluplexAE, Ehlers2017FormalVO, Carlini2017ProvablyMA} and Mixed Integer linear programming \cite{Cheng2017MaximumRO, Lomuscio2017AnAT, Dutta2018OutputRA, Fischetti2018DeepNN, Bunel2017AUV}. Unfortunately, they are computationally inefficient and
	difficult to scale up to even moderately sized neural networks. In contrast, conservative methods are more scalable and efficient which makes them useful for building certified defenses \cite{Wang2018MixTrainST, Wong2017ProvableDA, Wang2018EfficientFS, Wong2018ScalingPA, Raghunathan2018SemidefiniteRF, Raghunathan2018CertifiedDA, Dvijotham2018TrainingVL, Dvijotham2018ADA, Croce2018ProvableRO, Singh2018FastAE, Gowal2018OnTE, Gehr2018AI2SA, Mirman2018DifferentiableAI, Zhang2018EfficientNN, Weng2018TowardsFC}. However, even these methods have not been shown to scale to practical networks that are large and expressive enough to perform well on ImageNet, for example. To scale to such large networks, randomized
	smoothing has been proposed as a \textit{probabilistically} certified defense.
	
	\textbf{Randomized smoothing} Randomized smoothing was previously proposed by several works \cite{Liu2017TowardsRN, Cao2017MitigatingEA} as a empirical defense without any formal guarantees. \cite{Lcuyer2018CertifiedRT} first proved
	robustness guarantees for randomized smoothing classifier using inequalities from differential privacy. \cite{Li2018CertifiedAR} improved upon the same using tools from information theory. Recently, \cite{Cohen2019CertifiedAR} provided a even tighter robustness guarantee for randomized smoothing. \cite{Salman2019ProvablyRD} proposed a method of adversarial training for the randomized smoothing classifier giving state of the art results in the $l_{2}$ norm metric.

	\section{The Attack problem} \label{attack_problem}
	
	For a given input $\bx^{(0)}$ with true label $y$ and attack target $t$, consider the attack problem. We are given that the eigenvalues of the Hessian $\nabla^{2}_{\bx}(\bz^{(L)}_{y} - \bz^{(L)}_{t}) $ are bounded below i.e: 
	$$m\bI \preccurlyeq \nabla^{2}_{\bx} \left(\bz^{(L)}_{y} - \bz^{(L)}_{t}\right) \qquad \forall \bx \in \mathbb{R}^{D}$$
	Here $m<0$ (since $\bz^{(L)}_{y} - \bz^{(L)}_{t}$ is not convex in general). 
	
	The goal here is to find an adversarial example inside a $l_{2}$ ball of radius $\rho$ such that $(\bz^{(L)}_{y}-\bz^{(L)}_{t})(\bx)$ is minimized. That is, we want to solve the following optimization:
	\begin{align}
	&p^{*}_{attack} = \min_{\left\|\bx - \bx^{(0)}\right\| \leq \rho} \bigg[\left(\bz^{(L)}_{y} - \bz^{(L)}_{t}\right)(\bx)\bigg] \nonumber\\
	& = \min_{\bx} \max_{\eta \geq 0} \bigg[\left(\bz^{(L)}_{y} - \bz^{(L)}_{t}\right)(\bx) + \frac{\eta}{2}\left(\left\|\bx - \bx^{(0)}\right\|^{2} - \rho^{2}\right)\bigg] \label{eq:pstar_attack_def}
	\end{align}
	This optimization can be hard in general. Using the max-min inequality (primal $\geq$ dual), we have:
	\begin{align}
	p^{*}_{attack} &\geq \max_{\eta \geq 0} d_{attack}(\eta)\nonumber\\ 
	d_{attack}(\eta) &= \min_{\bx} \bigg[\left(\bz^{(L)}_{y} - \bz^{(L)}_{t}\right)(\bx) \nonumber\\
	&+ \frac{\eta}{2}\left(\left\|\bx - \bx^{(0)}\right\|^{2} - \rho^{2}\right)\bigg] \label{eq:d_attack_def}
	\end{align}
	We know that for every $\eta \geq 0,\ d_{attack}(\eta)$ gives a lower bound to the primal solution $p^{*}_{attack}$. But solving $d_{attack}(\eta)$ for any $\eta \geq 0$ can be hard unless the objective is convex. We prove that if the eigenvalues of the Hessian are bounded below i.e: 
	\begin{align}
	m\bI \preccurlyeq \nabla^{2}_{\bx} \left(\bz^{(L)}_{y} - \bz^{(L)}_{t}\right) \qquad \forall \bx \in \mathbb{R}^{D}\nonumber
	\end{align}
	In general $m < 0$, since $(\bz^{(L)}_{y} - \bz^{(L)}_{t})$ is non-convex. \\
	$d_{attack}(\eta)$ is a convex optimization problem for $-m \leq \eta$. Equivalently the objective function, i.e the function inside the $\min_{\bx}$:
	\begin{align}
	&\bigg[\left(\bz^{(L)}_{y} - \bz^{(L)}_{t}\right)(\bx) + \frac{\eta}{2}\left(\left\|\bx - \bx^{(0)}\right\|^{2} - \rho^{2}\right)\bigg] \nonumber
	\end{align}
	is a convex function in $\bx$ for $ -m \leq \eta$.\\
	The Hessian of the above function is given by:
	\begin{align}
	&\nabla^{2}_{\bx} \left(\bz^{(L)}_{y} - \bz^{(L)}_{t}\right) + \eta \bI \nonumber
	\end{align}
	
	Since we know that eigenvalues of $\nabla^{2}_{\bx} (\bz^{(L)}_{y} - \bz^{(L)}_{t}) \succcurlyeq  m\bI$, we know that eigenvalues of the above Hessian are 
	$\geq \eta + m$. For $\eta \geq -m$, the eigenvalues are positive implying that the objective function is convex.
	
	Since $d_{attack}(\eta)$ gives a lower bound to $p^{*}_{attack}$ for every $\eta \geq 0$, we get the following result: 
	\begin{align}
	p^{*}_{attack} \geq d^{*}_{attack} \text{ where } d^{*} _{attack} = \max_{-m \leq \eta}d_{attack}(\eta)\label{eq:dstar_attack_def}
	%\min_{\bx} \bigg[\frac{1}{2}\left\|\bx - \bx^{(0)}\right\|^{2} + \eta \left(\bz^{(L)}_{y} - \bz^{(L)}_{t}\right)(\bx)\bigg]\nonumber
	\end{align}
	Note that if $\bx^{(attack)}$ is the solution to $d^{*}_{attack}$ such that: $\left\|\bx^{(attack)} - \bx^{(0)}\right\|=\rho$, by the definition of $d^{*}_{attack}$:
	$$ d^{*}_{attack} = \left(\bz^{(L)}_{y} - \bz^{(L)}_{t}\right)(\bx^{(attack)}) $$
	But then by the definition of $p^{*}_{attack},\ p^{*}_{attack} \leq d^{*}_{attack}$, implying that the duality gap is zero, i.e $p^{*}_{attack}=d^{*}_{attack}$. This procedure leads to the Theorem \ref{thm:attack}.
	
	\section{Implementation Details} \label{implementation}
	\subsection{Computing the derivative of largest singular value} \label{subsec:singular_value_compute}
	Our objective is to compute derivative of the largest singular value, i.e $\|\bWW^{(I)}\|$ with respect to $\bWW^{(I)}$. Let $\bu^{(I)}, \bv^{(I)}$ be the singular vectors such that $\bWW^{(I)}\bv^{(I)} = \|\bWW^{(I)}\|\bu^{(I)}$. Then the derivative is given by:
	\begin{align*}
	\nabla_{\bWW^{(I)}} \|\bWW^{(I)}\| =  \bu^{(I)}\left(\bv^{(I)}\right)^{T}
	\end{align*}
	$\bv^{(I)},\ \|\bWW^{(I)}\|^{2}$ can be computed by running power iteration on $\left(\bWW^{(I)}\right)^{T}\bWW^{(I)}$. $\bu^{(I)}$ can be computed using the identity:
	$$ \bu^{(I)} = \frac{\bWW^{(I)}\bv^{(I)}}{\gamma^{(I)}} $$ 
	We use 25 iterations of the power method to compute the above quantities.
	\subsection{Update equation for the certificate problem}\label{subsec:certificate_update_derive}
	Our goal is to minimize $\left\|\bx -\bx^{(0)}\right\|$ such that $\left(\bz^{(L)}_{y} - \bz^{(L)}_{t}\right)(\bx) = 0$. We know that the Hessian satisfies the following LMIs:
	\begin{align}
	m\bI \preccurlyeq \nabla^{2}_{\bx} \left(\bz^{(L)}_{y} - \bz^{(L)}_{t}\right) \preccurlyeq M\bI \label{eq:m_M_bound_cert}
	\end{align}
	$K$ is given by Theorem \ref{thm:L_layer_p_n_theorem} for neural network of any depth ($L\ge 2$). For $2$ layer networks, $M$ and $m$ are given by Theorem \ref{thm:single_layer_p_n_theorem}. But for deeper networks ($L \geq 3$), $M=K$, $m=-K$. In either case, $K \geq \max(|m|, |M|)$. Thus, we also have:
	\begin{align}
	-K\bI \preccurlyeq \nabla^{2}_{\bx} \left(\bz^{(L)}_{y} - \bz^{(L)}_{t}\right) \preccurlyeq K\bI \label{eq:K_bound_cert}
	\end{align}
	We will solve the dual ($d^{*}_{cert}$) of the attack problem ($p^{*}_{cert}$). 
	
	The primal problem $(p^{*}_{cert})$ is given by:
	\begin{align*}
	&p^{*}_{cert} = \min_{\bz^{(L)}_{y}(\bx) = \bz^{(L)}_{t}(\bx)} \bigg[\frac{1}{2}\left\|\bx - \bx^{(0)}\right\|^{2}\bigg] \\
	&p^{*}_{cert} = \min_{\bx} \max_{\eta} \bigg[\frac{1}{2}\left\|\bx - \bx^{(0)}\right\|^{2} + \eta \left(\bz^{(L)}_{y} - \bz^{(L)}_{t}\right)(\bx)\bigg]
	\end{align*}
	Using inequality \eqref{eq:m_M_bound_cert} and Theorem \ref{thm:certificate} part (a), we know that the dual of the above problem is convex when $-1/M\leq \eta \leq -1/m$.
	
	The corresponding dual problem ($d^{*}_{cert}$) is given by:
	\begin{align*}
	& d^{*}_{cert} = \max_{-1/M \le \eta \le -1/m} d_{cert}(\eta)\\ & d_{cert}(\eta) = \min_{\bx}  \bigg[\frac{1}{2}\left\|\bx - \bx^{(0)}\right\|^{2} + \eta \left(\bz^{(L)}_{y} - \bz^{(L)}_{t}\right)\left(\bx\right)\bigg]
	\end{align*} 
	
	For a given $\eta$, we have the following optimization:
	$$d_{cert}(\eta) = \min_{\bx} \bigg[\frac{1}{2}\|\bx - \bx^{(0)}\|^{2} + \eta \left(\bz^{(L)}_{y} - \bz^{(L)}_{t}\right)(\bx)\bigg]$$
	We will use majorization-minimization to solve this optimization. 
	
	At a point $\bx^{(k)}$, we aim to solve for the point $\bx^{(k+1)}$ that decreases the objective function. Using the Taylor's theorem at point $\bx^{(k)}$, we have:
	\begin{align}
	&\left(\bz^{(L)}_{y} - \bz^{(L)}_{t}\right)\left(\bx\right) \nonumber\\
	&=  \left(\bz^{(L)}_{y} - \bz^{(L)}_{t}\right)\left(\bx^{(k)}\right) + \left(\bg^{(k)}\right)^{T}\left(\bx - \bx^{(k)}\right) \nonumber 
	\\&+ \frac{1}{2}\left(\bx - \bx^{(k)}\right)^{T}\bH^{(\xi)}\left(\bx - \bx^{(k)}\right) \nonumber
	\end{align}
	where $\bg^{(k)}$ is the gradient of $(\bz^{(L)}_{y} - \bz^{(L)}_{t})$ at $\bx^{(k)}$ and $\bH^{(\xi)}$ is the Hessian at a point $\xi$ on the line connecting $\bx$ and $\bx^{(k)}$.
	
	Multiplying both sides by $\eta$, we get the following equation:
	\begin{align}
	&\eta\left(\bz^{(L)}_{y} - \bz^{(L)}_{t}\right)\left(\bx\right) \nonumber\\
	&=  \eta\left(\bz^{(L)}_{y} - \bz^{(L)}_{t}\right)\left(\bx^{(k)}\right) + \eta\left(\bg^{(k)}\right)^{T}\left(\bx - \bx^{(k)}\right) \nonumber \\
	&+ \frac{\eta}{2}\left(\bx - \bx^{(k)}\right)^{T}\bH^{(\xi)}\left(\bx - \bx^{(k)}\right) \label{eq:taylor_thm_cert}
	\end{align}
	Using inequality \eqref{eq:K_bound_cert}, we know that $-K\bI \preccurlyeq \bH^{(\xi)} \preccurlyeq K\bI\quad \forall \xi \in \mathbb{R}^{D}$,
	\begin{align}
	&\frac{\eta}{2}\left(\bx - \bx^{(k)}\right)^{T}\bH^{(\xi)}\left(\bx - \bx^{(k)}\right) \leq \frac{\left|\eta K\right|}{2}\left\|\bx - \bx^{(k)}\right\|^{2} \label{eq:hess_inequality_cert}
	\end{align}
	Using equation \eqref{eq:taylor_thm_cert} and inequality \eqref{eq:hess_inequality_cert}:
	\begin{align*}
	&\eta\left(\bz^{(L)}_{y} - \bz^{(L)}_{t}\right)(\bx) \\
	&\leq \ \bigg[ \eta\left(\bz^{(L)}_{y} - \bz^{(L)}_{t}\right)(\bx^{(k)}) + \eta\left(\bg^{(k)}\right)^{T}(\bx - \bx^{(k)}) \\
	&+ \frac{|\eta K|}{2}\left\|\bx - \bx^{(k)}\right\|^{2} \bigg]
	\end{align*}
	Adding $1/2 \|\bx - \bx^{(0)}\|^{2}$ to both sides, we get the following inequality:
	\begin{align*}
	&\frac{1}{2}\left\|\bx - \bx^{(0)}\right\|^{2} +  \eta\left(\bz^{(L)}_{y} - \bz^{(L)}_{t}\right)(\bx) \\
	&\leq \bigg[ \frac{1}{2}\left\|\bx - \bx^{(0)}\right\|^{2} + \eta\left(\bz^{(L)}_{y} - \bz^{(L)}_{t}\right)(\bx^{(k)}) \\
	&+ \eta\left(\bg^{(k)}\right)^{T}(\bx - \bx^{(k)}) + \frac{|\eta K|}{2} \left\|\bx - \bx^{(k)}\right\|^{2} \bigg]
	\end{align*}
	LHS is the objective function of $d_{cert}(\eta)$ and RHS is an upper bound. In majorization-minimization, we minimize an upper bound on the objective function. Thus we set the gradient of RHS with respect to $\bx$ to zero and solve for $\bx$:
	\begin{align*}
	\nabla_{\bx} \bigg[&\frac{1}{2}\left\|\bx - \bx^{(0)}\right\|^{2} + \eta\left(\bz^{(L)}_{y} - \bz^{(L)}_{t}\right)(\bx^{(k)}) \\
	&+ \eta\left(\bg^{(k)}\right)^{T}(\bx - \bx^{(k)}) \ + \frac{|\eta K|}{2}\left\|\bx - \bx^{(k)}\right\|^{2} \bigg] = 0
	\end{align*}
	$$\bx - \bx^{(0)} + \eta \bg^{(k)} + \left| \eta K \right| \left(\bx -\bx^{(k)} \right) = 0$$
	$$(1 + \left|\eta K\right|)\bx - \bx^{(0)} + \eta \bg^{(k)} - \left| \eta K \right| \bx^{(k)} = 0$$
	$$\bx = -(1+|\eta K|)^{-1}\left(\eta\bg^{(k)} - |\eta K| \bx^{(k)} - \bx^{(0)} \right)$$
	This gives the following iterative equation:
	\begin{align}
	\bx^{(k+1)} = -(1+|\eta K|)^{-1}\left(\eta\bg^{(k)} - |\eta K| \bx^{(k)} - \bx^{(0)} \right) \label{xk_update_cert}
	\end{align}
	
	\subsection{Update equation for the attack problem}\label{subsec:attack_update_derive}
	Our goal is to minimize $\bz^{(L)}_{y} - \bz^{(L)}_{t}$ within an $l_{2}$ ball of radius of $\rho$. We know that the Hessian satisfies the following LMIs:
	\begin{align}
	m\bI \preccurlyeq \nabla^{2}_{\bx} \left(\bz^{(L)}_{y} - \bz^{(L)}_{t}\right) \preccurlyeq M\bI \label{eq:m_M_bound_attack}
	\end{align}
	$K$ is given by Theorem \ref{thm:L_layer_p_n_theorem}  for neural network of any depth ($L\ge 2$). For $2$ layer networks, $M$ and $m$ are given by Theorem \ref{thm:single_layer_p_n_theorem}. But for deeper networks ($L \geq 3$), $M=K$, $m=-K$. In either case, $K \geq \max(|m|, |M|)$. Thus, we also have:
	\begin{align}
	-K\bI \preccurlyeq \nabla^{2}_{\bx} \left(\bz^{(L)}_{y} - \bz^{(L)}_{t}\right) \preccurlyeq K\bI \label{eq:K_bound_attack}
	\end{align}
	We solve the dual ($d^{*}_{attack}$) of the attack problem ($p^{*}_{attack}$) for the given radius $\rho$.
	
	The primal problem $(p^{*}_{attack})$ is given by:
	\begin{align*}
	&p^{*}_{attack} = \min_{\left\|\bx - \bx^{(0)}\right\| \le \rho} \bz^{(L)}_{y} - \bz^{(L)}_{t} \\
	&p^{*}_{attack} = \min_{\bx} \max_{\eta 
		\ge 0} \bigg[ \bz^{(L)}_{y} - \bz^{(L)}_{t} + \frac{\eta}{2}\left( \left\|\bx - \bx^{(0)}\right\|^{2} - \rho^{2} \right)\bigg]
	\end{align*}
	Using inequality \eqref{eq:m_M_bound_attack} and Theorem \ref{thm:attack} part (a), we know that the dual of the above problem is convex when $-m\leq \eta$.
	
	The corresponding dual problem $(d^{*}_{cert})$ is given by:
	\begin{align*}
	&d^{*}_{attack} = \max_{\eta \ge -m} d_{attack}(\eta)
	\end{align*}
	where $d_{attack}(\eta)$ is given as follows:
	\begin{align*}
	d_{attack}(\eta) = \min_{\bx} \bigg[&\left(\bz^{(L)}_{y} - \bz^{(L)}_{t}\right)(\bx) \\
	&+ \frac{\eta}{2}\left(\left\|\bx - \bx^{(0)}\right\|^{2} - \rho^{2}\right)\bigg]
	\end{align*}
	For a given $\eta$, we have the following optimization:
	\begin{align*}
	d_{attack}(\eta) = \min_{\bx} \bigg[ &\left(\bz^{(L)}_{y} - \bz^{(L)}_{t}\right)\left(\bx\right) \\
	&+ \frac{\eta}{2}\left( \left\|\bx - \bx^{(0)}\right\|^{2} - \rho^{2} \right)\bigg]
	\end{align*}
	We will use majorization-minimization to solve this optimization.
	
	At a point $\bx^{(k)}$, we have to solve for the point $\bx^{(k+1)}$ that decreases the objective function. Using the Taylor's theorem at point $\bx^{(k)}$, we have:
	\begin{align}
	&\left(\bz^{(L)}_{y} - \bz^{(L)}_{t}\right)\left(\bx\right) \nonumber\\
	&=  \left(\bz^{(L)}_{y} - \bz^{(L)}_{t}\right)\left(\bx^{(k)}\right) + \left(\bg^{(k)}\right)^{T}\left(\bx - \bx^{(k)}\right) \nonumber\\ 
	&+ \frac{1}{2}\left(\bx - \bx^{(k)}\right)^{T}\bH^{(\xi)}\left(\bx - \bx^{(k)}\right)  \label{eq:taylor_thm_attack}
	\end{align}
	where $\bg^{(k)}$ is the gradient of $(\bz^{(L)}_{y} - \bz^{(L)}_{t})$ at $\bx^{(k)}$ and $\bH^{(\xi)}$ is the Hessian at a point $\xi$ on the line connecting $\bx$ and $\bx^{(k)}$.
	
	Using inequality \eqref{eq:K_bound_attack}, we know that $-K\bI \preccurlyeq \bH^{(\xi)} \preccurlyeq K\bI\quad \forall \xi \in \mathbb{R}^{D}$,
	\begin{align}
	&\frac{1}{2}\left(\bx - \bx^{(k)}\right)^{T}\bH^{(\xi)}\left(\bx - \bx^{(k)}\right) \leq \frac{K}{2}\left\|\bx - \bx^{(k)}\right\|^{2} \label{eq:hess_inequality_attack}
	\end{align}
	Using equation \eqref{eq:taylor_thm_attack} and inequality \eqref{eq:hess_inequality_attack}:
	\begin{align*}
	&\left(\bz^{(L)}_{y} - \bz^{(L)}_{t}\right)(\bx) \\
	&\leq \bigg[\left(\bz^{(L)}_{y} - \bz^{(L)}_{t}\right)(\bx^{(k)}) \\
	&+ \left(\bg^{(k)}\right)^{T}\left(\bx - \bx^{(k)}\right) \ + \frac{K}{2}\left\|\bx - \bx^{(k)}\right\|^{2} \bigg]
	\end{align*}
	Adding $\eta/2( \|\bx - \bx^{(0)}\|^{2} - \rho^{2})$ to both sides, we get the following inequality:
	\begin{align}
	&\left(\bz^{(L)}_{y} - \bz^{(L)}_{t}\right)(\bx) + \frac{\eta}{2}\left( \left\|\bx - \bx^{(0)}\right\|^{2} - \rho^{2}\right) \nonumber \\
	&\leq \  \bigg[\left(\bz^{(L)}_{y} - \bz^{(L)}_{t}\right)(\bx^{(k)}) + \left(\bg^{(k)}\right)^{T}\left(\bx - \bx^{(k)}\right) \nonumber\\
	&+ \frac{K}{2}\left\|\bx - \bx^{(k)}\right\|^{2} + \frac{\eta}{2}\left( \left\|\bx - \bx^{(0)}\right\|^{2} - \rho^{2}\right) \bigg]\nonumber
	\end{align}
	LHS is the objective function of $d_{attack}(\eta)$ and RHS is an upper bound. In majorization-minimization, we minimize an upper bound on the objective function. Thus we set the gradient of RHS with respect to $\bx$ to zero and solve for $\bx$:
	\begin{align}
	\nabla_{\bx} \bigg[&\left(\bz^{(L)}_{y} - \bz^{(L)}_{t}\right)(\bx^{(k)}) + \left(\bg^{(k)}\right)^{T}\left(\bx - \bx^{(k)}\right)  \nonumber\\
	&+ \frac{K}{2}\left\|\bx - \bx^{(k)}\right\|^{2}
	+ \frac{\eta}{2}\left( \left\|\bx - \bx^{(0)}\right\|^{2} - \rho^{2}\right)\bigg] = 0\nonumber
	\end{align}
    Rearranging the above equation, we get:
	\begin{align*}
	&\bg^{(k)} + K\left(\bx - \bx^{(k)}\right) + \eta\left( \bx - \bx^{(0)}\right) = 0\\
    & (K + \eta)\bx + \bg^{(k)} - K \bx^{(k)} - \eta \bx^{(0)} = 0\\
    & \bx = -(K + \eta)^{-1}\left(\bg^{(k)} - K\bx^{(k)} - \eta\bx^{(0)}\right)
	\end{align*}
	This gives the following iterative equation:
	\begin{align}
	\bx^{(k+1)} = -(K + \eta)^{-1}\left(\bg^{(k)} - K\bx^{(k)} - \eta\bx^{(0)}\right) \label{xk_update_attack}
	\end{align}
	
	\subsection{Algorithm to compute the certificate}\label{subsec:algorithm_certificate}
	We start with the following initial values of $\bx,\  \eta,\ \eta_{min},\ \eta_{max}$: 
	\begin{align*}
	&\eta_{min}=-1/M,\qquad \eta_{max}=-1/m\\
	&\eta = \frac{1}{2}(\eta_{min} + \eta_{max}),\qquad\bx = \bx^{(0)}
	\end{align*}
	To solve the dual for a given value of $\eta$, we run $20$ iterations of the following update (derived in Appendix \ref{subsec:certificate_update_derive}):
	\begin{align*}
	&\bx^{(k+1)} = -(1+|\eta K|)^{-1}\left(\eta\bg^{(k)} - |\eta K| \bx^{(k)} - \bx^{(0)} \right) 
	%&\bx^{(k+1)} = -(K + \eta)^{-1}\left(\bg^{(k)} - K\bx^{(k)} - \eta\bx^{(0)}\right) 
	\end{align*}
	
	To maximize the dual $d_{cert}(\eta)$ over $\eta$ in the range $[-1/M,\ -1/m]$, we use a bisection method: If the solution $\bx$ for a given value of $\eta$, $(\bz^{(L)}_{y} - \bz^{(L)}_{t})(\bx) > 0$, set $\eta_{min} = \eta$, else set $\eta_{max} = \eta$. Set the new $\eta = (\eta_{min} + \eta_{max})/2$ and repeat. The maximum number of updates to $\eta$ are set to 30. This method satisfied linear convergence. 
	The routine to compute the certificate example is given in Algorithm \ref{alg:cert}.
	\begin{algorithm}[h]
		\caption{Certificate optimization}
		\begin{algorithmic}\label{alg:cert}
			\REQUIRE input $\bx^{(0)}$, label $y$, target $t$ 
			\STATE $m,M,K \leftarrow compute\_bounds(\bz^{(L)}_{y} - \bz^{(L)}_{t})$
			\STATE $\eta_{min} \leftarrow -1/M$ 
			\STATE $\eta_{max} \leftarrow -1/m$ 
			\STATE $\eta \leftarrow 1/2(\eta_{min} + \eta_{max})$ 
			\STATE $\bx \leftarrow \bx^{(0)}$
			\FOR{$i$ in $[1,\dots,30]$}
			\FOR{$j$ in $[1,\dots,20]$}
			\STATE $\bg \leftarrow compute\_gradient(\bz^{(L)}_{y} - \bz^{(L)}_{t}, \bx)$ 
			\IF{$\|\eta\bg + (\bx - \bx^{(0)})\| < 10^{-5} $}
			\STATE \textbf{break}
			\ENDIF
			\STATE 	$\bx \leftarrow -(1+|\eta K|)^{-1}\left(\eta\bg - |\eta K| \bx - \bx^{(0)} \right)$ 
			\ENDFOR
			\IF{$(\bz^{(L)}_{y} - \bz^{(L)}_{t})(\bx) > 0 $}
			\STATE $\eta_{min} \leftarrow \eta$
			\ELSE
			\STATE $\eta_{max} \leftarrow \eta$
			\ENDIF
			\STATE $\eta \leftarrow (\eta_{min} + \eta_{max})/2$
			\ENDFOR\\
			\textbf{return} $\bx$
		\end{algorithmic}
	\end{algorithm}
	
	\subsection{Algorithm to compute the attack}\label{subsec:algorithm_attack}
	We start with the following initial values of $\bx,\  \eta,\ \eta_{min},\ \eta_{max}$: 
	\begin{align*}
	&\eta_{min}=-m,\qquad \eta_{max}=20(1-m)\\
	&\eta = \frac{1}{2}(\eta_{min} + \eta_{max}),\qquad\bx = \bx^{(0)}
	\end{align*}
	To solve the dual for a given value of $\eta$, we run $20$ iterations of the following update (derived in Appendix \ref{subsec:attack_update_derive}):
	\begin{align*}
	&\bx^{(k+1)} = -(K + \eta)^{-1}\left(\bg^{(k)} - K\bx^{(k)} - \eta\bx^{(0)}\right) 
	\end{align*}
	
	To maximize the dual $d_{cert}(\eta)$ over $\eta$ in the range $[-m,\ 20(1-m)]$, we use a bisection method: If the solution $\bx$ for a given value of $\eta$,  $\|\bx - \bx^{(0)}\| \leq \rho$, set $\eta_{max} = \eta$, else set $\eta_{min} = \eta$. Set new $\eta = (\eta_{min} + \eta_{max})/2$ and repeat. The maximum number of updates to $\eta$ are set to 30. This method satisfied linear convergence. The routine to compute the attack example is given in Algorithm \ref{alg:attack}.
	
	\begin{algorithm}[h]
		\caption{Attack optimization}
		\begin{algorithmic}\label{alg:attack}
			\REQUIRE input $\bx^{(0)}$, label $y$, target $t$ , radius $\rho$
			\STATE $m,M,K \leftarrow compute\_bounds(\bz^{(L)}_{y} - \bz^{(L)}_{t})$ 
			\STATE $\eta_{min} \leftarrow -m$ 
			\STATE $\eta_{max} \leftarrow 20(1-m)$
			\STATE $\eta \leftarrow 1/2(\eta_{min} + \eta_{max})$ 
			\STATE $\bx \leftarrow \bx^{(0)}$
			\FOR{$i$ in $[1,\dots,30]$}
			\FOR{$j$ in $[1,\dots,20]$}
			\STATE $\bg \leftarrow compute\_gradient(\bz^{(L)}_{y} - \bz^{(L)}_{t}, \bx)$ 
			\IF{$\|\bg + \eta (\bx - \bx^{(0)})\| < 10^{-5} $}
			\STATE \textbf{break}
			\ENDIF
			\STATE 	$\bx \leftarrow -(K + \eta)^{-1}\left(\bg - K \bx - \eta\bx^{(0)} \right)$ 
			\ENDFOR
			\IF{$\|\bx - \bx^{(0)}\| < \rho $}
			\STATE $\eta_{max} \leftarrow \eta$
			\ELSE
			\STATE $\eta_{min} \leftarrow \eta$
			\ENDIF
			\STATE $\eta \leftarrow (\eta_{min} + \eta_{max})/2$
			\ENDFOR\\
			\textbf{return} $\bx$
		\end{algorithmic}
	\end{algorithm}

    \subsection{Computing certificate using local curvature bounds}\label{appendix_local_global}
	To compute the robustness certificate in a local region around the input, we first compute the certificate using the global bounds on the curvature. Using the same certificate as the initial $l_{2}$ radius of the safe region, we can refine our certificate. Due to the reduction in curvature, this will surely increase the value of the certificate. We then use the new robustness certificate as the new $l_{2}$ radius of the safe region and repeat. We iterate over this process $5$ times to compute the local version of our robustness certificate.

	To ensure that the optimization trajectory does not escape the safe region, whenever the gradient descent step lies outside the "safe" region, we reduce the step size by a factor of two until it lies inside the region.
	
	\section{Summary Table comparing out certification method against existing methods}
	Table \ref{table:summary_table} provides a summary table comparing our certification method against the existing methods.
	\begin{table*}[h!]
		\centering
		\renewcommand{\arraystretch}{1.15}
		\caption{Comparison of methods for providing provable robustness certification. Note that \cite{Cohen2019CertifiedAR} is a probabilistic certificate.}
		\begin{tabular}{ | c | c | c | c | c | c | } 
			\hline
			Method & \multirow{2}{5em}{Non-trivial bound}  & \multirow{2}{4.5em}{Multi-layer} & \multirow{2}{5em}{Activation functions} & \multirow{2}{4.5em}{Norm}\\
			& & & & \\
			\hline
			\cite{42503} & \xmark & \cmark & All & $l_{2}$\\
			\hline
			\cite{Katz2017ReluplexAE} & \cmark & \cmark & ReLU & $l_{\infty}$\\
			\hline
			\cite{NIPS2017_6821} & \cmark & \xmark & Differentiable & $l_{2}$\\
			\hline
			\cite{Raghunathan2018CertifiedDA} & \cmark & \xmark & ReLU & $l_{\infty}$\\
			\hline
			\cite{Wong2017ProvableDA} & \cmark & \cmark & ReLU & $l_{\infty}$\\
			\hline
			\cite{Weng2018TowardsFC} & \cmark & \cmark & ReLU & $l_1,l_2,l_{\infty}$\\
			\hline
			\cite{Zhang2018EfficientNN} & \cmark & \cmark & All & $l_1,l_2,l_{\infty}$\\
			\hline
			\cite{Cohen2019CertifiedAR} & \cmark & \cmark & All & $l_{2}$\\
			\hline
			Ours & \cmark & \cmark & Differentiable & $l_{2}$ \\
			\hline
		\end{tabular}
		\label{table:summary_table}
	\end{table*}
	
	\section{Proofs} \label{proofs}
	\subsection{Proof of Theorem \ref{thm:certificate}}\label{proof:certificate}
	\begin{enumerate}[label=(\alph*)]
		\item 
	    \begin{align*}
	    d_{cert}(\eta) = \min_{\bx} \bigg[&\frac{1}{2}\left\|\bx - \bx^{(0)}\right\|^{2} \\
	    & + \eta \left(\bz^{(L)}_{y}(\bx) - \bz^{(L)}_{t}(\bx)\right)\bigg]
	    \end{align*}
		\begin{align*}
		&\nabla^{2}_{\bx} \bigg[\frac{1}{2}\left\|\bx - \bx^{(0)}\right\|^{2} + \eta \left(\bz^{(L)}_{y}(\bx) - \bz^{(L)}_{t}(\bx)\right)\bigg] \\
		&= \bI + \eta\nabla^{2}_{\bx}\left(\bz^{(L)}_{y} - \bz^{(L)}_{t}\right)
		\end{align*}
		We are given that the Hessian $\nabla^{2}_{\bx} (\bz^{(L)}_{y} - \bz^{(L)}_{t})$ satisfies the following LMIs:  
		$$m \bI \preccurlyeq \nabla^{2}_{\bx} \left(\bz^{(L)}_{y} - \bz^{(L)}_{t}\right) \preccurlyeq M \bI\qquad \forall \bx \in \mathbb{R}^{n}$$
		The eigenvalues of $\bI + \eta\nabla^{2}_{\bx}(\bz^{(L)}_{y} - \bz^{(L)}_{t})$ are bounded between: 
		$$(1 + \eta M,\ 1 + \eta m), \text{ if } \eta < 0$$
		$$(1 + \eta m,\ 1 + \eta M), \text{ if } \eta > 0$$
		We are given that $\eta$ satisfies the following inequalities where $m<0, M>0$ since $(\bz^{(L)}_{y} - \bz^{(L)}_{t})$ is neither convex, nor concave as a function of $\bx$:
		$$ \frac{-1}{M} \le \eta \le \frac{-1}{m},\qquad m<0, M>0$$
		We have the following inequalities:
		$$ 1+ \eta M \ge 0,\ 1 + \eta m \ge 0$$
		Thus, $\bI + \eta\nabla^{2}_{\bx}(\bz^{(L)}_{y} - \bz^{(L)}_{t})$ is a PSD matrix for all $\bx \in \mathbb{R}^{D}$ when $-1/M \leq \eta \leq -1/m$ .\\
		Thus $1/2\|\bx - \bx^{(0)}\|^{2} + \eta (\bz^{(L)}_{y} - \bz^{(L)}_{t})(\bx)$ is a convex function in $\bx$ and $d_{cert}(\eta)$ is a convex optimization problem.
		\item For every value of $\eta$, $d_{cert}(\eta)$ is a lower bound for $p^{*}_{cert}$. Thus $d^{*}_{cert} = \max_{-1/M \leq\ \eta \ \leq -1/m} d_{cert}(\eta)$ is a lower bound for $p^{*}_{cert}$, i.e:
		\begin{align}
		d^{*}_{cert} \leq p^{*}_{cert} \label{d_leq_p_cert}
		\end{align}
		Let $\eta^{(cert)},\bx^{(cert)}$ be the solution of the above dual optimization ($d^{*}_{cert}$) such that \begin{align}
		\bz^{(L)}_{y}(\bx^{(cert)})=\bz^{(L)}_{t}(\bx^{(cert)}) \label{fl_eq_ft_cert}
		\end{align}
		$d^{*}_{cert}$ is given by the following:
		\begin{align}
		d^{*}_{cert} = \bigg[&\frac{1}{2}\left\|\bx^{(cert)} - \bx^{(0)}\right\|^{2} \nonumber\\
		&+ \eta^{(cert)}\underbrace{\left(\bz^{(L)}_{y}(\bx^{(cert)}) - \bz^{(L)}_{t}(\bx^{(cert)})\right)}_{ = 0}\bigg]\nonumber
		\end{align}
		Since we are given that $\bz^{(L)}_{y}(\bx^{(cert)}) = \bz^{(L)}_{t}(\bx^{(cert)})$, we get the following equation for $d^{*}_{cert}$:
		\begin{align}
		d^{*}_{cert} &= \frac{1}{2}\left\|\bx^{(cert)} - \bx^{(0)}\right\|^{2}
		\label{d_star_dist_cert}
		\end{align} 
		Since $p^{*}_{cert}$ is given by the following equation:
		\begin{align}
		p^{*}_{cert} = \min_{\bz^{(L)}_{y}(\bx) = \bz^{(L)}_{t}(\bx)} \bigg[\frac{1}{2}\left\|\bx - \bx^{(0)}\right\|^{2}\bigg]\label{p_star_cert}
		\end{align}

		Using equations \eqref{fl_eq_ft_cert} and \eqref{p_star_cert}, $p^{*}_{cert}$ is the minimum value of $1/2\|\bx - \bx^{(0)}\|^{2}\quad \forall \bx: \bz^{(L)}_{y}(\bx) = \bz^{(L)}_{t}(\bx)$:
		\begin{align}
		p^{*}_{cert} \leq \frac{1}{2}\left\|\bx^{(cert)} - \bx^{(0)}\right\|^{2} \label{p_leq_dist_cert}
		\end{align}
		From equation \eqref{d_star_dist_cert}, we know that $d^{*}_{cert}=1/2\|\bx^{(cert)} - \bx^{(0)}\|^{2}$. Thus, we get:
		\begin{align}
		p^{*}_{cert} \leq d^{*}_{cert} \label{p_leq_d_cert}
		\end{align}
		Using equation \eqref{d_leq_p_cert} we have $d^{*}_{cert} \leq p^{*}_{cert}$ and using \eqref{p_leq_d_cert}, $p^{*}_{cert} \leq d^{*}_{cert}$
		$$p^{*}_{cert} = d^{*}_{cert}$$
	\end{enumerate}
	\subsection{Proof of Theorem \ref{thm:attack}}\label{proof:attack}
	\begin{enumerate}[label=(\alph*)]
		\item 
		\begin{align*}
		d_{attack}(\eta) = \min_{\bx} \bigg[&\left(\bz^{(L)}_{y} - \bz^{(L)}_{t}\right)(\bx) \\
		&+ \frac{\eta}{2}\left(\left\|\bx - \bx^{(0)}\right\|^{2} - \rho^{2}\right)\bigg]
		\end{align*}
		\begin{align*}
		&\nabla^{2}_{\bx} \bigg[\left(\bz^{(L)}_{y} - \bz^{(L)}_{t}\right)(\bx) + \frac{\eta}{2}\left\|\bx - \bx^{(0)}\right\|^{2}\bigg] \\
		&= \nabla^{2}_{\bx}\left(\bz^{(L)}_{y} - \bz^{(L)}_{t}\right) + \eta \bI
		\end{align*}
		Since the Hessian $\nabla^{2}_{\bx} (\bz^{(L)}_{y} - \bz^{(L)}_{t})$ is bounded below:  
		$$m \bI \preccurlyeq \nabla^{2}_{\bx} \left(\bz^{(L)}_{y} - \bz^{(L)}_{t}\right)\qquad \forall \bx \in \mathbb{R}^{n}$$
		The eigenvalues of $\nabla^{2}_{\bx}(\bz^{(L)}_{y} - \bz^{(L)}_{t}) + \eta \bI$ are bounded below: $$(m + \eta)\bI \preccurlyeq \nabla^{2}_{\bx}\left(\bz^{(L)}_{y} - \bz^{(L)}_{t}\right) + \eta \bI $$
		Since $\eta \ge -m$.
		$$ \eta + m \ge 0$$
		Thus $\nabla^{2}_{\bx}(\bz^{(L)}_{y} - \bz^{(L)}_{t}) + \eta \bI$ is a PSD matrix for all $\bx \in \mathbb{R}^{D}$ when $\eta \ge -m$.\\
		Thus $(\bz^{(L)}_{y} - \bz^{(L)}_{t})(\bx) + \eta/2(\|\bx - \bx^{(0)}\|^{2} - \rho^{2})$ is a convex function in $\bx$ and $d_{attack}(\eta)$ is a convex optimization problem.
		\item For every value of $\eta$, $d_{attack}(\eta)$ is a lower bound for $p^{*}_{attack}$. Thus $d^{*}_{attack} = \max_{-m \leq \eta} d_{attack}(\eta)$ is a lower bound for $p^{*}_{attack}$:
		\begin{align}
		d^{*}_{attack} \leq p^{*}_{attack} \label{d_leq_p_attack}
		\end{align}
		Let $\eta^{(attack)},\bx^{(attack)}$ be the solution of the above dual optimization ($d^{*}_{attack}$) such that \begin{align}
		\left\|\bx^{(attack)} - \bx^{(0)}\right\| = \rho \label{fl_eq_ft_attack}
		\end{align}
		$d^{*}_{attack}$ is given by the following:
		\begin{align}
		d^{*}_{attack} = \bigg[&\left(\bz^{(L)}_{y} - \bz^{(L)}_{t}\right)(\bx^{(attack)}) \\
		&+ \frac{\eta^{(attack)}}{2}\underbrace{\left(\left\|\bx^{(attack)} - \bx^{(0)}\right\|^{2} - \rho^{2}\right)}_{ = 0}\bigg]\nonumber
		\end{align} 
		Since we are given that $\|\bx^{(attack)} - \bx^{(0)}\| = \rho$, we get the following equation for $d^{*}_{attack}$:
		\begin{align}
		d^{*}_{attack} &= \left(\bz^{(L)}_{y} - \bz^{(L)}_{t}\right)(\bx^{(attack)})
		\label{d_star_dist_attack}
		\end{align} 
		Since $p^{*}_{attack}$ is given by the following equation:
		\begin{align}
		p^{*}_{attack} = \min_{\left\|\bx - \bx^{(0)}\right\| \leq \rho} \bigg[\left(\bz^{(L)}_{y} - \bz^{(L)}_{t}\right)(\bx)\bigg]\label{p_star_attack}
		\end{align}
		Using equations \eqref{fl_eq_ft_attack} and \eqref{p_star_attack}, $p^{*}_{attack}$ is the minimum value of $(\bz^{(L)}_{y} - \bz^{(L)}_{t})(\bx)\quad\forall \left\|\bx-\bx^{(0)}\right\|\leq \rho$:
		\begin{align}
		p^{*}_{attack} \leq \left(\bz^{(L)}_{y} - \bz^{(L)}_{t}\right)(\bx^{(attack)}) \label{p_leq_dist_attack}
		\end{align}
		From equation \eqref{d_star_dist_attack}, we know that $d^{*}_{attack}=(\bz^{(L)}_{y} - \bz^{(L)}_{t})(\bx^{(attack)})$. Thus, we get:
		\begin{align}
		p^{*}_{attack} \leq d^{*}_{attack} \label{p_leq_d_attack}
		\end{align}
		Using equation \eqref{d_leq_p_attack} we have $d^{*}_{attack} \leq p^{*}_{attack}$ and using \eqref{p_leq_d_attack}, $p^{*}_{attack} \leq d^{*}_{attack}$
		$$p^{*}_{attack} = d^{*}_{attack}$$
	\end{enumerate}
	
	\subsection{Proof of Lemma \ref{thm:deep_hessian_closed}}\label{proof:deep_hessian_closed}
	We have to prove that for an $L$ layer neural network, the hessian of the $i^{th}$ hidden unit in the $L^{th}$ layer with respect to the input $\bx$, i.e $\nabla_{\bx}^{2} \bz^{(L)}_{i}$ is given by the following formula:
	\begin{align}
	\nabla^{2}_{\bx} \bz^{(L)}_{i} = \sum_{I=1}^{L-1}\left(\bB^{(I)}\right)^{T}diag\bigg(\bF^{(L,I)}_{i}\odot\sigma^{''}\left(\bz^{(I)}\right)\bigg)\bB^{(I)}\label{final_formula}
	\end{align}
	where $\bB^{(I)}$, $I \in [L]$ is a matrix of size $N_{I} \times D$ defined as follows:
	\begin{align}
	\bB^{(I)} = \bigg[\nabla_{\bx} \bz^{(I)}_{1},\nabla_{\bx} \bz^{(I)}_{2},\dotsc,\nabla_{\bx} \bz^{(I)}_{N_{I}}\bigg]^{T} \label{B_def}
	\end{align}
	and $\bF^{(L,I)},\ I \in [L-1]$ is a matrix of size $N_{L} \times N_{I}$ defined as follows:
	\begin{align}
	\bF^{(L,I)} = \bigg[\nabla_{\ba^{(I)}} \bz^{(L)}_{1},\nabla_{\ba^{(I)}} \bz^{(L)}_{2},\dotsc,\nabla_{\ba^{(I)}} \bz^{(L)}_{N_{L}}\bigg]^{T} \label{F_def}
	\end{align}
	
	$\nabla^{2}_{\bx} \bz^{(L)}_{i}$ can be written in terms of the activations of the previous layer using the following formula:
	\begin{align}
	\nabla^{2}_{\bx}\bz^{(L)}_{i} &= \sum_{j=1}^{N_{I-1}}\bWW^{(L)}_{i,j}\left(\nabla^{2}_{\bx} \ba^{(L-1)}_{j}\right)\label{output_hessian_formula_i}
	\end{align}
	Using the chain rule of the Hessian and $\ba^{(I)} = \sigma(\bz^{(I)})$, we can write $\nabla^{2}_{\bx} \ba^{(L-1)}_{j}$ in terms of $\nabla_{\bx} \bz^{(L-1)}_{j}$ and $\nabla^{2}_{\bx} \bz^{(L-1)}_{j}$ as the following:
	\begin{align}
	\nabla^{2}_{\bx} \ba^{(L-1)}_{j} &= \sigma^{''}\left(\bz^{(L-1)}_{j}\right) \left(\nabla_{\bx} \bz^{(L-1)}_{j}\right)\left(\nabla_{\bx} \bz^{(L-1)}_{j}\right)^{T} \nonumber\\
	&+ \sigma^{'}\left(\bz^{(L-1)}_{j}\right) \left(\nabla^{2}_{\bx} \bz^{(L-1)}_{j}\right)\label{act_hessian_formula_j}
	\end{align}
	Replacing $\nabla^{2}_{\bx} \ba^{(L-1)}_{j}$ using equation \eqref{act_hessian_formula_j} into equation \eqref{output_hessian_formula_i}, we get:
	\begin{align}
	&\nabla^{2}_{\bx} \left(\bz^{(L)}_{i}\right) = \nonumber\\
	& \sum_{j=1}^{N_{L-1}}\bWW^{(L)}_{i,j}\bigg[\sigma^{''}\left(\bz^{(L-1)}_{j}\right) \left(\nabla_{\bx} \bz^{(L-1)}_{j}\right)\left(\nabla_{\bx} \bz^{(L-1)}_{j}\right)^{T} \nonumber\\
	&+ \sigma^{'}\left(\bz^{(L-1)}_{j}\right) \left(\nabla^{2}_{\bx} \bz^{(L-1)}_{j}\right)\bigg]\nonumber
	\end{align}
	\begin{align}
	&\nabla^{2}_{\bx} \left(\bz^{(L)}_{i}\right) =\\ &\sum_{j=1}^{N_{L-1}}\bWW^{(L)}_{i,j}\sigma^{''}\left(\bz^{(L-1)}_{j}\right) \left(\nabla_{\bx} \bz^{(L-1)}_{j}\right)\left(\nabla_{\bx} \bz^{(L-1)}_{j}\right)^{T}\nonumber\\ 
	&+ \sum_{j=1}^{N_{L-1}}\bWW^{(L)}_{i,j}\sigma^{'}\left(\bz^{(L-1)}_{j}\right) \left(\nabla^{2}_{\bx} \bz^{(L-1)}_{j}\right)\label{hessian_z_L_raw}
	\end{align}
	For each $I \in [2, L],\ i \in N_{I}$, we define the matrix $\bA^{(I)}_{i}$ as the following:
	\begin{align}
	&\nabla^{2}_{\bx} \left(\bz^{(I)}_{i}\right) \nonumber\\
	&= \underbrace{\sum_{j=1}^{N_{I-1}}\bWW^{(I)}_{i,j}\sigma^{''}\left(\bz^{(I-1)}_{j}\right) \left(\nabla_{\bx} \bz^{(I-1)}_{j}\right)\left(\nabla_{\bx} \bz^{(I-1)}_{j}\right)^{T}}_{\bA^{(I)}_{i}} \nonumber\\
	&+\sum_{j=1}^{N_{I-1}}\bWW^{(I)}_{i,j}\sigma^{'}\left(\bz^{(I-1)}_{j}\right) \left(\nabla^{2}_{\bx} \bz^{(I-1)}_{j}\right)\label{L_R_def}\\
	&\bA^{(I)}_{i} = \sum_{j=1}^{N_{I-1}}\bWW^{(I)}_{i,j}\sigma^{''}\left(\bz^{(I-1)}_{j}\right) \left(\nabla_{\bx} \bz^{(I-1)}_{j}\right)\left(\nabla_{\bx} \bz^{(I-1)}_{j}\right)^{T}\label{L_def}
	\end{align}
	Substituting $\bA^{(L)}_{i}$ using equation $\eqref{L_def}$ into equation \eqref{hessian_z_L_raw}, we get:
	\begin{align}
	\nabla^{2}_{\bx} \left(\bz^{(L)}_{i}\right) &= \bA^{(L)}_{i} + \sum_{j=1}^{N_{I-1}}\bWW^{(I)}_{i,j}\sigma^{'}\left(\bz^{(I-1)}_{j}\right) \left(\nabla^{2}_{\bx} \bz^{(I-1)}_{j}\right)\label{hessian_i_in_L_R}
	\end{align}
	We first simplify the expression for $\bA^{(L)}_{i}$. Note that $\bA^{(L)}_{i}$ is a sum of symmetric rank one matrices $\left(\nabla_{\bx} \bz^{(L-1)}_{j}\right)\left(\nabla_{\bx} \bz^{(L-1)}_{j}\right)^{T}$ with the coefficient $\bWW^{(L)}_{i,j}\sigma^{''}\left(\bz^{(L-1)}_{j}\right)$ for each $j$. We create a diagonal matrix for the coefficients and another matrix $\bB^{(L-1)}$ such that each $j^{th}$ row of $\bB^{(L-1)}$ is the vector $\nabla_{\bx} \bz^{(L-1)}_{j}$. This leads to the following equation:
	\begin{align}
	\bA^{(L)}_{i} &= \sum_{j=1}^{N_{L-1}}\bWW^{(L)}_{i,j}\sigma^{''}\left(\bz^{(L-1)}_{j}\right) \left(\nabla_{\bx} \bz^{(L-1)}_{j}\right)\left(\nabla_{\bx} \bz^{(L-1)}_{j}\right)^{T}\nonumber\\
	&= \left(\bB^{(L-1)}\right)^{T} diag \left(\bWW^{(L)}_{i}\odot\sigma^{''}\left(\bz^{(L-1)}\right)\right)\bB^{(L-1)}\label{L_W_eq}
	\end{align}
	$\bB^{(I)}$ where $I \in [L]$ is a matrix of size $N_{I} \times D$ defined as follows:
	\begin{align}
	\bB^{(I)} &= \bigg[\nabla_{\bx} \bz^{(I)}_{1},\nabla_{\bx} \bz^{(I)}_{2},\dotsc,\nabla_{\bx} \bz^{(I)}_{N_{I}}\bigg]^{T},\qquad I \in [L] \nonumber	\end{align}
	Thus $\bB^{(I)}$ is the jacobian of $\bz^{(I)}$ with respect to the input $\bx$.\\
	Using the chain rule of the gradient, we have the following properties of $\bB^{(I)}$:
	\begin{align}
	&\bB^{(1)} = 
	\bWW^{(1)} \label{B_base}\\
	&\bB^{(I)} = \bWW^{(I)}diag \left(\sigma^{'}\left(\bz^{(I-1)}\right)\right)\bB^{(I-1)} \label{B_induction}
	\end{align}
	Similarly, $\bF^{(I,J)}$ where $I \in [L],\ J \in [I-1]$ is a matrix of size $N_{I} \times N_{J}$ defined as follows:
	\begin{align}
	\bF^{(I,J)} &= \bigg[\nabla_{\ba^{(J)}} \bz^{(I)}_{1},\nabla_{\ba^{(J)}} \bz^{(I)}_{2},\dotsc,\nabla_{\ba^{(J)}} \bz^{(I)}_{N_{I}}\bigg]^{T}\nonumber \end{align}
	Thus $\bF^{(I,J)}$ is the jacobian of $\bz^{(I)}$ with respect to the activations $\ba^{(J)}$. \\
	Using the chain rule of the gradient, we have the following properties for $\bF^{(L,I)}$:
	\begin{align}
	&\bF^{(L,L-1)} = \bWW^{(L)} \label{F_base}\\
	&\bF^{(L,I)} = \bWW^{(L)} diag \left(\sigma^{'}\left(\bz^{(L-1)}\right)\right)\bF^{(L-1,I)} \label{F_induction}
	\end{align}
	Recall that in our notation: For a matrix $\bE$, $\bE_{i}$ denotes the column vector constructed by taking the transpose of the $i^{th}$ row of the matrix $\bE$. Thus $i^{th}$ row of $\bWW^{(L)}$ is $\left(\bWW^{(L)}_{i}\right)^{T}$ and $\bF^{(L,I)}$ is $\left(\bF^{(L,I)}_{i}\right)^{T}$. Equating the $i^{th}$ rows in equation \eqref{F_induction}, we get:
	\begin{align}
	&\left(\bF^{(L,I)}_{i}\right)^{T} = \left(\bWW^{(L)}_{i}\right)^{T} diag \left(\sigma^{'}\left(\bz^{(L-1)}\right)\right)\bF^{(L-1,I)}\nonumber
	\end{align}
	Taking the transpose of both the sides and expressing the RHS as a summation, we get:
	\begin{align}
	&\bF^{(L,I)}_{i} = \left(\left(\bWW^{(L)}_{i}\right)^{T} diag \left(\sigma^{'}\left(\bz^{(L-1)}\right)\right)\bF^{(L-1,I)}\right)^{T} \nonumber\\
	&\bF^{(L,I)}_{i} = \sum_{j=1}^{N_{L-1}} \bWW^{(L)}_{i,j}\sigma^{'}\left(\bz^{(L-1)}_{j}\right)\bF^{(L-1, I)}_{j} \label{F_induction_i}
	\end{align}
	Substituting $\bWW^{(L)}$ using equation \eqref{F_base} into equation \eqref{L_W_eq}, we get:
	\begin{align}
	\bA^{(L)}_{i} &= \left(\bB^{(L-1)}\right)^{T} diag \left(\bF^{(L,L-1)}_{i}\odot\sigma^{''}\left(\bz^{(L-1)}\right)\right)\bB^{(L-1)}\label{L_simple_eq}
	\end{align}
	Substituting $\bA^{(L)}_{i}$ using equation \eqref{L_simple_eq} into \eqref{hessian_i_in_L_R}, we get:
	\begin{align}
	&\nabla^{2}_{\bx} \bz^{(L)}_{i} =\nonumber \\
	&\bigg[\left(\bB^{(L-1)}\right)^{T} diag \left(\bF^{(L,L-1)}_{i}\odot\sigma^{''}\left(\bz^{(L-1)}\right)\right)\bB^{(L-1)}\nonumber\\ &+\sum_{j=1}^{N_{L-1}}\bWW^{(L)}_{i,j}\sigma^{'}\left(\bz^{(L-1)}_{j}\right) \left(\nabla^{2}_{\bx} \bz^{(L-1)}_{j}\right)\bigg]\label{z_simple_recurrence}
	\end{align}
	Thus, equation \eqref{z_simple_recurrence} allows us to write the hessian of $i^{th}$ unit at layer $L$, i.e $\left(\nabla^{2}_{\bx} \bz^{(L)}_{i}\right)$ in terms of the hessian of $j^{th}$ unit at layer $L-1$, i.e $\left(\nabla^{2}_{\bx} \bz^{(L-1)}_{j}\right)$.\\
	We will prove the following using induction:
	\begin{align}
	\nabla^{2}_{\bx} \bz^{(L)}_{i} &= \sum_{I=1}^{L-1}\left(\bB^{(I)}\right)^{T} diag \left(\bF^{(L,I)}_{i}\odot\sigma^{''}\left(\bz^{(I)}\right)\right)\bB^{(I)}\label{induct_hypothesis}
	\end{align}
	Note that for $L=2, \nabla^{2}_{\bx} \bz^{(L-1)}_{j}=0,\ \forall j \in N_{1}$. Thus using \eqref{z_simple_recurrence} we have:
	\begin{align}
	\nabla^{2}_{\bx} \bz^{(2)}_{i} &= \left(\bB^{(1)}\right)^{T} diag \left(\bF^{(2,1)}_{i}\odot\sigma^{''}\left(\bz^{(1)}\right)\right)\bB^{(1)}\nonumber
	\end{align}
	Hence the induction hypothesis \eqref{induct_hypothesis} is true for $L=2$. \\
	Now we will assume \eqref{induct_hypothesis} is true for $L-1$. Thus we have:
	\begin{align}
	&\nabla^{2}_{\bx} \bz^{(L-1)}_{j} \nonumber\\ &=\sum_{I=1}^{L-2}\left(\bB^{(I)}\right)^{T} diag \left(\bF^{(L-1,I)}_{j}\odot\sigma^{''}\left(\bz^{(I)}\right)\right)\bB^{(I)} \nonumber\\ 
	&\quad \forall j \in N_{L-1} \label{true_L_minus_1}
	\end{align}
	We will prove the same for $L$. \\
	Using equation \eqref{z_simple_recurrence}, we have:
	\begin{align}
	&\nabla^{2}_{\bx} \bz^{(L)}_{i} \nonumber\\
	&= \left(\bB^{(L-1)}\right)^{T} diag \left(\bF^{(L,L-1)}_{i}\odot\sigma^{''}\left(\bz^{(L-1)}\right)\right)\bB^{(L-1)}\nonumber\\ &+\sum_{j=1}^{N_{L-1}}\bWW^{(L)}_{i,j}\sigma^{'}\left(\bz^{(L-1)}_{j}\right) \left(\nabla^{2}_{\bx} \bz^{(L-1)}_{j}\right)\nonumber
	\end{align}
	In the next set of steps, we will be working with the second term of the above equation, i.e:\  $\sum_{j=1}^{N_{L-1}}\bWW^{(L)}_{i,j}\sigma^{'}(\bz^{(L-1)}_{j}) (\nabla^{2}_{\bx} \bz^{(L-1)}_{j})$\\
	Substituting $\nabla^{2}_{\bx} \bz^{(L-1)}_{j}$ using equation \eqref{true_L_minus_1} we get:
	\begin{align}
	&\nabla^{2}_{\bx} \bz^{(L)}_{i} \nonumber\\
	&= \left(\bB^{(L-1)}\right)^{T} diag \left(\bF^{(L,L-1)}_{i}\odot\sigma^{''}\left(\bz^{(L-1)}\right)\right)\bB^{(L-1)}\nonumber\\ &+\sum_{j=1}^{N_{L-1}}\bWW^{(L)}_{i,j}\sigma^{'}\left(\bz^{(L-1)}_{j}\right) \bigg[\\
	&\sum_{I=1}^{L-2} \left(\bB^{(I)}\right) diag \left(\bF^{(L-1,I)}_{j}\odot\sigma^{''}\left(\bz^{(I)}\right)\right)\left(\bB^{(I)}\right)^{T} \bigg]\nonumber
	\end{align}
	Combining the two summations in the second term, we get:
	\begin{align}
	&\nabla^{2}_{\bx} \bz^{(L)}_{i} \nonumber\\
	&= \left(\bB^{(L-1)}\right)^{T} diag \left(\bF^{(L,L-1)}_{i}\odot\sigma^{''}\left(\bz^{(L-1)}\right)\right)\bB^{(L-1)}\nonumber\\ 
	&+\sum_{j=1}^{N_{L-1}} \sum_{I=1}^{L-2}\bigg[\bWW^{(L)}_{i,j}\sigma^{'}\left(\bz^{(L-1)}_{j}\right)\nonumber\\
	&\left(\bB^{(I)}\right)^{T} diag \left(\bF^{(L-1,I)}_{j}\odot\sigma^{''}\left(\bz^{(I)}\right)\right)\bB^{(I)}\bigg]\nonumber
	\end{align}
	Exchanging the summation over $I$ and summation over $j$:
	\begin{align}
	&\nabla^{2}_{\bx} \bz^{(L)}_{i} \nonumber\\
	&= \left(\bB^{(L-1)}\right)^{T} diag \left(\bF^{(L,L-1)}_{i}\odot\sigma^{''}\left(\bz^{(L-1)}\right)\right)\bB^{(L-1)}\nonumber\\ &+ \sum_{I=1}^{L-2}\sum_{j=1}^{N_{L-1}}\bWW^{(L)}_{i,j}\sigma^{'}\left(\bz^{(L-1)}_{j}\right)\bigg[\nonumber\\
	&\left(\bB^{(I)}\right)^{T} diag \left(\bF^{(L-1,I)}_{j}\odot\sigma^{''}\left(\bz^{(I)}\right)\right)\bB^{(I)}\bigg]\nonumber
	\end{align}
	Since $\bB^{(I)}$ is independent of $j$, we take it out of the summation over $j$:
	\begin{align}
	&\nabla^{2}_{\bx} \bz^{(L)}_{i} \nonumber\\
	&=\left(\bB^{(L-1)}\right)^{T} diag \left(\bF^{(L,L-1)}_{i}\odot\sigma^{''}\left(\bz^{(L-1)}\right)\right)\bB^{(L-1)}\nonumber\\ 
	&+ \sum_{I=1}^{L-2}\left(\bB^{(I)}\right)^{T}\bigg[\nonumber\\
	&\sum_{j=1}^{N_{L-1}}\bWW^{(L)}_{i,j}\sigma^{'}\left(\bz^{(L-1)}_{j}\right) diag \left(\bF^{(L-1,I)}_{j}\odot\sigma^{''}\left(\bz^{(I)}\right)\right)\bigg]\bB^{(I)}\nonumber
	\end{align}
	Using the property, $\alpha\left(diag(\bu)\right) + \beta\left( diag (\bv)\right) =  diag \left(\alpha \bu + \beta \bv\right)\ \forall \alpha, \beta \in \mathbb{R}, \bu, \bv \in \mathbb{R}^{n}$; we can move the summation inside the diagonal:
	\begin{align}
	\nabla^{2}_{\bx} \bz^{(L)}_{i} &= \left(\bB^{(L-1)}\right)^{T} diag \left(\bF^{(L,L-1)}_{i}\odot\sigma^{''}\left(\bz^{(L-1)}\right)\right)\bB^{(L-1)}\nonumber\\ &+ \sum_{I=1}^{L-2}\left(\bB^{(I)}\right)^{T}diag\bigg[ \nonumber\\
	&\sum_{j=1}^{N_{L-1}}\bWW^{(L)}_{i,j}\sigma^{'}\left(\bz^{(L-1)}_{j}\right)  \bigg(\bF^{(L-1,I)}_{j}\odot\sigma^{''}\left(\bz^{(I)}\right)\bigg)\bigg]\bB^{(I)}\nonumber
	\end{align}
	Since $\sigma^{''}\left(\bz^{(I)}\right)$ is independent of $j$, we can take it out of the summation over $j$:
	\begin{align}
	\nabla^{2}_{\bx} \bz^{(L)}_{i} &= \left(\bB^{(L-1)}\right)^{T} diag \left(\bF^{(L,L-1)}_{i}\odot\sigma^{''}\left(\bz^{(L-1)}\right)\right)\bB^{(L-1)}\nonumber\\ &+ \sum_{I=1}^{L-2}\left(\bB^{(I)}\right)^{T}diag\bigg[ \nonumber\\
	&\bigg(\sum_{j=1}^{N_{L-1}}\bWW^{(L)}_{i,j}\sigma^{'}\left(\bz^{(L-1)}_{j}\right)\bF^{(L-1,I)}_{j}\bigg)\odot\sigma^{''}\left(\bz^{(I)}\right)\bigg]\bB^{(I)}\nonumber
	\end{align}
	Using equation \eqref{F_induction_i}, we can replace $\sum_{j=1}^{N_{L-1}}\bWW^{(L)}_{i,j}\sigma^{'}\left(\bz^{(L-1)}_{j}\right)\bF^{(L-1,I)}_{j}$ with $\bF^{(L, I)}_{i}$:
	\begin{align}
	&\nabla^{2}_{\bx} \bz^{(L)}_{i} \nonumber\\
	&= \left(\bB^{(L-1)}\right)^{T} diag \left(\bF^{(L,L-1)}_{i}\odot\sigma^{''}\left(\bz^{(L-1)}\right)\right)\bB^{(L-1)}\nonumber\\ 
	&+ \sum_{I=1}^{L-2}\left(\bB^{(I)}\right)^{T}diag\bigg(  \bF^{(L, I)}_{i}\odot \sigma^{''}\left(\bz^{(I)}\right)\bigg)\bB^{(I)}\nonumber\\
	&\nabla^{2}_{\bx} \bz^{(L)}_{i} = \sum_{I=1}^{L-1}\left(\bB^{(I)}\right)^{T}diag\bigg(\bF^{(L,I)}_{i}\odot\sigma^{''}\left(\bz^{(I)}\right)\bigg)\bB^{(I)}\nonumber
	\end{align}
	
	\subsection{Proof of Theorem \ref{thm:single_layer_p_n_theorem}}\label{proof:single_layer_p_n_theorem}
	Using Lemma \ref{thm:deep_hessian_closed}, we have the following formula for $\nabla^{2}_{\bx} \left(\bz^{(2)}_{y} - \bz^{(2)}_{t}\right)$:
	\begin{align}
	&\nabla^{2}_{\bx} \left(\bz^{(2)}_{y} - \bz^{(2)}_{t}\right) \nonumber \\
	&= \left(\bWW^{(1)}\right)^{T} diag \bigg(\left(\bWW^{(2)}_{y} - \bWW^{(2)}_{t}\right)\odot\sigma^{''}\left(\bz^{(1)}\right)\bigg)\bWW^{(1)}\nonumber\\
	&= \sum_{i=1}^{N_{1}} \left(\bWW^{(2)}_{y,i} - \bWW^{(2)}_{t,i}\right)\sigma^{''}\left(\bz^{(1)}_{i}\right)\bWW^{(1)}_{i}\big(\bWW^{(1)}_{i}\big)^{T}\label{hessian_single_layer}
	\end{align}
	
	We are also given that the activation function $\sigma$ satisfies the following property:
	\begin{align}
	h_{L} \leq \sigma^{''}(x) \leq h_{U} \quad  \forall x \in \mathbb{R} \label{scalar_hessian_bound}
	\end{align}
	
	\begin{enumerate}[label=(\alph*)]
		\item 
		We have to prove the following linear matrix inequalities (LMIs):
		\begin{align}
		&\bN \preccurlyeq \nabla^{2}_{\bx} \left(\bz^{(2)}_{y} - \bz^{(2)}_{t}\right) \preccurlyeq \bP\qquad \forall \bx \in \mathbb{R}^{D} \label{single_layer_lmis}
		\end{align}
		where $\bP$ and $\bN$ are given as following:
		\begin{align}
		&\bP = \sum_{i=1}^{N_{1}} p_{i}\left(\bWW^{(2)}_{y,i} - \bWW^{(2)}_{t,i}\right)\bWW^{(1)}_{i}\left(\bWW^{(1)}_{i}\right)^{T}\label{bP_eq}\\
		&\bN = \sum_{i=1}^{N_{1}} n_{i}\left(\bWW^{(2)}_{y,i} - \bWW^{(2)}_{t,i}\right)\bWW^{(1)}_{i}\left(\bWW^{(1)}_{i}\right)^{T}\label{bN_eq}\\
		&p_{i} = \left\{\begin{array}{@{}lr@{}}
		h_{U}, & \bWW^{(2)}_{y,i} - \bWW^{(2)}_{t,i} \geq 0\\
		h_{L}, & \bWW^{(2)}_{y,i} - \bWW^{(2)}_{t,i} \leq 0\\
		\end{array}\right\}, \nonumber\\ 
		&n_{i} = \left\{\begin{array}{@{}lr@{}}
		h_{L}, & \bWW^{(2)}_{y,i} - \bWW^{(2)}_{t,i} \geq 0\\
		h_{U}, & \bWW^{(2)}_{y,i} - \bWW^{(2)}_{t,i} \leq 0\\
		\end{array}\right\}\label{p_i_n_i_equation}
		\end{align}
		
		We first prove: $\bN \preccurlyeq \nabla^{2}_{\bx}\left(\bz^{(2)}_{y} - \bz^{(2)}_{t}\right) \quad \forall \bx \in \mathbb{R}^{D}$:\\ 
		We substitute $\nabla^{2}_{\bx}\left(\bz^{(2)}_{y} - \bz^{(2)}_{t}\right)$ and $\bN$ from equations \eqref{hessian_single_layer} and \eqref{bN_eq} respectively in $\nabla^{2}_{\bx}\left(\bz^{(2)}_{y} - \bz^{(2)}_{t}\right) - \bN$:
		\begin{align*}
		&\nabla^{2}_{\bx}\left(\bz^{(2)}_{y} - \bz^{(2)}_{t}\right) - \bN \nonumber\\
		&= \sum_{i=1}^{N_{1}} \left(\bWW^{(2)}_{y,i} - \bWW^{(2)}_{t,i}\right)\left(\sigma^{''}\left(\bz^{(1)}_{i}\right) - n_{i}\right) \bWW^{(1)}_{i}\left(\bWW^{(1)}_{i}\right)^{T}
		\end{align*}
		Thus $\nabla^{2}_{\bx}\left(\bz^{(2)}_{y} - \bz^{(2)}_{t}\right) - \bN$ is a weighted sum of symmetric rank one matrices i.e, $\bWW^{(1)}_{i}\left(\bWW^{(1)}_{i}\right)^{T}$ and it is PSD if and only if coefficient of each rank one matrix i.e, $\left(\bWW^{(2)}_{y,i} - \bWW^{(2)}_{t,i}\right)\left(\sigma^{''}\left(\bz^{(1)}_{i}\right) - n_{i}\right)$ is positive. Using equations \eqref{scalar_hessian_bound} and \eqref{p_i_n_i_equation}, we have the following:
		\begin{align}
		&\left(\bWW^{(2)}_{y,i} - \bWW^{(2)}_{t,i}\right) \geq 0 \implies  n_{i}=h_{L} \nonumber\\
		&\implies \left(\sigma^{''}\left(\bz^{(1)}_{i}\right) - n_{i}\right) \geq 0\qquad \forall i \in [N_{1}],\ \forall \bx \in \mathbb{R}^{D}\nonumber\\
		&\left(\bWW^{(2)}_{y,i} - \bWW^{(2)}_{t,i}\right) \leq 0 \implies  n_{i}=h_{U} \nonumber\\
		&\implies \left(\sigma^{''}\left(\bz^{(1)}_{i}\right) - n_{i}\right) \leq 0\qquad \forall i \in [N_{1}],\ \forall \bx \in \mathbb{R}^{D}\nonumber
		\end{align}
		Putting the above results together we have:
		\begin{align}
		&\left(\bWW^{(2)}_{y,i} - \bWW^{(2)}_{t,i}\right)\left(\sigma^{''}\left(\bz^{(1)}_{i}\right) - n_{i}\right) \geq 0\nonumber\\ 
		&\forall i \in [N_{1}],\ \forall \bx \in \mathbb{R}^{D} \label{n_always_pos}
		\end{align}
		Thus $\nabla^{2}_{\bx}\left(\bz^{(2)}_{y} - \bz^{(2)}_{t}\right) - \bN$ is a PSD matrix i.e:
		\begin{align}
		&\nabla^{2}_{\bx}\left(\bz^{(2)}_{y} - \bz^{(2)}_{t}\right) - \bN \nonumber\\
		&= \sum_{i=1}^{N_{1}} \underbrace{\left(\bWW^{(2)}_{y,i} - \bWW^{(2)}_{t,i}\right)\left(\sigma^{''}\left(\bz^{(1)}_{i}\right) - n_{i}\right)}_{\text{always positive using eq.}\ \eqref{n_always_pos}} \bWW^{(1)}_{i}\left(\bWW^{(1)}_{i}\right)^{T} \nonumber\\ 
		&\implies \bN \preccurlyeq \nabla^{2}_{\bx}\left(\bz^{(2)}_{y} - \bz^{(2)}_{t}\right)\qquad \forall \bx \in \mathbb{R}^{D}\label{neg_proven}
		\end{align}
		Now we prove that $\nabla^{2}_{\bx}\left( \bz^{(2)}_{y} - \bz^{(2)}_{t}\right) \preccurlyeq \bP \quad \forall \bx \in \mathbb{R}^{D}$:\\
		We substitute $\nabla^{2}_{\bx}\left(\bz^{(2)}_{y} - \bz^{(2)}_{t}\right)$ and $\bP$ from equations \eqref{hessian_single_layer} and \eqref{bN_eq} respectively in $\bP - \nabla^{2}_{\bx}\left(\bz^{(2)}_{y} - \bz^{(2)}_{t}\right)$:
		\begin{align*}
		&\bP - \nabla^{2}_{\bx}\left(\bz^{(2)}_{y} - \bz^{(2)}_{t}\right) \nonumber\\
		&= \sum_{i=1}^{N_{1}} \left(\bWW^{(2)}_{y,i} - \bWW^{(2)}_{t,i}\right)\left(p_{i} - \sigma^{''}\left(\bz^{(1)}_{i}\right)\right) \bWW^{(1)}_{i}\big(\bWW^{(1)}_{i}\big)^{T}
		\end{align*}
		Thus $\bP - \nabla^{2}_{\bx}\left(\bz^{(2)}_{y} - \bz^{(2)}_{t}\right)$ is a weighted sum of symmetric rank one matrices i.e, $\bWW^{(1)}_{i}\left(\bWW^{(1)}_{i}\right)^{T}$ and it is PSD if and only if coefficient of each rank one matrix i.e, $\left(\bWW^{(2)}_{y,i} - \bWW^{(2)}_{t,i}\right)\left(p_{i} - \sigma^{''}\left(\bz^{(1)}_{i}\right)\right)$ is positive. Using equations \eqref{scalar_hessian_bound} and \eqref{p_i_n_i_equation}, we have the following:
		\begin{align}
		&\left(\bWW^{(2)}_{y,i} - \bWW^{(2)}_{t,i}\right) \geq 0 \implies  p_{i}=h_{U} \nonumber\\
		&\implies \left(p_{i} - \sigma^{''}\left(\bz^{(1)}_{i}\right)\right) \geq 0\qquad\forall i\in N_{1},\ \bx \in \mathbb{R}^{D}\nonumber\\
		&\left(\bWW^{(2)}_{y,i} - \bWW^{(2)}_{t,i}\right) \leq 0 \implies  p_{i}=h_{L} \nonumber\\
		&\implies \left(p_{i} - \sigma^{''}\left(\bz^{(1)}_{i}\right)\right) \leq 0\qquad\forall i\in N_{1},\ \bx \in \mathbb{R}^{D}\nonumber
		\end{align}
        Putting the above results together we have:
		\begin{align}
		&\implies \left(\bWW^{(2)}_{y,i} - \bWW^{(2)}_{t,i}\right)\left(p_{i} - \sigma^{''}\left(\bz^{(1)}_{i}\right)\right) \geq 0\nonumber\\ 
		&\forall i \in [N_{1}],\ \bx \in \mathbb{R}^{D} \label{p_always_pos}
		\end{align}
		Thus $\bP - \nabla^{2}_{\bx}\left(\bz^{(2)}_{y} - \bz^{(2)}_{t}\right)$ is PSD matrix i.e:
		\begin{align}
		&\bP - \nabla^{2}_{\bx}\left(\bz^{(2)}_{y} - \bz^{(2)}_{t}\right) \nonumber\\
		&= \sum_{i=1}^{N_{1}} \underbrace{\left(\bWW^{(2)}_{y,i} - \bWW^{(2)}_{t,i}\right)\left(p_{i} - \sigma^{''}\left(\bz^{(1)}_{i}\right)\right)}_{\text{always positive using eq.}\ \eqref{p_always_pos}} \bWW^{(1)}_{i}\left(\bWW^{(1)}_{i}\right)^{T} \nonumber\\ 
		&\implies \bP \succcurlyeq \nabla^{2}_{\bx}\left(\bz^{(2)}_{y} - \bz^{(2)}_{t}\right)\qquad \forall \bx \in \mathbb{R}^{D}\label{pos_proven}
		\end{align}
		Thus by proving the LMIs \eqref{neg_proven} and \eqref{pos_proven}, we prove \eqref{single_layer_lmis}.
		\item 
		We have to prove that if $h_{U} \geq 0$ and $h_{L} \leq 0$, $\bP$ is a PSD matrix, $\bN$ is a NSD matrix.\\
		We are given $h_{U} \geq 0$,\ $h_{L} \leq 0$.
		Using equation \eqref{p_i_n_i_equation}, we have the following:
		\begin{align}
		&\left(\bWW^{(2)}_{y,i} - \bWW^{(2)}_{t,i}\right) \geq 0 \implies  p_{i}=h_{U} \geq 0 \nonumber\\
		&\implies p_{i}\left(\bWW^{(2)}_{y,i} - \bWW^{(2)}_{t,i}\right) \geq 0\nonumber\\
		&\left(\bWW^{(2)}_{y,i} - \bWW^{(2)}_{t,i}\right) \leq 0 \implies  p_{i}=h_{L} \leq 0 \nonumber \\
		&\implies p_{i}\left(\bWW^{(2)}_{y,i} - \bWW^{(2)}_{t,i}\right) \geq 0\nonumber
		\end{align}
		Putting these results together we have:
		\begin{align}
		&\implies p_{i}\left(\bWW^{(2)}_{y,i} - \bWW^{(2)}_{t,i}\right) \geq 0\qquad \forall i \in [N_{1}] \label{p_wdiff}
		\end{align}
		Thus $\bP$ is a weighted sum of symmetric rank one matrices i.e, $\bWW^{(1)}_{i}\left(\bWW^{(1)}_{i}\right)^{T}$ and each coefficient $p_{i}\left(\bWW^{(2)}_{y,i} - \bWW^{(2)}_{t,i}\right)$ is positive.
		\begin{align}
		&\bP = \sum_{i=1}^{N_{1}} \underbrace{p_{i}\left(\bWW^{(2)}_{y,i} - \bWW^{(2)}_{t,i}\right)}_{\text{always positive using eq. \eqref{p_wdiff}} }\bWW^{(1)}_{i}\left(\bWW^{(1)}_{i}\right)^{T} \succcurlyeq 0\nonumber
		\end{align}
		Using equation \eqref{p_i_n_i_equation}, we have the following:
		\begin{align}
		&\left(\bWW^{(2)}_{y,i} - \bWW^{(2)}_{t,i}\right) \geq 0 \implies  n_{i}=h_{L} \leq 0 \nonumber\\
		&\implies n_{i}\left(\bWW^{(2)}_{y,i} - \bWW^{(2)}_{t,i}\right) \leq 0\nonumber\\
		&\left(\bWW^{(2)}_{y,i} - \bWW^{(2)}_{t,i}\right) \leq 0 \implies  n_{i}=h_{U} \geq 0 \nonumber\\
		&\implies n_{i}\left(\bWW^{(2)}_{y,i} - \bWW^{(2)}_{t,i}\right) \leq 0\nonumber
		\end{align}
        Putting these results together we have:
		\begin{align}
		&\implies n_{i}\left(\bWW^{(2)}_{y,i} - \bWW^{(2)}_{t,i}\right) \geq 0\qquad \forall i \in [N_{1}] \label{n_wdiff}
		\end{align}
		\begin{align}
		&\bN = \sum_{i=1}^{N_{1}} \underbrace{n_{i}\left(\bWW^{(2)}_{y,i} - \bWW^{(2)}_{t,i}\right)}_{\text{always positive using eq. \eqref{n_wdiff}} }\bWW^{(1)}_{i}\left(\bWW^{(1)}_{i}\right)^{T} \preccurlyeq 0\nonumber
		\end{align}
		Thus $\bP$ is a PSD and $\bN$ is a NSD matrix if $h_{U} \geq 0$ and $h_{L} \leq 0$.
		\item 
		
		We have to prove the following global bounds on the eigenvalues of $\nabla^{2}_{\bx}(\bz^{(2)}_{y} - \bz^{(2)}_{t})$:
		\begin{align}
		&m\bI \preccurlyeq \nabla^{2}_{\bx}\left(\bz^{(2)}_{y} - \bz^{(2)}_{t}\right) \preccurlyeq M\bI,\nonumber\\
		&\text{where } M = \max_{\|\bv\|=1} \bv^{T}\bP\bv,\ m = \min_{\|\bv\|=1} \bv^{T}\bN\bv \nonumber
		\end{align}
		Since $\nabla_{\bx}^{2}\left(\bz^{(2)}_{y} - \bz^{(2)}_{t}\right) \preccurlyeq \bP \quad \forall \bx \in \mathbb{R}^{D}$:
		\begin{align} &\bv^{T}\left[\nabla_{\bx}^{2}\left(\bz^{(2)}_{y} - \bz^{(2)}_{t}\right)\right]\bv \leq \bv^{T}\bP\bv \nonumber\\
		&\forall \bv \in \mathbb{R}^{D},\ \forall \bx \in \mathbb{R}^{D} \label{all_v_x_P}
		\end{align}
		Let $\bv^{*}$,  $\bx^{*}$ be vectors such that:
		\begin{align*}
		&(\bv^{*})^{T}\left[\nabla_{\bx^{*}}^{2}\left(\bz^{(2)}_{y} - \bz^{(2)}_{t}\right)\right]\bv^{*} \\
		&= \max_{\bx} \max_{\|\bv\|=1}\bv^{T}\left[\nabla_{\bx}^{2}\left(\bz^{(2)}_{y} - \bz^{(2)}_{t}\right)\right]\bv
		\end{align*}
		Thus using inequality \eqref{all_v_x_P}:
		\begin{align}
		&(\bv^{*})^{T}\left[\nabla_{\bx^{*}}^{2}\left(\bz^{(2)}_{y} - \bz^{(2)}_{t}\right)\right]\bv^{*} \leq \max_{\|\bv\|=1}\bv^{T}\bP\bv\label{p_M_proven}
		\end{align}
		Since $\bN \preccurlyeq\nabla_{\bx}^{2}\left(\bz^{(2)}_{y} - \bz^{(2)}_{t}\right) \quad \forall \bx \in \mathbb{R}^{D}$:
		\begin{align} 
		&\bv^{T}\bN\bv \leq \bv^{T}\left[\nabla_{\bx}^{2}\left(\bz^{(2)}_{y} - \bz^{(2)}_{t}\right)\right]\bv \nonumber \\ 
		&\forall \bv \in \mathbb{R}^{D},\ \forall \bx \in \mathbb{R}^{D} \label{all_v_x_N}
		\end{align}
		Let $\bv^{*}$,  $\bx^{*}$ be vectors such that:
		\begin{align*}
		&(\bv^{*})^{T}\left[\nabla_{\bx^{*}}^{2}\left(\bz^{(2)}_{y} - \bz^{(2)}_{t}\right)\right]\bv^{*} \\
		&= \min_{\bx} \min_{\|\bv\|=1}\bv^{T}\left[\nabla_{\bx}^{2}\left(\bz^{(2)}_{y} - \bz^{(2)}_{t}\right)\right]\bv
		\end{align*}
		Thus using inequality \eqref{all_v_x_N}:
		\begin{align}
		&(\bv^{*})^{T}\left[\nabla_{\bx^{*}}^{2}\left(\bz^{(2)}_{y} - \bz^{(2)}_{t}\right)\right]\bv^{*} \geq \min_{\|\bv\|=1}\bv^{T}\bN\bv\label{n_m_proven}
		\end{align}
		Using the inequalities \eqref{p_M_proven} and \eqref{n_m_proven}, we get:
		$$ m\bI \preccurlyeq \nabla^{2}_{\bx}\left(\bz^{(2)}_{y} - \bz^{(2)}_{t}\right) \preccurlyeq M\bI$$ 
		where $M = \max_{\|\bv\|=1} \bv^{T}\bP\bv,\ m = \min_{\|\bv\|=1} \bv^{T}\bN\bv$
	\end{enumerate}
	
	\subsection{Proof of Theorem \ref{thm:L_layer_p_n_theorem}}\label{proof:L_layer_p_n_theorem}
	We are given that the activation function $\sigma$ is such that $\sigma^{'},\ \sigma^{''}$ are bounded, i.e:
	\begin{align}
	&|\sigma^{'}(x)| \leq g,\ |\sigma^{''}(x)| \leq h \qquad \forall x \in \mathbb{R}\label{grad_hess_bound}
	\end{align}	
	We have to prove the following:
	\begin{align}
	&\left\|\nabla^{2}_{\bx} \bz^{(L)}_{i}\right\| \leq h\sum_{I=1}^{L-1} \left(r^{(I)}\right)^{2}\max_{j}\left(\bS^{(I)}_{i,j}\right)\ \ \ \forall \bx \in \mathbb{R}^{D}  \nonumber
	\end{align}
	where $\bS^{(L,I)}$ is a matrix of size $N_{L} \times N_{I}$ defined as follows:
	\begin{align}
	\bS^{(L,I)} &=
	\begin{dcases*}
	\left|\bWW^{(L)}\right| &\quad $I = L-1$\\
	g \left|\bWW^{(L)}\right|\bS^{(L-1,I)}
	&\quad $I \in [L-2]$
	\end{dcases*}\label{bq_def}
	\end{align}
	and $r^{(I)}$ is a scalar defined as follows:
	\begin{align}
	r^{(I)} &=
	\begin{dcases*}
	\left\|\bWW^{(1)}\right\| &\quad $I=1$\\
	g\left\|\bWW^{(I)}\right\|r^{(I-1)} &\quad $I \in [2, L-1]$
	\end{dcases*}\label{br_def}
	\end{align}
	We will prove the same in $3$ steps. \\
	In step (a), we will prove:
	\begin{align}
	%&\left|\bF^{(L, I)}_{i,j}\right| \leq \bP^{(L,I)}_{i,j},\quad 
	&\left|\bF^{(L, I)}_{i,j}\right| \leq \bS^{(L,I)}_{i,j}\qquad \forall \bx \in \mathbb{R}^{D} \label{diag_vec_prove}	
	\end{align}
	In step (b), we will prove:
	\begin{align}
	\left\|\bB^{(I)}\right\| \leq r^{(I)},\qquad \forall \bx \in \mathbb{R}^{D}\label{B_norm_res}
	\end{align}
	In step (c), we will use (a) and (b) to prove:
	\begin{align}
	\left\|\nabla^{2}_{\bx} \bz^{(L)}_{i}\right\| \leq   h\sum_{I=1}^{L-1} \left(r^{(I)}\right)^{2}\max_{j}\left(\bS^{(L,I)}_{i,j}\right)\label{hess_norm_general_bound}
	\end{align}	
	
	Note that $\bB^{(I)}$ and $\bF^{(L,I)}$ are defined using \eqref{B_def} and \eqref{F_def} respectively. \\

	\begin{enumerate}[label=(\alph*)]
		\item We have to prove that for $L \geq 2,\ I \in [L-1],\ i \in N_{L},\ j \in N_{I}$:
		\begin{align}
		%&\left|\bF^{(L, I)}_{i,j}\right| \leq \bP^{(L,I)}_{i,j},\quad 
		\left|\bF^{(L, I)}_{i,j}\right| \leq \bS^{(L,I)}_{i,j}\qquad \forall \bx \in \mathbb{R}^{D} \nonumber		
		\end{align}
		where $\bS^{(L,I)}$ is a matrix of size $N_{I} \times N_{J}$ defined as follows:
		\begin{align*}
		\bS^{(L,I)} &=
		\begin{dcases*}
		\left|\bWW^{(L)}\right| &\quad $I = L-1$\\
		g \left|\bWW^{(L)}\right|\bS^{(L-1,J)}
		&\quad $I \in [L-2]$
		\end{dcases*}
		\end{align*}
		
		We first prove the case when $I = L-1$. \\
		Using equation \eqref{F_base}, 
		$\bF^{(L, L-1)}_{i,j} = \bWW^{(L)}_{i,j}$.\\
		Since 
		$\bS^{(L,L-1)}_{i,j} = \left|\bWW^{(L)}_{i,j}\right|$:
		$$%\left|\bF^{(L, L-1)}_{i,j}\right| = \bP^{(L)}_{i,j},\quad 
		\left|\bF^{(L, L-1)}_{i,j}\right| = \bS^{(L,L-1)}_{i,j}$$
		Hence for $L\ge2,\ I=L-1$, we have equality in \eqref{diag_vec_prove}. Hence proved.\\
		Now, we will use proof by induction. \\	
		To prove the base case $L=2$, note that $I = L-1 = 1$ is the only possible value for $I$. Thus, using the result for $I=L-1$, the theorem holds for $L=2$. This proves the base case.\\
		Now we assume the induction hypothesis is true for depth $=L-1,\ I \in [L-2]$. and prove for depth $=L,\ I \in [L-1]$. Since for $I=L-1$, we have proven already, we prove for $I \leq L-2$.\\
		Using equation \eqref{F_induction_i}, we have the following formula for $\bF^{(L,I)}_{i}$: 
		\begin{align}
		&\bF^{(L,I)}_{i} = \sum_{k=1}^{N_{L-1}} \bWW^{(L)}_{i,k}\sigma^{'}\left(\bz^{(L-1)}_{k}\right)\bF^{(L-1, I)}_{k}\nonumber
		\end{align}
		Taking the $j^{th}$ element of the vectors on both sides:
		\begin{align}
		&\bF^{(L,I)}_{i,j} = \sum_{k=1}^{N_{L-1}} \bWW^{(L)}_{i,k} \sigma^{'}\left(\bz^{(L-1)}_{k}\right)\bF^{(L-1, I)}_{k,j}\label{F_i_j_eq}
		\end{align}
		By induction hypothesis, we know that:
		\begin{align}
		\left|\bF^{(L-1, I)}_{k,j}\right| \leq \bS^{(L-1, I)}_{k,j}\label{ind_hypo_F_kj}
		\end{align}
		Using the absolute value properties for equation \eqref{F_i_j_eq}, we have:
		\begin{align}
		&\left|\bF^{(L,I)}_{i,j}\right| = 
		\left|\sum_{k=1}^{N_{L-1}} \bWW^{(L)}_{i,k} \sigma^{'}\left(\bz^{(L-1)}_{k}\right)\bF^{(L-1, I)}_{k,j}\right|\nonumber\\
		&\left|\bF^{(L,I)}_{i,j}\right| \leq 
		\sum_{k=1}^{N_{L-1}} \left|\bWW^{(L)}_{i,k} \sigma^{'}\left(\bz^{(L-1)}_{k}\right)\bF^{(L-1, I)}_{k,j}\right|\nonumber\\
		&\left|\bF^{(L,I)}_{i,j}\right| \leq 
		\sum_{k=1}^{N_{L-1}} \left|\bWW^{(L)}_{i,k}\right| \left|\sigma^{'}\left(\bz^{(L-1)}_{k}\right)\right|\left|\bF^{(L-1, I)}_{k,j}\right|\nonumber
		\end{align}
		Using $|\sigma^{'}(x)| \leq g\quad \forall x \in \mathbb{R}\ $ (inequality \eqref{grad_hess_bound}) :
		$$\left|\bF^{(L,I)}_{i,j}\right| \leq 
		g\sum_{k=1}^{N_{L-1}} \left|\bWW^{(L)}_{i,k}\right|\left|\bF^{(L-1, I)}_{k,j}\right|\nonumber$$
		Using the induction hypothesis (inequality \eqref{ind_hypo_F_kj}): 
		$$\left|\bF^{(L,I)}_{i,j}\right| \leq 
		g\sum_{k=1}^{N_{L-1}} \left|\bWW^{(L)}_{i,k}\right|\left|\bS^{(L-1, I)}_{k,j}\right|\nonumber$$
		Using equation \eqref{bq_def} for definition of $\bS^{(L,I)}_{i,j}$:
		$$\left|\bF^{(L,I)}_{i,j}\right| \leq \bS^{(L, I)}_{i,j}\nonumber$$
		%	Proof for $\bq^{(I)}_{j}$ follows exactly the same line of arguments.
		Hence we prove \eqref{diag_vec_prove} for all $L \ge 2$ and $I \leq L-1$ using induction.
		\item 
		We have to prove that for $1 \leq I \leq M-1$:
		\begin{align}
		\left\|\bB^{(I)}\right\| \leq r^{(I)},\qquad \forall \bx \in \mathbb{R}^{D}\nonumber%\label{B_norm_res}
		\end{align}
		where $r^{(I)}$ is a scalar given as follows:
		\begin{align}
		r^{(I)} &=
		\begin{dcases*}
		\left\|\bWW^{(1)}\right\| &\quad $I=1$\\
		g\left\|\bWW^{(I)}\right\|r^{(I-1)} &\quad $I \in [2, L-1]$
		\end{dcases*}\nonumber%\label{br_def}
		\end{align}

		Using equation \eqref{B_base}, for $I=1$ we have: 
		\begin{align}
		\left\|\bB^{(1)}\right\| &= \left\|\bWW^{(1)}\right\| = r^{(1)} \label{B_base_norm}
		\end{align}
		Using equation \eqref{B_induction}, for $I>1$, we have:
		\begin{align}
		\left\|\bB^{(I)}\right\| &= \left\|\bWW^{(I)}diag \left(\sigma^{'}\left(\bz^{(I-1)}\right)\right)\bB^{(I-1)}\right\|\nonumber
		\end{align}
		\begin{align}
		\left\|\bB^{(I)}\right\|&\leq\left\|\bWW^{(I)}\right\|\left\|diag \left(\sigma^{'}\left(\bz^{(I-1)}\right)\right)\right\|\left\|\bB^{(I-1)}\right\|\nonumber
		\end{align}
		Since $\left\|diag \left(\sigma^{'}\left(\bz^{(I-1)}\right)\right)\right\| = \max_{j} \left|\sigma^{'}\left(\bz^{(I-1)}_{j}\right)\right|$, using equation \eqref{grad_hess_bound}:
		\begin{align}
		&\left\|\bB^{(I)}\right\|&\leq g\left\|\bWW^{(I)}\right\| \left\|\bB^{(I-1)}\right\|\leq g\left\|\bWW^{(I)}\right\| r^{(I-1)} \label{B_induction_norm}
		\end{align}
		Using inequalities \eqref{B_base_norm} and \eqref{B_induction_norm}, the proof follows using induction.
		\item 
		We have to prove that:
		\begin{align}
		\left\|\nabla^{2}_{\bx} \bz^{(L)}_{i}\right\| \leq  h\sum_{I=1}^{L-1} \left(r^{(I)}\right)^{2}\max_{j}\left(\bS^{(I)}_{i,j}\right)\nonumber%\label{hess_norm_general_bound}
		\end{align}	
		Using Lemma \ref{thm:deep_hessian_closed}, we have the following equation for $\nabla^{2}_{\bx} \bz^{(L)}_{i} $:
		\begin{align}
		\nabla^{2}_{\bx} \bz^{(L)}_{i} = \sum_{I=1}^{L-1}\left(\bB^{(I)}\right)^{T}diag\bigg(\bF^{(L,I)}_{i} \odot\sigma^{''}\left(\bz^{(I)}\right)\bigg) \bB^{(I)}\nonumber
		\end{align}
		Using the properties of norm we have:
		\begin{align}
		&\left\|\nabla^{2}_{\bx} \bz^{(L)}_{i}\right\| \nonumber\\
		&= \left\|\sum_{I=1}^{L-1}\left(\bB^{(I)}\right)^{T}diag\left(\bF^{(L,I)}_{i} \odot\sigma^{''}\left(\bz^{(I)}\right)\right) \bB^{(I)}\right\|\nonumber\\
		%& \leq \sum_{I=1}^{L-1}\left\|\left(\bB^{(I)}\right)^{T}diag\bigg(\left(\bF^{(L,I)}_{y} - \bF^{(L,I)}_{t}\right)\odot\sigma^{''}\left(\bz^{(I)}\right)\bigg) \bB^{(I)}\right\|\nonumber\\
		& \leq \sum_{I=1}^{L-1}\left\|diag\left(\bF^{(L,I)}_{i} \odot\sigma^{''}\left(\bz^{(I)}\right)\right)\right\| \left\|\bB^{(I)}\right\|^{2}\nonumber\\
		&\leq 
		\sum_{I=1}^{L-1}\max_{j}\bigg(\left|\bF^{(L,I)}_{i,j} \sigma^{''}\left(\bz^{(I)}_{j}\right)\right|\bigg) \left\|\bB^{(I)}\right\|^{2}\nonumber
		\end{align}
		In the last inequality, we use the property that norm of a diagonal matrix is the maximum absolute value of the diagonal element. Using the product property of absolute value, we get:
		\begin{align}
		\left\|\nabla^{2}_{\bx} \bz^{(L)}_{i} \right\| &\leq 
		\sum_{I=1}^{L-1}\max_{j}\bigg(\left|\bF^{(L,I)}_{i,j}\right|\left|\sigma^{''}\left(\bz^{(I)}_{j}\right)\right|\bigg) \left\|\bB^{(I)}\right\|^{2}\nonumber
		\end{align}
		Since $\left|\bF^{(L,I)}_{i,j}\right|$ and $\left|\sigma^{''}\left(\bz^{(I)}_{j}\right)\right|$ are positive terms:
		\begin{align}
		&\left\|\nabla^{2}_{\bx} \bz^{(L)}_{i} \right\| \nonumber\\ 
		&\leq \sum_{I=1}^{L-1}\max_{j}\bigg(\left|\bF^{(L,I)}_{i,j} \right|\bigg)\max_{j}\bigg(\left|\sigma^{''}\left(\bz^{(I)}_{j}\right)\right|\bigg) \left\|\bB^{(I)}\right\|^{2}\nonumber
		\end{align}
		Since $\left\|\sigma^{''}\right\|$ is bounded by $h$:
		\begin{align}
		\left\|\nabla^{2}_{\bx} \bz^{(L)}_{i} \right\| &\leq 
		h\sum_{I=1}^{L-1}\max_{j}\bigg(\left|\bF^{(L,I)}_{i,j} \right|\bigg) \left\|\bB^{(I)}\right\|^{2}\nonumber
		\end{align}
		Using inequality \eqref{diag_vec_prove}:
		\begin{align}
		\left\|\nabla^{2}_{\bx} \bz^{(L)}_{i} \right\| &\leq 
		h\sum_{I=1}^{L-1}\max_{j}\left(\bS^{(I)}_{i,j}\right) \left\|\bB^{(I)}\right\|^{2}\nonumber
		\end{align}
		Using inequality \eqref{B_norm_res}:
		\begin{align}
		\left\|\nabla^{2}_{\bx} \bz^{(L)}_{i} \right\| &\leq 
		h\sum_{I=1}^{L-1}\left(r^{(I)}\right)^{2}\max_{j}\left(\bS^{(I)}_{i,j}\right)\quad \forall \bx \in \mathbb{R}^{D} \nonumber
		\end{align}
	\end{enumerate}
	
	\subsection{Proof of Theorem \ref{thm:rand_smooth}}\label{proof:rand_smooth}	
	\begin{theorem}\label{thm:rand_smooth} For a binary classifier $f$, let $g$ denote the indicator function such that $g(\bx)=1 \iff f(\bx) > 0,\ g(\bx)=0\ \text{otherwise} $. Let $\hat{g}$ be the function constructed by applying randomized smoothing on $g$ such that:
		$$ \hat{g}\left(\bu\right) = \frac{1}{(2\pi s^{2})^{n/2}} \int_{\mathbb{R}^{D}}g(\bv)\exp\left(-\frac{\|\bv-\bu\|^{2}}{2s^{2}}\right)d\bv $$	
		then the curvature of the resulting function $\hat{g}$ is bounded i.e:
		$$ -\frac{\bI}{s^{2}}  \preccurlyeq \nabla^{2}_{\bu}\ \hat{g} \preccurlyeq \frac{\bI}{s^{2}} $$
	\end{theorem}
	\begin{proof}
		\begin{align*}
		&\nabla_{\bu}\ \hat{g}\left(\bu\right) \\
		&= \frac{1}{(2\pi s^{2})^{n/2}}\int_{\mathbb{R}^{D}}g(\bv)\frac{(\bv-\bu)}{s^{2}}\exp\left(-\frac{\|\bv-\bu\|^{2}}{2s^{2}}\right)d\bv \\
		&\nabla^{2}_{\bu}\ \hat{g}\left(\bu\right) \\
		&= \frac{1}{(2\pi s^{2})^{n/2}}\int_{\mathbb{R}^{D}}g(\bv)\frac{-\bI}{s^{2}}\exp\left(-\frac{\|\bv-\bu\|^{2}}{2s^{2}}\right)d\bv\\ &+ \frac{1}{(2\pi s^{2})^{n/2}}\int_{\mathbb{R}^{D}}g(\bv)\frac{(\bv-\bu)(\bv-\bu)^{T}}{s^{4}}\bigg[\\
		&\quad \exp\left(-\frac{\|\bv-\bu\|^{2}}{2s^{2}}\right)\bigg]d\bv
		\end{align*}
		Since $0 \leq g(\bv) \leq 1$,\ $-\bI/s^{2} \preccurlyeq 0$,\ $(\bv-\bu)(\bv-\bu)^{T} \succcurlyeq 0$ and $\exp(x) \geq 0\ \forall x$:
		\begin{align*}
		\nabla^{2}_{\bu}\ \hat{g}\left(\bu\right) &= \frac{1}{(2\pi s^{2})^{n/2}}\int_{\mathbb{R}^{D}}\underbrace{g(\bv)\frac{-\bI}{s^{2}}\exp\left(-\frac{\|\bv-\bu\|^{2}}{2s^{2}}\right)d\bv}_{\text{Negative Semi-Definite}}\\ &+ \frac{1}{(2\pi s^{2})^{n/2}}\int_{\mathbb{R}^{D}}\underbrace{g(\bv)\frac{(\bv-\bu)(\bv-\bu)^{T}}{s^{4}}}_{\text{Positive Semi-Definite}}\bigg[\nonumber\\
		&\exp\left(-\frac{\|\bv-\bu\|^{2}}{2s^{2}}\right)\bigg]d\bv
		\end{align*}
		\begin{align*}
		\nabla^{2}_{\bu}\ \hat{g}\left(\bu\right) &\preccurlyeq  \frac{1}{(2\pi s^{2})^{n/2}}\int_{\mathbb{R}^{D}}\frac{(\bv-\bu)(\bv-\bu)^{T}}{s^{4}}\bigg[\\
		&\exp\left(-\frac{\|\bv-\bu\|^{2}}{2s^{2}}\right)\bigg]d\bv\\
		\nabla^{2}_{\bu}\ \hat{g}\left(\bu\right) &\preccurlyeq \frac{1}{(2\pi s^{2})^{n/2}}\int_{\mathbb{R}^{D}}\frac{\bq\bq^{T}}{s^{4}}\exp\left(-\frac{\|\bq\|^{2}}{2s^{2}}\right)d\bq\\
		\nabla^{2}_{\bu}\ \hat{g}\left(\bu\right) &\preccurlyeq \frac{\bI}{s^{2}}\\
		\nabla^{2}_{\bu}\ \hat{g}\left(\bu\right) &\succcurlyeq  \frac{1}{(2\pi s^{2})^{n/2}}\int_{\mathbb{R}^{D}}\frac{-\bI}{s^{2}}\exp\left(-\frac{\|\bv-\bu\|^{2}}{2s^{2}}\right)d\bv\\
		\nabla^{2}_{\bu}\ \hat{g}\left(\bu\right) &\succcurlyeq -\frac{\bI}{s^{2}}
		\end{align*}
	\end{proof}

	\section{Computing $g,h,$ $h_{U}$ and $h_{L}$ for different activation functions}  \label{grad_hess_bounds_appendix}
	
	\subsection{Softplus activation}
	For softplus activation, we have the following. We use $S(x)$ to denote sigmoid:
	\begin{align*}
	&\sigma(x) = \log(1 + \exp(x))\\
	&\sigma^{'}(x) = S(x)\\
	&\sigma^{''}(x) = S(x)(1-S(x))
	\end{align*}
	To bound $S(x)(1-S(x))$, let $\alpha$ denote $S(x)$. We know that $0\leq \alpha \leq 1$: \\
	$$\alpha(1-\alpha) = \frac{1}{4} - \bigg(\frac{1}{2} - \alpha\bigg)^{2}$$
	Thus, $S(x)(1-S(x))$ is maximum at $S(x)=1/2$ and minimum at $S(x)=0$ and $S(x)=1$. The maximum value is 0.25 and minimum value is 0.
	$$ 0 \leq S(x)(1-S(x)) \leq 0.25 \implies 0 \leq \sigma^{''}(x) \leq 0.25$$
	Thus, $h_{U}=0.25,\ h_{L}=0$ (for use in Theorem \ref{thm:single_layer_p_n_theorem}) and $g=1,\ h=0.25$ (for use in Theorem \ref{thm:L_layer_p_n_theorem}). 
	\subsection{Sigmoid activation}
	For sigmoid activation, we have the following. We use $S(x)$ to denote sigmoid:
	\begin{align*}
	&\sigma(x) = S(x) = \frac{1}{1 + \exp(-x)}\\
	&\sigma^{'}(x) = S(x)(1-S(x))\\
	&\sigma^{''}(x) = S(x)(1-S(x))(1-2S(x))
	\end{align*}
	The second derivative of sigmoid $(\sigma^{''}(x))$ can be bounded using standard differentiation. Let $\alpha$ denote $S(x)$. We know that $0\leq \alpha \leq 1$:\\
	$$ h_{L} \leq \sigma^{''}(x) \leq h_{U}$$
	$$ h_{L} = \min_{0 \leq \alpha \leq 1} \alpha(1-\alpha)(1-2\alpha) $$
	$$ h_{U} = \max_{0 \leq \alpha \leq 1} \alpha(1-\alpha)(1-2\alpha) $$
	To solve for both $h_{L}$ and $h_{U}$, we first differentiate $\alpha(1-\alpha)(1-2\alpha)$ with respect to $\alpha$:
	\begin{align*}
	 &\grad_{\alpha}\left(\alpha(1-\alpha)(1-2\alpha)\right) = \grad_{\alpha}\left(2\alpha^{3} - 3\alpha^{2} + \alpha\right) \\
	 &= \left(6\alpha^{2} - 6\alpha + 1\right)
	 \end{align*}
	Solving for $6\alpha^{2} - 6\alpha + 1 = 0$, we get the solutions:
	$$\alpha = \bigg(\frac{3 + \sqrt{3}}{6}\bigg), \bigg(\frac{3 - \sqrt{3}}{6}\bigg)$$
	Since both $(3 + \sqrt{3}/6), (3 - \sqrt{3}/6)$ lie between 0 and 1, we check for the second derivatives:
	\begin{align*}
	\grad^{2}_{\alpha}\left(\alpha(1-\alpha)(1-2\alpha)\right) = \grad_{\alpha}\left(6\alpha^{2} - 6\alpha + 1\right) \\
	= 12\alpha-6 = 6(2\alpha-1)
	\end{align*}
	At $\alpha = (3 + \sqrt{3})/6$, $ \grad^{2}_{\alpha} = 6(2\alpha-1) = 2\sqrt{3} > 0$.\\
	At $\alpha = (3 - \sqrt{3})/6$, $ \grad^{2}_{\alpha} = 6(2\alpha-1) = -2\sqrt{3} < 0$.\\
	Thus $\alpha = (3 + \sqrt{3})/6$ is a local minima, $\alpha = (3 - \sqrt{3})/6$ is a local maxima.\\
	Substituting the two critical points into $\alpha(1-\alpha)(1-2\alpha)$, we get $h_{U} = 9.623 \times 10^{-2}$, $h_{L} = -9.623 \times 10^{-2}$.\\
	Thus, $h_{U}=9.623 \times 10^{-2},\ h_{L}=-9.623 \times 10^{-2}$ (for use in Theorem \ref{thm:single_layer_p_n_theorem}) and $g=0.25,\ h=0.09623$ (for use in Theorem \ref{thm:L_layer_p_n_theorem}). 
	
	\subsection{Tanh activation}
	For tanh activation, we have the following:
	\begin{align*}
	&\sigma(x) = \tanh(x) = \frac{\exp(x) - \exp(-x)}{\exp(x) + \exp(-x)}\\
	&\sigma^{'}(x) = \left(1 - \tanh(x)\right)\left(1 + \tanh(x)\right)\\
	&\sigma^{''}(x) = -2\tanh(x)\left(1 - \tanh(x)\right)\left(1 + \tanh(x)\right)
	\end{align*}
	The second derivative of tanh , i.e $(\sigma^{''}(x))$ can be bounded using standard differentiation. Let $\alpha$ denote $\tanh(x)$. We know that $-1\leq \alpha\leq 1$:
	$$ h_{L} \leq \sigma^{''}(x) \leq h_{U}$$
	$$ h_{L} = \min_{0 \leq \alpha \leq 1} -2\alpha(1-\alpha)(1+\alpha) $$
	$$ h_{U} = \max_{0 \leq \alpha \leq 1} -2\alpha(1-\alpha)(1+\alpha) $$
	To solve for both $h_{L}$ and $h_{U}$, we first differentiate $-2\alpha(1-\alpha)(1+\alpha)$ with respect to $\alpha$:
	$$ \grad_{\alpha}\left(-2\alpha(1-\alpha)(1+\alpha)\right) = \grad_{\alpha}\left(2\alpha^{3} - 2\alpha\right) = \left(6\alpha^{2} - 2\right)$$
	Solving for $6\alpha^{2} - 2 = 0$, we get the solutions:
	$$\alpha = -\frac{1}{\sqrt{3}}, \frac{1}{\sqrt{3}}$$
	Since both $-1/\sqrt{3}, 1/\sqrt{3}$ lie between -1 and 1, we check for the second derivatives:
	$$ \grad^{2}_{\alpha}\left(-2\alpha(1-\alpha)(1+\alpha)\right) = \grad_{\alpha}\left(6\alpha^{2} - 2\right) = 12\alpha$$
	At $\alpha = -1/\sqrt{3}$, $ \grad^{2}_{\alpha} = 12\alpha = -4\sqrt{3} < 0$.\\
	At $\alpha = 1/\sqrt{3}$, $ \grad^{2}_{\alpha} = 12\alpha = 4\sqrt{3} > 0$.\\
	Thus $\alpha = 1/\sqrt{3}$ is a local minima, $\alpha = -1/\sqrt{3}$ is a local maxima.\\
	Substituting the two critical points into $-2\alpha(1-\alpha)(1+\alpha)$, we get $h_{U} = 0.76981$, $h_{L} = -0.76981$.\\
	Thus, $h_{U}=0.76981,\ h_{L}=-0.76981$ (for use in Theorem \ref{thm:single_layer_p_n_theorem}) and $g=1,\ h=0.76981$ (for use in Theorem \ref{thm:L_layer_p_n_theorem}). 
	
	\section{Quadratic bounds for two-layer ReLU networks} \label{relu_quad_bound}
	For a 2 layer network with ReLU activation, such that the input $\bx$ lies in the ball $\left\|\bx - \bx^{(0)}\right\| \leq \rho$, we can compute the bounds over $\bz^{(1)}$ directly:
	\begin{align*}
	&\bWW^{(1)}_{i}\bx^{(0)} + \bb^{(1)}_{i} - \rho\left\|\bWW^{(1)}_{i}\right\| \leq \bz^{(1)}_{i} \\
	&\bz^{(1)}_{i} \leq \bWW^{(1)}_{i}\bx^{(0)} + \bb^{(1)}_{i} + \rho\left\|\bWW^{(1)}_{i}\right\| 
	\end{align*}
	Thus we can get a lower bound and upper bound for each $\bz^{(1)}_{i}$. We define $d_{i}$ and $u_{i}$ as the following:
	\begin{align}
	&d_{i} = \bWW^{(1)}_{i}\bx^{(0)} + \bb^{(1)}_{i} - \rho\left\|\bWW^{(1)}_{i}\right\| \label{l_bound_relu}\\ 
	&u_{i} = \bWW^{(1)}_{i}\bx^{(0)} + \bb^{(1)}_{i} + \rho\left\|\bWW^{(1)}_{i}\right\| \label{u_bound_relu}
	\end{align}
	We can derive the following quadratic lower and upper bounds for each $\ba^{(1)}_{i}$:
	\begin{align*}
	&\ba^{(1)}_{i} \leq 
	\begin{dcases}
	\frac{-d_{i}}{(u_{i}-d_{i})^2}\left(\bz^{(1)}_{i}\right)^2 + \frac{u_{i}^2 + d_{i}^2}{(u_{i}-d_{i})^2}\bz^{(1)}_{i} - \frac{u_{i}^2d_{i}}{(u_{i}-d_{i})^2} \\ \qquad \qquad \qquad \qquad \qquad \qquad \qquad \text{ if } |d_{i}| \leq |u_{i}|\\
	\frac{u_{i}}{(u_{i}-d_{i})^2}\left(\bz^{(1)}_{i}\right)^2 - \frac{2u_{i}d_{i}}{(u_{i}-d_{i})^2}\bz^{(1)}_{i} + \frac{u_{i}d_{i}^2}{(u_{i}-d_{i})^2} \\ \qquad \qquad \qquad \qquad \qquad \qquad \qquad \text{if } |d_{i}| \geq |u_{i}|
	\end{dcases}\\
	&\ba^{(1)}_{i} \geq 
	\begin{dcases}
	0 \qquad &2|d_{i}| \leq |u_{i}|\\
	\bz^{(1)}_{i} \qquad &|d_{i}| \geq 2|u_{i}|\\
	\frac{1}{u_{i} - d_{i}}\left(\bz^{(1)}_{i}\right)^{2} - \frac{d_{i}}{u_{i}-d_{i}}\bz^{(1)}_{i} \qquad & \text{otherwise}
	\end{dcases}
	\end{align*}
	The above steps are exactly the same as the quadratic upper and lower bounds used in \cite{NIPS2018_7742}.\\
	Using the above two inequalities and the identity: 
	$$\bz^{(2)}_{y}  - \bz^{(2)}_{t} = \sum_{i=1}^{N_{1}}\left(\bWW^{(2)}_{y,i} - \bWW^{(2)}_{t,i}\right) \ba^{(1)}_{i} $$ 
	we can compute a quadratic lower bound for $\bz^{(2)}_{y}  - \bz^{(2)}_{t}$ in terms of $\bz^{(1)}_{i}$ by taking the lower bound for $\ba^{(1)}_{i}$ when
	$\left(\bWW^{(2)}_{y,i} - \bWW^{(2)}_{t,i}\right) > 0$ and upper bound when $\left(\bWW^{(2)}_{y,i} - \bWW^{(2)}_{t,i}\right) <= 0$. Furthermore since $\bz^{(1)}_{i} = \bWW^{(1)}_{i}\bx + \bb^{(1)}_{i}$, we can express the resulting quadratic in terms of $\bx$. Thus, we get the following quadratic function :
	$$\bz^{(2)}_{y}  - \bz^{(2)}_{t} \geq \frac{1}{2}\bx^{T}\bP\bx + \bq + r $$ 
	The coefficients $\bP$, $\bq$ and $r$ can be determined using the above procedure. Note that unlike in \cite{NIPS2018_7742}, RHS can be a non-convex function.\\ 
	Thus, it becomes an optimization problem where the goal is to minimize the distance $1/2\left\|\bx - \bx^{(0)}\right\|^{2}$ subject to RHS (which is quadratic in $\bx$) being zero. That is both our objective and constraint are quadratic functions. In the optimization literature, this is called the S-procedure and is one of the few non-convex problems that can be solved efficiently \cite{Boyd:2004:CO:993483}.\\
	We start with two initial values called $\rho_{low}$ (initialized to 0) and $\rho_{high}$ (initialized to 5).\\
	We start with an initial value of $\rho$, initialized at $1/2\left(\rho_{low} + \rho_{high}\right)$ to compute $d_{i}$ (eq. \eqref{l_bound_relu}) and $u_{i}$ (eq. \eqref{u_bound_relu}). If the final distance after solving the S-procedure is less than $\rho$, we set $\rho_{low} = \rho$. if the final distance is greater than $\rho$, we set $\rho_{high}=\rho$. Set new $\rho=1/2\left(\rho_{low} + \rho_{high}\right)$. Repeat until convergence.
	
	\section{Additional experiments} 
	\label{additional_experiments}
	Empirical accuracy means the fraction of test samples that were correctly classified after running a PGD attack \cite{madry2018towards} with an $l_{2}$ bound on the adversarial perturbations. Certified accuracy means the fraction of test samples that were classified correctly initially and had the robustness certificate greater than a pre-specified attack radius $\rho$. Unless otherwise specified, for both empirical and certified accuracy, we use $\rho=0.5$. Unless otherwise specified, we use the class with the second largest logit as the attack target for the given input (i.e. the class $t$). Unless specified, the experiments were run on the MNIST dataset while noting that our results are scalable for more complex datasets. The notation ($L \times [1024]$, activation) denotes a neural network with $L$ layers with the specified activation function, ($\gamma=c$) denotes standard training with $\gamma$ set to $c$, (CRT, $c$) denotes CRT training with $\gamma=c$.  Certificates CROWN and CRC are computed over 150 correctly classified images.
	
	\subsection{Computing $K_{lb}$ and $K_{ub}$}\label{subsec:K_emp_the}
	
	First, note that $K$ does not depend on the input, but on network weights $\bWW^{(I)}$, label $y$ and target $t$. Different images may still have different $K$ because label $y$ and target $t$ may be different. 
	
	To compute $K_{lb}$ in the table, first for each pair $y$ and $t$, we find the largest eigenvalue of the Hessian of all test images that have label $y$ and second largest logit of class $t$. Then we take the max of the largest eigenvalue across all test images. This gives a rough estimate of the largest curvature in the vicinity of test images with label $y$ and target $t$. We can directly take the mean across all such pairs to compute $K_{lb}$. However, we find that some pairs $y$ and $t$ were infrequent (with barely 1,2 test images in them). Thus, for all such pairs we cannot get a good estimate of the largest curvature in vicinity. We select all pairs $y$ and $t$ that have at least 100 images in them and compute $K_{lb}$ by taking the mean across all such pairs.
	
	To compute $K_{ub}$ in the table, we compute $K$ for all pairs $y$ and $t$ that have at least 100 images, i.e at least 100 images should have label $y$ and target $t$. And then we compute the mean across all $K$ that satisfy this condition. This was done to do a fair comparison with $K_{lb}$.
	Figure \ref{fig:gamma_effect_4} shows a plot of the $K_{ub}$ and $K_{lb}$ with increasing $\gamma$ for a sigmoid network (with 4 layers).
	\begin{figure}[H]
		\centering
		\includegraphics[width=0.45\textwidth]{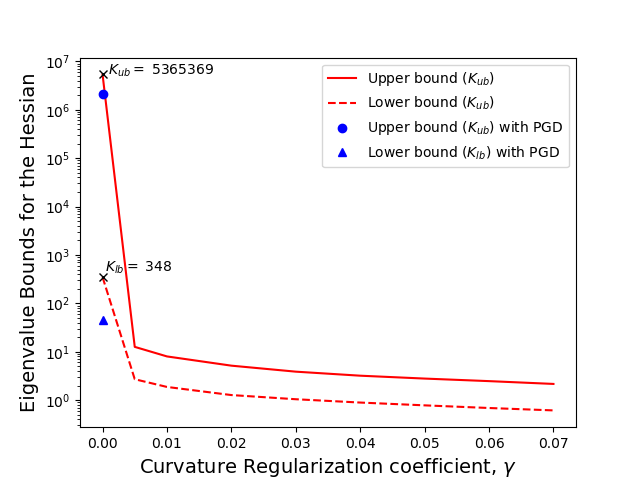}
		\caption{Effect of $\gamma$ on $K_{ub}$ and $K_{lb}$ for a 4 layer network. We observe a similar trend as in 2 and 3 layer networks (Figure \ref{fig:gamma_effect_2_3}). At $\gamma=0$, we observe $K_{ub} \approx 15418 \times K_{lb}$.}
		\label{fig:gamma_effect_4}
	\end{figure}
    \subsection{Comparison with provable defenses}
    In this section, we compare Curvature-based Robust Training (Ours) against state-of-the-art interval-bound propagation based adversarial training methods: COAP i.e Convex Outer Adversarial Polytope \cite{Wong2017ProvableDA} and CROWN-IBP \cite{ZhangCROWNIBP} with different attack radius on MNIST and Fashion-MNIST datasets. For CROWN-IBP, we vary the final\_beta parameter between 0.5 to 3 (using an interval of 0.1) and choose the model with best certified accuracy. 

	\begin{table}[ht!]
		\centering
		\renewcommand{\arraystretch}{1.15}
 		\caption{Comparison with interval-bound propagation based adversarial training methods with attack radius $\rho=0.5$ on MNIST dataset. Note that the certified accuracy of softplus network with CROWN-IBP is significantly less than that of a similar ReLU network.}
		\begin{tabular}{ | l | l | l | l | } 
			\hline
			\multicolumn{1}{|c|}{\multirow{3}{*}{Network}} & \multicolumn{1}{c|}{\multirow{3}{*}{Training}} & 
			\multicolumn{1}{c|}{\multirow{3}{4em}{Standard Accuracy}} & 
			\multicolumn{1}{c|}{\multirow{3}{4em}{Certified Accuracy}} \\ 
			& & & \\ 
			& & & \\ 
			\hline
			\multirow{2}{4em}{$2\times[1024]$, softplus} & \textbf{CRT, 0.01} & \textbf{98.69\%} & \textbf{95.5\%} \\ 
			\cline{2-4}
			& CROWN-IBP & 98.72\% & 89.31\% \\ 
			\hline 
			\multirow{2}{4em}{$2\times[1024]$, relu} & CROWN-IBP & 98.69\% & 91.38\% \\ 
			\cline{2-4}
			& COAP & 98.8\% & 90.2\% \\ 
			\hline 
			\hline
			\multirow{2}{4em}{$3\times[1024]$, softplus} & \textbf{CRT, 0.01} & \textbf{98.56\%} & \textbf{94.44\%} \\ 
			\cline{2-4}
			& CROWN-IBP & 98.55\% & 88.67\% \\
			\hline 
			\multirow{2}{4em}{$3\times[1024]$, relu} & CROWN-IBP & 98.9\% & 90.67\% \\ 
			\cline{2-4}
			& COAP & 98.9\% & 89.0\%\\ 
			\hline 
            \hline
			\multirow{2}{4em}{$4\times[1024]$, softplus} & \textbf{CRT, 0.01} & \textbf{98.43\%} & \textbf{93.35\%} \\ 
			\cline{2-4}
			& CROWN-IBP & 98.34\% & 87.41\% \\ 
			\hline
			\multirow{2}{4em}{$4\times[1024]$, relu} & CROWN-IBP & 98.78\% & 90.45\% \\ 
			\cline{2-4}
			& COAP & 98.9\% & 89.0\% \\ 
			\hline 
		\end{tabular}
		\label{table:empirical_provable_adversarial_0.5}
	\end{table}

	\begin{table}[h!]
		\centering
		\renewcommand{\arraystretch}{1.15}
		\caption{Comparison with interval-bound propagation based adversarial training methods with attack radius $\rho=0.5$  on Fashion-MNIST dataset. }
		\begin{tabular}{ | l | l | l | l | } 
			\hline
			\multicolumn{1}{|c|}{\multirow{3}{*}{Network}} & \multicolumn{1}{c|}{\multirow{3}{4em}{Training}} &		\multicolumn{1}{c|}{\multirow{3}{4em}{Standard Accuracy}} & 
			\multicolumn{1}{c|}{\multirow{3}{4em}{Certified Robust Accuracy}} \\ 
			& & & \\ 
			& & & \\ 
			\hline
			\multirow{2}{4em}{$2\times[1024]$, softplus} & \multirow{2}{4em}{CRT, 0.01} & \multirow{2}{4em}{88.45\%} & \multirow{2}{4em}{\textbf{78.45\%}} \\ 
			& & & \\ 
			\hline 
			\multirow{2}{4em}{$2\times[1024]$, relu} & COAP & 86.0\% & 74.0\%\\
			\cline{2-4}
			& CROWN-IBP & 85.89\% & 74.62\% \\ 
			\hline
			\hline 
			\multirow{2}{4em}{$3\times[1024]$, softplus} & \multirow{2}{4em}{CRT, 0.01} & \multirow{2}{4em}{86.21\%} & \multirow{2}{4em}{\textbf{76.94\%}} \\ 
			& & & \\
			\hline
			\multirow{2}{4em}{$3\times[1024]$, relu} & COAP & 85.9\% & 74.3\%\\		
			\cline{2-4}
			& CROWN-IBP & 86.27\% & 74.56\% \\ 
			\hline 
			\hline
			\multirow{2}{4em}{$4\times[1024]$, softplus} & \multirow{2}{4em}{CRT, 0.01} & \multirow{2}{4em}{86.37\%} & \multirow{2}{4.5em}{\textbf{75.02\%}} \\ 
			& & & \\
			\hline 
			\multirow{2}{4em}{$4\times[1024]$, relu} &  COAP & 85.9\% & 74.2\% \\ 
			\cline{2-4}
			& CROWN-IBP & 86.03\% & 74.38\% \\ 
			\hline 
		\end{tabular}
		\label{table:empirical_provable_adversarial_0.5_fashion_mnist}
	\end{table}

	\begin{table}[ht!]
		\centering
		\renewcommand{\arraystretch}{1.15}
		\caption{Comparison with interval-bound propagation based adversarial training methods with attack radius $\rho=1.58$ on MNIST dataset. We again observe that the certified accuracy of softplus network with CROWN-IBP is significantly less than that of a similar ReLU network.} 
		\begin{tabular}{ | l | l | l | l | } 
			\hline
			\multicolumn{1}{|c|}{\multirow{3}{*}{Network}} & \multicolumn{1}{c|}{\multirow{3}{*}{Training}} & \multicolumn{1}{c|}{\multirow{3}{3.5em}{Standard Accuracy}} & 
			\multicolumn{1}{c|}{\multirow{3}{3.5em}{Certified Robust Accuracy}} \\ 
			& & & \\ 
			& & & \\ 
			\hline
			\multirow{2}{4em}{$2\times[1024]$, softplus} & \textbf{CRT, 0.01} & \textbf{98.68\%} & \textbf{69.79\%} \\
			\cline{2-4}
			& CROWN-IBP & 88.48\% & 42.36\% \\ 
			\hline 
			\multirow{2}{4em}{$2\times[1024]$, relu} & COAP & 89.33\% & 44.29\% \\ 
			\cline{2-4}
			& CROWN-IBP & 89.49\% & 44.96\% \\ 
			\hline 
			\hline
			\multirow{4}{4em}{$3\times[1024]$, softplus} & CRT, 0.01 & 98.26\% & 14.21\% \\ 
			\cline{2-4}
			& CRT, $0.03$ & 97.82\% & 50.72\% \\ 
			\cline{2-4}
			& \textbf{CRT, 0.05} & \textbf{97.43\%} & \textbf{57.78\%} \\ 
            \cline{2-4}
            & CROWN-IBP & 86.58\% & 42.14\% \\ 
			\hline 
			\multirow{2}{4em}{$3\times[1024]$, relu} & COAP  & 89.12\% & 44.21\% \\ 
			\cline{2-4}
			& CROWN-IBP & 87.77\% & 44.74\% \\ 
			\hline 
			\hline
			\multirow{5}{4em}{$4\times[1024]$, softplus} & CRT, $0.01$ & 97.80\% & 6.25\% \\ 
			\cline{2-4}
			& CRT, $0.03$ & 97.09\% & 29.64\% \\ 
			\cline{2-4}
			& CRT, $0.05$ & 96.33\% & 44.44\% \\ 
			\cline{2-4}
			& \textbf{CRT, 0.07} & \textbf{95.60\%} & \textbf{53.19\%} \\ 
            \cline{2-4}
			& CROWN-IBP & 82.74\% & 41.34\% \\ 
			\hline 
			\multirow{2}{4em}{$4\times[1024]$, relu} & COAP & 90.17\% & 44.66\% \\ 
			\cline{2-4}
			& CROWN-IBP & 84.4\% & 43.83\% \\ 
			\hline 
		\end{tabular}
		\label{table:empirical_provable_adversarial_1.58}
	\end{table}
    \FloatBarrier
	\subsection{Comparing Randomized Smoothing with CRT}
	\begin{table}[h!]
		\centering
		\renewcommand{\arraystretch}{1.15} 
		\caption{Comparison between CRT and Randomized Smoothing\cite{Cohen2019CertifiedAR}. 
			$s$ denotes the standard deviation for smoothing. We use $\rho=0.5$. For CRT, we use $\gamma=0.01$}
		\begin{tabular}{ | l | l |  l | l | l | } 
			\hline
			\multicolumn{1}{|c|}{\multirow{2}{*}{Network}} & \multicolumn{3}{c|}{Randomized Smoothing} & \multicolumn{1}{c|}{CRT}\\ 
			\cline{2-5}
			& $s=0.25$ & $s=0.50$ & $s=1.0$ & \_ \\ 
			\hline
			\multirow{2}{4em}{$2\times[1024]$, sigmoid} 
			& \multirow{2}{*}{93.75\%} & \multirow{2}{*}{93.09\%} & \multirow{2}{*}{88.91\%} & \textbf{95.61\%} \\ 
			& & & & \\
			\hline
			\multirow{2}{4em}{$2\times[1024]$, tanh} & \multirow{2}{*}{94.61\%} & \multirow{2}{*}{93.08\%} & \multirow{2}{*}{82.26\%} & \multirow{2}{*}{\textbf{95.00\%}}\\ 
			& & & & \\
			\hline 
			\multirow{2}{4em}{$3\times[1024]$, sigmoid} 
			& \multirow{2}{*}{94.00\%} & \multirow{2}{*}{93.03\%} & \multirow{2}{*}{86.58\%} & \multirow{2}{*}{\textbf{94.99\%}} \\ 
            & & & & \\
			\hline
			\multirow{2}{4em}{$3\times[1024]$, tanh} & \multirow{2}{*}{93.69\%} & \multirow{2}{*}{91.68\%} & \multirow{2}{*}{80.55\%} & \multirow{2}{*}{\textbf{94.16\%}}\\ 
			& & & & \\
			\hline 
			\multirow{2}{4em}{$4\times[1024]$, sigmoid} & \multirow{2}{*}{\textbf{93.68\%}} & \multirow{2}{*}{92.45\%} & \multirow{2}{*}{84.99\%} & \multirow{2}{*}{93.41\%} \\ 
			& & & & \\
			\hline
			\multirow{2}{4em}{$4\times[1024]$, tanh} & \multirow{2}{*}{\textbf{93.57\%}} & \multirow{2}{*}{92.19\%} & \multirow{2}{*}{83.90\%} & \multirow{2}{*}{91.37\%}\\ 
			& & & & \\
			\hline 
		\end{tabular}
		\label{table:smoothing_adversarial}
	\end{table}	
	Since, randomized smoothing is designed to work in untargeted attack settings while CRT is for targeted attacks, we make the following changes in randomized smoothing. First, we use $n_{0}=100$ initial samples to select the label class ($l$) and false target class ($t$). The samples for estimation were $n=100,000$ and failure probability was $\alpha=0.001$. Then we use the binary version of randomized smoothing for estimation, i.e classify between $y$ and $t$. To find the adversarial example for adversarial training, we use the cross entropy loss for $2$ classes ($y$ and $t$).
    \subsection{Additional experiments}

	\begin{table}[h!]
		\centering
		\renewcommand{\arraystretch}{1.15} 
		\caption{Table showing success rates $(primal=dual)$ for different values of $\gamma$. Certificate success rate denotes the fraction of points ($\bx^{(0)})$ satisfying $\bz_{y}-\bz_{t}=0$, Attack success rate denotes the fraction of points ($\bx^{(0)})$ satisfying $\|\bx^{(attack)} - \bx^{(0)}\|_{2} = \rho$ implying $primal = dual$ in Theorems \ref{thm:certificate} and \ref{thm:attack} respectively. We observe that as we increase $\gamma$, the fraction of points satisfying $primal=dual$ increases for both the certificate and attack problems. This can be attributed to the curvature bound $K(\bWW, y, t)$ becoming tight on increasing $\gamma$.}
		\begin{tabular}{ | l | l |  l | l | l | } 
			\hline
			\multicolumn{1}{|c|}{\multirow{3}{*}{Network}} & \multicolumn{1}{c|}{\multirow{3}{*}{$\gamma$}} & \multicolumn{1}{c|}{\multirow{3}{*}{Accuracy}} & \multirow{3}{4em}{Attack success rate} & \multirow{3}{4em}{Certificate success rate} \\ 
			& & & & \\
			& & & & \\
			\hline
			\multirow{4}{4em}{$2\times[1024]$, sigmoid} & 0. & 98.77\%  & 5.05\% & 2.24\% \\ 
			\cline{2-5}
			& 0.01 & 98.57\% & 100\% & 15.68\%\\ 
			\cline{2-5}
			& 0.02 & 98.59\% & 100\% & 31.56\%\\ 
			\cline{2-5}
			& 0.03 & 98.30\% & 100\% & 44.17\%\\ 
			\hline
			\multirow{4}{4em}{$3\times[1024]$, sigmoid} & 0. & 98.52\% & 0.\% & 0.12\% \\ 
			\cline{2-5}
			& 0.01 & 98.23\% & 44.86\% & 3.34\%\\ 
			\cline{2-5}
			& 0.03 & 97.86\% & 100\%  & 11.51\% \\ 
			\cline{2-5}
			& 0.05 & 97.60\% & 100\% & 22.59\%\\ 
			\hline
			\multirow{6}{4em}{$4\times[1024]$, sigmoid} & 0. & 98.22\% & 0.\% & 0.01\%\\ 
			\cline{2-5}
			& 0.01 & 97.24\% & 24.42\% & 2.68\%\\ 
			\cline{2-5}
			& 0.03 & 96.27\% & 44.42\% & 6.45\%\\ 
			\cline{2-5}
			& 0.05 & 95.77\% & 99.97\% & 12.40\%\\ 
			\cline{2-5}
			& 0.06 & 95.52\% & 100\% & 15.87\%\\ 
			\cline{2-5}
			& 0.07 & 95.24\% & 100\% & 19.53\% \\ 
			\hline
		\end{tabular}
		\label{table:primal_dual_eq_fraction_full}
	\end{table}	

	\begin{table*}[ht!]
		\centering
		\renewcommand{\arraystretch}{1.15} 
		\caption{Results for CIFAR-10 dataset (only curvature regularization, no CRT training)}
		\begin{tabular}{ | l | l | l | l | l | l | l | l | } 
			\hline
			\multicolumn{1}{|c|}{\multirow{3}{*}{Network}} & \multicolumn{1}{c|}{\multirow{3}{*}{Training}} & \multicolumn{1}{c|}{\multirow{3}{3.5em}{Standard Accuracy}} & 
			\multicolumn{1}{c|}{\multirow{3}{3.5em}{Empirical Robust Accuracy}} &
			\multicolumn{1}{c|}{\multirow{3}{3.5em}{Certified Robust Accuracy}} & \multicolumn{2}{c|}{\multirow{2}{4em}{Certificate (mean)}} \\ 
			& & & \multicolumn{1}{c|}{} & \multicolumn{1}{c|}{} & \multicolumn{1}{c}{} & \multicolumn{1}{c|}{}\\
			\cline{6-7}
			& & & & & CROWN & CRC \\
			\hline
			\multirow{2}{*}{$2\times[1024]$, sigmoid}
			& standard & 46.23\% & 37.82\% & 14.10\% & 0.37219 & \textbf{0.38173} \\ 
			\cline{2-7}
			& $\gamma = 0.01$ & 45.42\% & 38.17\% & 26.50\% & 0.40540 & \textbf{0.55010}  \\ 
			\hline 
			\multirow{2}{*}{$3\times[1024]$, sigmoid}
			& standard & 48.57\% & 34.80\% & 0.00\% & 0.19127 & 0.01404 \\
			\cline{2-7} 
			& $\gamma = 0.01$ & 50.31\% & 39.87\% & 18.28\% & 0.24778 & \textbf{0.37895 }\\ 
			\hline 
			\multirow{2}{*}{$4\times[1024]$, sigmoid} &
			standard & 46.04\% & 34.38\% & 0.00\% & 0.19340 & 0.00191  \\ 
			\cline{2-7}
			& $\gamma = 0.01$ & 48.28\% & 40.10\% & 21.07\% & 0.29654 & \textbf{0.40005}  \\ 
			\hline 
		\end{tabular}
		\label{table:cifar_results}
	\end{table*}

	\begin{table*}[b!]
		\centering
		\renewcommand{\arraystretch}{1.15}
		\caption{Comparison between CRT, PGD \cite{madry2018towards} and TRADES \cite{Zhang2019TheoreticallyPT} for sigmoid and tanh networks. CRC outperforms CROWN significantly for 2 layer networks and when trained with our regularizer for deeper networks. CRT outperforms TRADES and PGD giving higher certified accuracy. }
		\begin{tabular}{ | l | l | l | l | l | l | l | } 
			\hline
			\multicolumn{1}{|c|}{\multirow{3}{*}{Network}} & \multicolumn{1}{c|}{\multirow{3}{*}{Training}} & \multicolumn{1}{c|}{\multirow{3}{3.5em}{Standard Accuracy}} & 
			\multicolumn{1}{c|}{\multirow{3}{3.5em}{Empirical Robust Accuracy}} &
			\multicolumn{1}{c|}{\multirow{3}{3.5em}{Certified Robust Accuracy}} & \multicolumn{2}{c|}{\multirow{2}{4em}{Certificate (mean)}} \\ 
			& & & \multicolumn{1}{c|}{} & \multicolumn{1}{c|}{} & \multicolumn{1}{c}{} & \multicolumn{1}{c|}{}\\
			\cline{6-7}
			& & & & & CROWN & CRC \\
						\hline
						\multirow{3}{4em}{$2\times[1024]$, sigmoid} & PGD & 98.80\% & 96.26\% & 93.37\% & 0.37595 & 0.82702 \\ 
						\cline{2-7} 
						& TRADES & 98.87\% & 96.76\% & 95.13\% & 0.41358 & 0.92300 \\ 
						\cline{2-7} 
						& CRT, $0.01$ & 98.57\% & 96.28\% & \textbf{95.59\%} & 0.43061 & \textbf{1.54673} \\ 
			\hline 
			\multirow{3}{4em}{$2\times[1024]$, tanh} & PGD & 98.76\% & 95.79\% & 84.11\% & 0.30833 & 0.61340 \\ 
			\cline{2-7} 
			& TRADES & 98.63\% & 96.20\% & 93.72\% & 0.40601 & 0.86287 \\ 
			\cline{2-7} 
			& CRT, $0.01$ & 98.52\% & 95.90\% & \textbf{95.00\%} & 0.37691 & \textbf{1.47016} \\ 
						\hline 
						\multirow{3}{4em}{$3\times[1024]$, sigmoid} & PGD & 98.84\% & 96.14\% & 0.00\% & 0.29632 & 0.07290 \\ 
						\cline{2-7} 
						& TRADES & 98.95\% & 96.79\% & 0.00\% & 0.30576 & 0.09108 \\ 
						\cline{2-7} 
						& CRT, $0.01$ & 98.23\% & 95.70\% & \textbf{94.99\%} & 0.39603 & \textbf{1.24100} \\ 
			\hline 
			\multirow{3}{4em}{$3\times[1024]$, tanh} & PGD & 98.78\% & 94.92\% & 0.00\% & 0.12706 & 0.03036 \\ 
			\cline{2-7} 
			& TRADES & 98.16\% & 94.78\% & 0.00\% & 0.15875 & 0.02983  \\ 
			\cline{2-7} 
			& CRT, $0.01$ & 98.15\% & 95.00\% & \textbf{94.16\%} & 0.28004 & \textbf{1.14995}  \\ 
						\hline 
						\multirow{3}{4em}{$4\times[1024]$, sigmoid} & PGD & 98.84\% & 96.26\% & 0.00\% & 0.25444 & 0.00658  \\ 
						\cline{2-7} 
						& TRADES & 98.76\% & 96.67\% & 0.00\% & 0.26128 & 0.00625 \\ 
						\cline{2-7} 
						& CRT, $0.01$ & 97.83\% & 94.65\% & \textbf{93.41\%} & 0.40327 & \textbf{1.06208}  \\ 
			\hline 
			\multirow{3}{4em}{$4\times[1024]$, tanh} & PGD & 98.53\% & 94.53\% & 0.00\% & 0.07439 & 0.00140\\ 
			\cline{2-7} 
			& TRADES & 97.08\% & 92.85\% & 0.00\% & 0.11889 & 0.00068 \\ 
			\cline{2-7} 
			& CRT, $0.01$ & 97.24\% & 93.05\% & \textbf{91.37\%} & 0.33649 & \textbf{0.93890} \\ 
			\hline 
		\end{tabular}
		\label{table:empirical_adversarial_appendix}
	\end{table*}
	
	\begin{table*}[h!]
		\centering
		\renewcommand{\arraystretch}{1.15} 
		\caption{Comparison between CRC and CROWN-general (CROWN-Ada for relu) for different targets. For CRT training, we use $\gamma=0.01$. We compare CRC with CROWN-general for different targets for 150 correctly classified images. Runner-up means class with second highest logit is considered as adversarial class. Random means any random class other than the label is considered adversarial. Least means class with smallest logit is adversarial. For 2-layer networks, CRC outperforms CROWN-general significantly even without adversarial training. For deeper networks (3 and 4 layers), CRC works better on networks that are trained with curvature regularization. Both CROWN and CRC are computed on CPU but the running time numbers mentioned here are not directly comparable because our CRC implementation uses a batch of images while the CROWN implementation uses a single image at a time.}
		\begin{tabular}{ | l | l | l | l | l | l | l | l | }
			\hline
			\multirow{2}{*}{Network} & \multirow{2}{*}{Training} & \multirow{2}{*}{Target} & \multicolumn{2}{c|}{Certificate (mean)} & \multicolumn{2}{c|}{Time per Image (s)} \\ 
			\cline{4-7}
			&  & & CROWN & CRC & CROWN & CRC \\ 
			\hline
			\multirow{3}{*}{$2\times[1024]$, relu} & \multirow{3}{*}{standard} & runner-up & 0.50110 & \textbf{0.59166} & 0.1359 & 2.3492 \\ 
			\cline{3-7} 
			& & random & 0.68506 & \textbf{0.83080} & 0.2213 & 3.5942 \\ 
			\cline{3-7} 
			& & least & 0.86386 & \textbf{1.04883} & 0.1904 & 3.0292 \\ 
			\hline
			\multirow{6}{*}{$2\times[1024]$, sigmoid} & \multirow{3}{*}{standard} & runner-up & 0.28395 & \textbf{0.48500} & 0.1818 & 0.1911 \\ 
			\cline{3-7} 
			& & random & 0.38501 & \textbf{0.69087} & 0.1870 & 0.1912 \\ 
			\cline{3-7} 
			& & least & 0.47639 & \textbf{0.85526} & 0.1857 & 0.1920 \\ 
			\cline{2-7} 
			& \multirow{3}{4.1em}{CRT,\ $0.01$} & runner-up & 0.43061 & \textbf{1.54673} & 0.1823 & 0.1910 \\ 
			\cline{3-7} 
			& & random & 0.52847 & \textbf{1.99918} & 0.1853 & 0.1911 \\ 
			\cline{3-7} 
			& & least & 0.62319 & \textbf{2.41047} & 0.1873 & 0.1911 \\ 
			\hline 
			\multirow{6}{*}{$2\times[1024]$, tanh} & \multirow{3}{*}{standard} & runner-up & 0.23928 & \textbf{0.40047} & 0.1672 & 0.1973 \\ 
			\cline{3-7} 
			& & random & 0.31281 & \textbf{0.52025} & 0.1680 & 0.1986 \\ 
			\cline{3-7} 
			& & least & 0.38964 & \textbf{0.63081} & 0.1726 & 0.1993 \\ 
			\cline{2-7} 
			& \multirow{3}{4.1em}{CRT, $0.01$} & runner-up & 0.37691 & \textbf{1.47016} & 0.1633 & 0.1963 \\ 
			\cline{3-7} 
			& & random & 0.45896 & \textbf{1.87571} & 0.1657 & 0.1982 \\ 
			\cline{3-7} 
			& & least & 0.52800 & \textbf{2.21704} & 0.1697 & 0.1981
			\\ 
			\hline
			\multirow{6}{*}{$3\times[1024]$, sigmoid} & \multirow{3}{*}{standard} & runner-up & \textbf{0.24644} & 0.06874 & 1.6356 & 0.5012 \\ 
			\cline{3-7} 
			& & random & \textbf{0.29496} & 0.08275 & 1.5871 & 0.5090 \\ 
			\cline{3-7} 
			& & least & \textbf{0.33436} & 0.09771 & 1.6415 & 0.5056 \\ 
			\cline{2-7} 
			& \multirow{3}{4.1em}{CRT, $0.01$} & runner-up & 0.39603 & \textbf{1.24100} & 1.5625 & 0.5013 \\ 
			\cline{3-7} 
			& & random & 0.46808 & \textbf{1.54622} & 1.6142 & 0.4974\\ 
			\cline{3-7} 
			& & least & 0.51906 & \textbf{1.75916} & 1.6054 & 0.4967 \\ 
			\hline 
			\multirow{6}{*}{$3\times[1024]$, tanh} & \multirow{3}{*}{standard} & runner-up & \textbf{0.08174} & 0.01169 & 1.4818 & 0.4908 \\ 
			\cline{3-7} 
			& & random & \textbf{0.10012} & 0.01432 & 1.5906 & 0.4963\\ 
			\cline{3-7} 
			& & least & \textbf{0.12132} & 0.01757 & 1.5888 & 0.5076\\ 
			\cline{2-7} 
			& \multirow{3}{4.1em}{CRT, $0.01$} & runner-up & 0.28004 & \textbf{1.14995} & 1.4832 & 0.4926 \\ 
			\cline{3-7} 
			& & random & 0.32942 & \textbf{1.41032} & 1.5637 & 0.4957 \\ 
			\cline{3-7} 
			& & least & 0.38023 & \textbf{1.65692} & 1.5626 & 0.4930 \\ 
			\hline
			\multirow{6}{*}{$4\times[1024]$, sigmoid} & \multirow{3}{*}{standard} & runner-up & \textbf{0.19501} & 0.00454 & 4.7814 & 0.8107 \\ 
			\cline{3-7} 
			& & random & \textbf{0.21417} & 0.00542 & 4.6313 & 0.8377 \\ 
			\cline{3-7} 
			& & least & \textbf{0.22706} & 0.00609 & 4.7973 & 0.8313 \\ 
			\cline{2-7} 
			& \multirow{3}{4.1em}{CRT, $0.01$} & runner-up & 0.40327 & \textbf{1.06208} & 4.1830 & 0.8088 \\ 
			\cline{3-7} 
			& & random & 0.47038 & \textbf{1.29095} & 4.3922 &  0.7333\\ 
			\cline{3-7} 
			& & least & 0.52249 & \textbf{1.49521} & 4.4676 & 0.7879 \\ 
			\hline
			\multirow{6}{*}{$4\times[1024]$, tanh} &  \multirow{3}{*}{standard} & runner-up & \textbf{0.03554} & 0.00028 & 5.7016 & 0.8836 \\ 
			\cline{3-7} 
			& & random & \textbf{0.04247} & 0.00036 & 5.8379 & 0.8602\\ 
			\cline{3-7} 
			& & least & \textbf{0.04895} & 0.00044 & 5.8298 & 0.9045\\ 
			\cline{2-7} 
			& \multirow{3}{4.1em}{CRT, $0.01$} & runner-up & 0.33649 & \textbf{0.93890} & 3.8815 & 0.8182 \\ 
			\cline{3-7} 
			& & random & 0.41617 & \textbf{1.18956} & 4.0013 & 0.8215 \\ 
			\cline{3-7} 
			& & least & 0.47778 & \textbf{1.41429} & 4.3856 & 0.8311 \\ 
			\hline 
		\end{tabular}
		\label{table:cert_compare_diff_targets}
	\end{table*}

	\begin{table*}[h!]
		\centering
		\renewcommand{\arraystretch}{1.15} 
		\caption{In this table, we measure the effect of increasing $\gamma$, when the network is trained with CRT on standard, empirical, certified robust accuracy, $K_{lb}$  and $K_{ub}$ (defined in subsection \ref{subsec:K_emp_the}) for different depths (2, 3, 4 layer) and activations  (sigmoid, tanh). We find that for all networks $\gamma=0.01$ works best. We find that the lower bound, $K_{lb}$ increases (for $\gamma=0$) for deeper networks suggesting that deep networks have higher curvature. Furthermore, for a given $\gamma$ (say $0.005$), we find that the gap between $K_{ub}$ and $K_{lb}$ increases as we increase the depth suggesting that $K$ is not a tight bound for deeper networks.} 
		\begin{tabular}{ | l | l | l | l | l | l | l |} 
			\hline
			\multirow{3}{*}{Network} & \multirow{3}{*}{$\gamma$} & \multirow{3}{4em}{Standard Accuracy} & \multirow{3}{4em}{Empirical Robust Accuracy} & \multirow{3}{4em}{Certified Robust Accuracy} & \multicolumn{2}{c|}{Curvature bound (mean)}\\ 
			\cline{6-7}
			& &  & & & \multirow{2}{4em}{$K_{lb}$} & \multirow{2}{4em}{$K_{ub}$} \\
			& & & & & & \\
			\hline
			\multirow{5}{4em}{$2\times[1024]$, sigmoid} & 0.0 & 98.77\% & 96.17\% & 95.04\% & 7.2031 & 72.0835 \\ 
			\cline{2-7} 
			& 0.005 & 98.82\% & 96.33\% & \textbf{95.61\%} & 3.8411 & 8.2656 \\ 
			\cline{2-7} 
			& 0.01 & 98.57\% & 96.28\% & \textbf{95.59\%} & 2.8196 & 5.4873 \\ 
			\cline{2-7} 
			& 0.02 & 98.59\% & 95.97\% & 95.22\% & 2.2114 & 3.7228 \\ 
			\cline{2-7} 
			& 0.03 & 98.30\% & 95.73\% & 94.94\% & 1.8501 & 2.9219 \\ 
			\cline{1-7}
			\multirow{5}{4em}{$2\times[1024]$, tanh} & 0.0 & 98.65\% & 95.48\% & 92.69\% & 12.8434 & 107.5689 \\ 
			\cline{2-7} 
			& 0.005 & 98.71\% & 95.88\% & 94.76\% & 4.8116 & 10.1860 \\ 
			\cline{2-7} 
			& 0.01 & 98.52\% & 95.90\% & \textbf{95.00\%} & 3.4269 & 6.3529 \\ 
			\cline{2-7} 
			& 0.02 & 98.35\% & 95.71\% & 94.77\% & 2.3943 & 4.1513 \\ 
			\cline{2-7} 
			& 0.03 & 98.29\% & 95.39\% & 94.54\% & 1.9860 & 3.933 \\ 
			\hline 
			\multirow{7}{4em}{$3\times[1024]$, sigmoid} 
			& 0. & 98.52\% & 90.26\% & 0.00\% & 19.2131 & 3294.9070 \\
			\cline{2-7} 		
			& 0.005 & 98.41\% & 95.81\% & 94.91\% & 2.6249 & 13.4985 \\ 
			\cline{2-7} 
			& 0.01 & 98.23\% & 95.70\% & \textbf{94.99\%} & 1.9902 & 8.6654 \\ 
			\cline{2-7} 
			& 0.02 & 97.99\% & 95.33\% & 94.64\% & 1.4903 & 5.4380 \\ 
			\cline{2-7} 
			& 0.03 & 97.86\% & 94.98\% & 94.15\% & 1.2396 & 4.1409 \\ 
			\cline{2-7} 
			& 0.04 & 97.73\% & 94.60\% & 93.88\% & 1.0886 & 3.3354 \\
			\cline{2-7} 
			& 0.05 & 97.60\% & 94.45\% & 93.65\% & 0.9677 & 2.7839 
			\\ 
			\hline 
			\multirow{7}{4em}{$3\times[1024]$, tanh} 
			& 0. & 98.19\% & 86.38\% & 0.00\% & 133.7992 & 17767.5918 \\ 
			\cline{2-7}
			& 0.005 & 98.13\% & 94.56\% & 93.01\% & 3.2461 & 17.5500 \\ 
			\cline{2-7} 
			& 0.01 & 98.15\% & 95.00\% & \textbf{94.16\%} & 2.2347 & 10.8635 \\ 
			\cline{2-7} 
			& 0.02 & 97.84\% & 94.79\% & 94.05\% & 1.6556 & 6.7072 \\ 
			\cline{2-7} 
			& 0.03 & 97.70\% & 94.19\% & 93.42\% & 1.3546 & 5.0533 \\ 
			\cline{2-7} 
			& 0.04 & 97.57\% & 94.04\% & 92.95\% & 1.1621  & 4.0071 \\
			\cline{2-7} 
			& 0.05 & 97.31\% & 93.66\% & 92.65\% & 1.0354 & 3.3439 
			\\ 
			\hline 
			\multirow{7}{4em}{$4\times[1024]$, sigmoid} 
			& 0. & 98.22\% & 83.04\% & 0.00\% & 86.9974 & 343582.3125 \\ 
			%			\cline{2-7}
			%			& 0.005 & 98.18\% & 95.02\% & 93.20\% & 2.1760 & 15.3358 \\ 
			\cline{2-7} 
			& 0.01 & 97.83\% & 94.65\% & \textbf{93.41\%} & 1.6823 & 10.2289 \\ 
			\cline{2-7} 
			& 0.02 & 97.33\% & 94.02\% & 92.94\% & 1.2089 & 6.5573 \\ 
			\cline{2-7} 
			& 0.03 & 97.07\% & 93.52\% & 92.65\% & 1.0144 & 4.9576 \\ 
			\cline{2-7} 
			& 0.04 & 96.70\% & 92.78\% & 91.95\% & 0.8840 & 3.9967 \\
			\cline{2-7} 
			& 0.05 & 96.38\% & 92.29\% & 91.33\% & 0.7890 & 3.4183 \\ 
			%			\cline{2-7} 
			%			& 0.06 & 96.29\% & 92.17\% & 91.11\% & 0.7128 & 3.0050 \\
			\cline{2-7} 
			& 0.07 & 96.08\% & 91.83\% & 90.67\% & 0.6614 & 2.6905 
			\\ 
			\hline 
			\multirow{7}{4em}{$4\times[1024]$, tanh} 
			& 0. & 97.45\% & 75.18\% & 0.00\% & 913.6984 & 37148156 \\ 
			%			\cline{2-7}
			%			& 0.005 & 97.48\% & 93.29\% & 89.98\% & 2.8690 & 18.8079 \\ 
			\cline{2-7} 
			& 0.01 & 97.24\% & 93.05\% & \textbf{91.37\%} & 1.9114 & 12.2148 \\ 
			\cline{2-7} 
			& 0.02 & 96.82\% & 92.65\% & 91.35\% & 1.3882 & 7.1771 \\ 
			\cline{2-7} 
			& 0.03 & 96.27\% & 91.43\% & 90.09\% & 1.1643 & 5.1671 \\ 
			\cline{2-7} 
			& 0.04 & 95.62\% & 90.69\% & 89.41\% & 0.9620 & 3.9061 \\
			\cline{2-7} 
			& 0.05 & 95.77\% & 90.69\% & 89.40\% & 0.9160 & 3.2909 \\ 
			%			\cline{2-7} 
			%			& 0.06 & 95.52\% & 90.00\% & 88.38\% & 0.8234 & 2.8808 \\
			\cline{2-7} 
			& 0.07 & 95.24\% & 89.51\% & 87.91\% & 0.7540 & 2.5635 
			\\ 
			\hline 
		\end{tabular}
		\label{table:hyperparam_search}
	\end{table*}
	
	\begin{table*}[h!]
		\centering
		\renewcommand{\arraystretch}{1.15} 
		\caption{In this table, we measure the impact of increasing curvature regularization $(\gamma)$ on accuracy, empirical robust accuracy, certified robust accuracy, CROWN-general and CRC when the network is trained without any adversarial training. We find that adding a very small amount of curvature regularization has a minimal impact on the accuracy but significantly increases CRC. Increase in CROWN certificate is not of similar magnitude. Somewhat surprisingly, we observe that even without any adversarial training, we can get nontrivial certified accuracies of $84.73\%, 88.66\%, 89.61\%$ on 2,3,4 layer sigmoid networks respectively.
		}
		\begin{tabular}{ | l | l | l | l | l | l | l | l | } 
			\hline
			\multirow{3}{*}{Network} & \multirow{3}{*}{$\gamma$} & \multirow{3}{4em}{Standard Accuracy} & \multirow{3}{4em}{Empirical Robust Accuracy} & \multirow{3}{4em}{Certified Robust Accuracy} & \multicolumn{2}{c|}{Certificate (mean)}\\ 
			\cline{6-7}
			& & & & & \multirow{2}{*}{CROWN} & \multirow{2}{*}{CRC} \\
			& & & & & & \\
			%			\hline
			%			Config & $\gamma$ & Standard & Empirical & Certified & CROWN & CRC \\ 
			\hline
			\multirow{5}{*}{$2\times[1024]$, sigmoid}
			& 0. & 98.37\% & 76.28\% & 54.17\% & 0.28395 & 0.48500 \\ 
			\cline{2-7}
			& 0.005 & 97.96\% & 88.65\% & 82.68\% & 0.36125 & 0.83367 \\ 
			\cline{2-7} 
			& 0.01 & 98.08\% & 88.82\% & 83.53\% & 0.32548 & 0.84719  \\ 
			\cline{2-7} 
			& 0.02 & 97.88\% & 88.90\% & 83.68\% & 0.34744 & 0.86632  \\ 
			\cline{2-7} 
			& 0.03 & 97.73\% & 89.28\% & \textbf{84.73\%} & 0.35387 & 0.90490  \\ 
			\hline 
			\multirow{5}{*}{$2\times[1024]$, tanh}
			& 0. & 98.34\% & 79.10\% & 14.42\% & 0.23938 & 0.40047 \\ 
			\cline{2-7}
			& 0.005 & 98.01\% & 89.95\% & 85.70\% & 0.27262 & 0.89672 \\ 
			\cline{2-7} 
			& 0.01 & 97.99\% & 90.17\% & 86.18\% & 0.28647 & 0.93819  \\ 
			\cline{2-7} 
			& 0.02 & 97.64\% & 90.13\% & \textbf{86.40\%} & 0.30075 & 0.99166  \\ 
			\cline{2-7} 
			& 0.03 & 97.52\% & 89.96\% & 86.22\% & 0.30614 & 0.98771  \\ 
			\hline 
			\multirow{7}{*}{$3\times[1024]$, sigmoid}
			& 0. & 98.37\% & 85.19\% & 0.00\% & 0.24644 & 0.06874 \\
			\cline{2-7} 
			& 0.005 & 97.98\% & 91.93\% & \textbf{88.66\%} & 0.38030 & 0.99044 \\ 
			\cline{2-7} 
			& 0.01 & 97.71\% & 91.49\% & 88.33\% & 0.39799 & 1.07842 \\ 
			\cline{2-7} 
			& 0.02 & 97.50\% & 91.34\% & 88.38\% & 0.38091 & 1.08396  \\ 
			\cline{2-7} 
			& 0.03 & 97.16\% & 91.10\% & 88.63\% & 0.41015 & 1.15505  \\ 
			\cline{2-7} 
			& 0.04 & 97.03\% & 90.96\% & 88.48\% & 0.42704 & 1.18073  \\
			\cline{2-7} 
			& 0.05 & 96.76\% & 90.65\% & 88.30\% & 0.43884 & 1.19296 
			\\ 
			\hline 
			\multirow{7}{*}{$3\times[1024]$, tanh} & 0. & 97.91\% & 77.40\% & 0.00\% & 0.08174 & 0.01169  \\ 
			\cline{2-7}
			& 0.005 & 97.45\% & 91.32\% & \textbf{88.57\%} & 0.28196 & 0.95367  \\ 
			\cline{2-7} 
			& 0.01 & 97.29\% & 90.98\% & 88.31\% & 0.31237 & 1.05915  \\ 
			\cline{2-7} 
			& 0.02 & 97.04\% & 90.21\% & 87.77\% & 0.30901 & 1.08607  \\ 
			\cline{2-7} 
			& 0.03 & 96.88\% & 90.02\% & 87.52\% & 0.34148 & 1.11717  \\ 
			\cline{2-7} 
			& 0.04 & 96.53\% & 89.61\% & 86.87\% & 0.36583 & 1.11307  \\
			\cline{2-7} 
			& 0.05 & 96.31\% & 89.25\% & 86.26\% & 0.38519 & 1.11689 
			\\ 
			\hline 
			\multirow{7}{*}{$4\times[1024]$, sigmoid} &
			0. & 98.39\% & 83.27\% & 0.00\% & 0.19501 & 0.00454  \\ 
			%			\cline{2-7}
			%			& 0.005 & 97.74\% & 91.67\% & 88.95\% & 0.36863 & 0.91840  \\ 
			\cline{2-7} 
			& 0.01 & 97.41\% & 91.71\% & \textbf{89.61\%} & 0.40620 & 1.05323  \\ 
			\cline{2-7} 
			& 0.02 & 96.47\% & 90.03\% & 87.77\% & 0.45074 & 1.14219  \\ 
			\cline{2-7} 
			& 0.03 & 96.24\% & 90.40\% & 88.14\% & 0.47961 & 1.30671  \\ 
			\cline{2-7} 
			& 0.04 & 95.65\% & 89.61\% & 87.54\% & 0.49987 & 1.35129  \\
			\cline{2-7} 
			& 0.05 & 95.36\% & 89.10\% & 87.09\% & 0.51187 & 1.36064  \\ 
			%			\cline{2-7} 
			%			& 0.06 & 95.29\% & 88.96\% & 87.01\% & 0.52629 & 1.38666  \\
			\cline{2-7} 
			& 0.07 & 95.23\% & 88.03\% & 85.93\% & 0.54754 & 1.27948 
			\\ 
			\hline 
			\multirow{7}{*}{$4\times[1024]$, tanh} & 
			0. & 97.65\% & 69.20\% & 0.00\% & 0.03554 & 0.00028  \\ 
			%			\cline{2-7} 
			%			& 0.005 & 97.02\% & 89.77\% & 85.98\% & 0.29410 & 0.82364  \\ 
			\cline{2-7} 
			& 0.01 & 96.52\% & 89.38\% & \textbf{86.40\%} & 0.34778 & 0.97365  \\ 
			\cline{2-7} 
			& 0.02 & 96.09\% & 88.79\% & 86.09\% & 0.41662 & 1.10860  \\ 
			\cline{2-7} 
			& 0.03 & 95.74\% & 88.36\% & 85.65\% & 0.44981 & 1.17400  \\ 
			\cline{2-7} 
			& 0.04 & 95.10\% & 87.50\% & 84.74\% & 0.48356 & 1.21957  \\
			\cline{2-7} 
			& 0.05 & 95.14\% & 87.72\% & 84.77\% & 0.49113 & 1.25076  \\ 
			%			\cline{2-7} 
			%			& 0.06 & 94.66\% & 86.96\% & 84.28\% & 0.51104 & 1.28653  \\
			\cline{2-7} 
			& 0.07 & 94.34\% & 86.67\% & 83.90\% & 0.49750 & 1.24198
			\\ 
			\hline 
		\end{tabular}
		\label{table:curv_reg}
	\end{table*}
    \FloatBarrier
	
	\bibliography{main}
	\bibliographystyle{icml2020}
	
\end{document}